\documentclass[twoside,11pt]{article}

\usepackage{blindtext}
\usepackage{amsmath} 

\usepackage{jmlr2e}
\newtheorem{assumption}{Assumption} 

\usepackage{amsmath,amssymb}
\usepackage{algorithm}
\usepackage{algpseudocode}
\usepackage{mathrsfs}
\usepackage{mathtools}
\usepackage{tikz}

\newcommand{\T}{^\top}
\usepackage{color} %textcolor
\usepackage{xcolor}
\usepackage{relsize}
\usepackage{amsmath}%required
\usepackage[normalem]{ulem}

\definecolor{green}{RGB}{20, 171, 132}

\definecolor{BlueforTZ}{RGB}{0,128,255}
\definecolor{GreenforTZ}{RGB}{63,127,127}

\definecolor{olive}{RGB}{128,128,0}
\definecolor{firebrick}{RGB}{178,34,34}

\usepackage{pdflscape}
\usepackage{adjustbox}

\newcommand{\rescellfirst}[9]{%
\begin{tabular}[c]{@{}l@{\hspace{3pt}}r@{}}
SAVE      & $#1\;(#2)\,/\,#3$ \\
LFQ  & $#4\;(#5)\,/\,#6$ \\
GD        & $#7\;(#8)\,/\,#9$
\end{tabular}%
}

\newcommand{\rescell}[9]{%
\begin{tabular}[c]{@{}r@{}}
$#1\;(#2)\,/\,#3$ \\
$#4\;(#5)\,/\,#6$ \\
$#7\;(#8)\,/\,#9$
\end{tabular}%
}

\usepackage{color} %textcolor
\definecolor{BlueforTZ}{RGB}{0,128,255}
\definecolor{GreenforTZ}{RGB}{63,127,127}

\usepackage{booktabs}
\usepackage{array}
\usepackage{graphicx}
\usepackage{pdflscape}
\usepackage{lastpage}

\jmlrheading
{}
{}
{1--\pageref{LastPage}}
{}
{}
{}
{Tuoyi Zhao, Yale University, \texttt{tuoyi.zhao@yale.edu}; 
 Chengchun Shi, London School of Economics and Political Science, \texttt{c.shi7@lse.ac.uk}; 
 Zhengling Qi, George Washington University, \texttt{qizhengling@email.gwu.edu}; 
 Lan Wang, University of Miami, \texttt{lanwang@mbs.miami.edu}}
{}

\ShortHeadings
{Sparse Additive Off-Policy Evaluation}
{Zhao, Shi, Qi and Wang}

\firstpageno{1}

\begin{document}
\title{Sparse Additive Off-Policy Evaluation for Reinforcement
Learning with Potentially Limited Number of Trajectories}

\author{}

\editor{}

\maketitle

%%%%%%%%%%%%%%%%%%%%%%%%%%%%%%%%%%%%%%%%%%%%%%%%%%%%%%%%%%%%%%%%%%%%%%%%%%%%%%
\vspace{-.4cm}
\centerline{Tuoyi Zhao$^1$, Chengchun Shi$^2$, Zhengling Qi$^3$ and Lan Wang$^4$}
\vspace{.4cm}
\centerline{\it $^1$ Yale University}
\centerline{\it $^2$ London School of Economics and Political Science}
\centerline{\it $^3$ George Washington University}
\centerline{\it $^4$ University of Miami}

\vspace{.55cm}
%\fontsize{9}{11.5pt plus.8pt minus.6pt}\selectfont%

\begin{quotation}
\noindent {\it Abstract:}
We develop a new framework for flexible, nonlinear, and interpretable off-policy evaluation for infinite-horizon reinforcement learning. To handle large state spaces and support transparent decision-making, we model the Q-function using a nonlinear function class with a sparse additive structure.

We derive high-probability finite-sample error bounds for estimating the 
value function of a target policy and show that the bounds depend only logarithmically on the ambient dimension $d$, thereby alleviating the curse of dimensionality. In contrast to most existing theory for off-policy evaluation, which typically assumes access to many trajectories,
our analysis guarantees accurate value estimation when either the number of trajectories or the time horizon is sufficiently large.

In addition, we propose a group-sparsity-based feature screening procedure that identifies, with high probability, a reduced feature set containing all relevant covariates. Numerical experiments demonstrate the effectiveness of the proposed approach.\\

\noindent {\it Key words and phrases:}
High-dimensional data, Off-policy evaluation,
Reinforcement learning, Sparsity.
\end{quotation}\par

\section{INTRODUCTION}
%Towards Optimal Off-Policy Evaluation for Reinforcement Learning with Marginalized Importance Sampling

Off-policy evaluation (OPE) is a cornerstone of reinforcement learning (RL). Its goal is to estimate the performance of a target policy using offline data (historical interaction data) generated by a potentially different behavior policy (\citet{sutton2018reinforcement}). Unlike online evaluation, where an agent directly interacts with the environment to test the policy, OPE avoids the risks and costs of deploying suboptimal or unsafe policies. This is especially important for domains where real-world interactions are expensive, risky, or ethically constrained. Examples include targeted advertising and marketing \citep{tang2013automatic,thomas2017predictive}, education \citep{mandel2014offline}, and healthcare \citep{murphy2001marginal,murphy2003optimal,thomas2017predictive}, where poor decisions can lead to wasted resources, lost opportunities, or even harm to individuals. 
%By enabling reliable policy evaluation  without active experimentation, OPE not only offers a safer and more cost-effective alternative to online testing, but also plays a key role in bridging the gap between theoretical advances in RL and their responsible deployment in high-stakes real-world settings.

Existing off-policy evaluation methods in RL can be broadly grouped into four categories: model-based approaches \citep{liu2018representation,gottesman2019combining,yin2020asymptotically,zhang2021autoregressive}, importance sampling (IS)-based methods \citep{liu2018breaking,wang2021projected}, value function-based methods \citep{luckett2019estimating,liao2021off}, and doubly robust (DR) estimation methods \citep{jiang2016doubly,thomas2016data,bibaut2021sequential,kallus2022efficiently,liao2022batch}. Each class of methods strikes a different balance between bias and variance.
However, the theoretical guarantees of most existing approaches rely on access to a sufficiently large number of trajectories, an assumption that is often unrealistic in real-world applications where data collection is costly or constrained. To address this challenge, \citet{kallus2022efficiently} proposed a double reinforcement learning framework tailored to the limited-trajectory regime. While promising, their approach still requires estimating density ratios, which remains a technically challenging and can be unstable in practice. More recently, \citet{shi2022statistical} investigated inference based on linear approximation of the $Q$-function 
and proved theoretical guarantees even with a limited number of trajectories, provided that the number of decision points is sufficiently large. 

Offline RL with linear function approximation, but without exploiting sparsity structures, has been extensively studied \citep{bertsekas1995neuro,lagoudakis2003,munos2003,alizadeh2003second,tosatto2017,duan2020minimax,jin2020provably,shi2022statistical}, among others. Building on this line of work, sparse learning in RL has also attracted significant attention. Early efforts focused on combining temporal-difference (TD) learning with $\ell_1$ regularization in on-policy settings \citep{kolter2009regularization,geist2011l1,hoffman2011regularized,painter2012greedy}, while \citet{liu2012regularized} extended these ideas to off-policy evaluation via regularized TD learning. 
However, these approaches largely lacked error bounds or consistency guarantees.
To address this gap, \citet{ghavamzadeh2011finite} proposed Lasso-TD and established in-sample guarantees for value function estimation, and \citet{farahmand2016regularized} developed a regularized policy iteration algorithm in nonparametric function spaces with accompanying statistical analysis. More recently, attention has shifted toward sparse linear MDPs: \citet{hao2021sparse} studied sparse linear approximation for offline RL and derived finite-sample error bounds that avoid polynomial dependence on the ambient dimension, thereby providing a theoretical foundation for the sample-efficiency benefits of sparsity. Their analysis was restricted to unstructured sparsity. Extensions include \citet{hao2021online}, who analyzed the on-policy setting, and \citet{golowich2023exploring}, who considered a broader class of $\ell_1$-bounded linear MDPs and designed algorithms achieving near-optimal on-policy performance. Beyond Lasso-type approaches, \citet{ma2023sequential} proposed a model-free variable selection framework using sequential knockoffs, which efficiently identifies minimal sufficient state representations.

In this work, we study infinite-horizon offline reinforcement learning, where 
the data-generating mechanism is modeled by a Markov decision process.
Our focus is off-policy evaluation in large state spaces, and we extend semiparamtric function approximation methods to incorporate group-wise sparsity structures.  We are particularly interested
in the high-dimensional regime where the state dimension $d$ is very large and satisfies
$d\gg \max\{n, T\}$,
with $n$ denoting the number of observed trajectories and $T$ the number of observed time points.

We propose a new algorithm for estimating the cumulative discounted value of a target policy in the infinite-horizon setting that performs policy evaluation and feature selection simultaneously.  {Motivated by sparse additive modeling in supervised learning
\citep{yuan2006model,ravikumar2009sparse,pmlr-v9-liu10a}, our approach
extends the infinite-horizon Bellman estimating framework of
\citet{shi2022statistical}, 
which is developed for low-dimensional state spaces.
In our high-dimensional setting, the
spline-expanded coefficient dimension can exceed $nT$. This necessitates group regularization and requires new theoretical analysis
to establish uniform concentration bounds for the empirical Bellman scores
under Markov dependence, while simultaneously accounting for both the diverging spline
dimension and the spline approximation error.}

We derive high-probability finite-sample error bounds that depend on the ambient dimension $d$ only through $\log d$ and $d_0$, the number of relevant features. In particular, the estimation error has no polynomial dependence on $d$, thereby mitigating the curse of dimensionality.  {Our theory guarantees accurate estimation when either the number of
trajectories $n$ or the number of decision points $T$ is sufficiently large.
%Under the mixing and stated growth conditions, increasing $T$ improves not
%only the magnitude of the finite-sample error bound but also its confidence
%level. Consequently, 
The proposed estimator remains consistent in probability
as $T\to\infty$ even when $n$ is fixed.} Furthermore, we propose a group-sparsity-aware feature screening procedure that, with high probability, can identify a reduced set containing all relevant features.
Our work thus significantly extends the tools and theory available for handling large state space for off-policy evaluation.

The rest of the paper is organized as follows.
We introduce the notation and background in Section 2.
Section 3 presents the sparsity-aware nonlinear sparse model and the estimation procedure. Section 4 drives the finite-sample error bounds for the value function under Markov errors. Numerical studies are reported in Sections. The online supplement contains 
additional technical proofs.

\section{PRELIMINARIES}

%In this section, we will introduce the framework of our %OPE problem and sparse MDP.

\subsection{Off-policy Evaluation for Reinforcement Learning }

 The dataset consists of $n$ independent trajectories: 
$
\{(X_{i,t}, A_{i,t}, R_{i,t}, X_{i,t+1}): 1 \le i \le n, \; 1 \le t \le T\}$,    
where each trajectory $\{(X_{0,t},A_{0,t},R_{0,t},X_{0,t+1}) : t \ge 0\}$ is generated according to a Markov decision process (MDP). At time $t$, a state $X_{0,t}\in\mathcal{X}^d$ is observed, an action $A_{0,t}\in\mathcal{A}$ is chosen according to the behavior policy $b(\cdot|X_{0,t})$, the system then transits to the next state $X_{0,t+1}\sim p(\cdot|X_{0,t},A_{0,t})$ where $p(\cdot|\cdot, \cdot)$ denotes the transition probability (also called transition kernel or transition dynamics), and an immediate reward $R_{0,t}$ is received with conditional mean 
$
r(x,a) = \mathbb{E}(R_{0,t}\mid X_{0,t}=x,A_{0,t}=a)$.
It is standard to assume $|R_{i,t}|\leq c_r$ for a positive constant $c_r$, $\forall \ i, t$.

We assume that the random trajectory satisfies the
following Markov assumption (MA): for any $\mathcal{B}\subset\mathcal{X}^d$,
\begin{align*}
    P\Big(X_{0, t+1} \in \mathcal{B} & \,\big|\, X_{0, t}=x, A_{0, t}=a,
   \{R_{0, j}\}_{0 \leq j<t}, \{X_{0, j}\}_{0 \leq j<t}, \{A_{0, j}\}_{0 \leq j<t}\Big) \\
   &= P\big(X_{0, t+1}\in \mathcal{B} \mid X_{0, t}=x, A_{0, t}=a\big),
\end{align*}
  and the conditional mean independence assumption (CMIA):
\begin{equation*}
    \begin{aligned}
        \mathbb{E}\Big(R_{0, t} \mid & X_{0, t}=x, A_{0, t}=a,\left\{R_{0, j}\right\}_{0 \leq j<t},\left\{X_{0, j}\right\}_{0 \leq j<t},\\  &\left\{A_{0, j}\right\}_{0 \leq j<t}\Big)  
= \mathbb{E}\left(R_{0, t} \mid X_{0, t}=x, A_{0, t}=a\right).
    \end{aligned}
\end{equation*}
If MA holds and $R_{0,t}$ is a deterministic function of $X_{0,t}$ and $A_{0,t}$, 
%and $X_{0,t+1}$,
then CMIA is automatically  satisfied  \citep{shi2022statistical}.

Let $\pi$ be a stationary policy of interest, which
chooses the action $a\in \mathcal{A}$ with probability
$\pi(a|x)$ given the state $x$. 
The performance
of $\pi$ is measured by the expected cumulative discounted reward, also called value function, defined as
$
V^{\pi}(x)=\sum_{t=0}^\infty \gamma^t \mathbb{E}^\pi\left(R_{0, t} \mid X_{0,0}=x\right)$, where
$\gamma \in [0,1)$ is the discount factor
that balances the trade-off
between the immediate and future rewards,
$\mathbb{E}^\pi(\cdot \mid X_{0,0}=x)$ denotes the expectation 
with respect to the random trajectory
when it begins with the initial state $x$ and
the actions along the trajectory are taken
according to $\pi$.
We also introduce the state-action value function
or the Q-function:
$$
Q^{\pi}( x,a)=\sum_{t=0}^\infty \gamma^{t} \mathbb{E}^\pi\left(R_{0, t} \mid X_{0,0}=x,A_{0,0} = a\right).
$$
The $Q$ function satisfies the Bellman consistency equations
$$ Q^{\pi}(x,a) = r(x,a) +{\gamma}\mathbb{E}^{\pi}(Q^{\pi}(x^{\prime},a^{\prime})\mid x,a),
$$
for all $(x,a)\in \mathcal{X}^d\times \mathcal{A}$.
It is known that $V^{\pi}(x)=\sum_a\pi(a|s)Q^{\pi}( x,a)$.
Our main goal is to estimate the cumulative discounted value of 
$\pi$ using the offline trajectories data.

\subsection{Q-function Modeling under Sparsity Constraint}

In many applications, the number of features can be very large, potentially exceeding $\max\{n, T\}$. In such high-dimensional settings, existing methods may perform poorly: from a theoretical perspective, coverage conditions based on eigenvalues of feature matrices may fail because the corresponding matrices are no longer full rank.
In this paper, we develop a new sparsity-aware Q-function modeling procedure to address the challenges associated with high-dimensional feature space. 

\begin{assumption}[\textbf{Sparse Q-function}]\label{sparseQ}
Let $\pi$ be the target policy. There exists an active feature set $\mathcal{K} \subseteq [d]$ with cardinality $|\mathcal{K}| \leq d_0 \ll d$. For each action $a \in \mathcal{A}$, there exists a subset $S_a \subseteq \mathcal{K}$ such that the Q-function depends only on the coordinates indexed by $S_a$, i.e.,
\begin{equation}
    Q^{\pi}(x,a) = Q^{\pi}(x_{S_a}, a), \quad \forall \ (x,a) \in \mathcal{X} \times \mathcal{A}.
\end{equation}

\end{assumption}

 {Because our goal is policy evaluation, we impose sparsity directly on the target-policy Q-function. Defined in terms of the original state coordinates, $S_a$ identifies the characteristics that determine how long-run policy performance varies with the initial state and action $a$, while the estimated component functions describe these relationships. Variables outside $S_a$ may affect rewards or transitions but provide no additional information about the expected cumulative return conditional on the selected variables and initial action.}
This differs from the sparse linear MDP framework of \citet{hao2021sparse}, which imposes sparsity on the transition structure relative to a prespecified feature dictionary. Our formulation instead captures sparsity in the Q-function after rewards and target-policy transition dynamics are combined.

 {The following proposition presents two alternative sufficient conditions for Assumption~\ref{sparseQ}. These conditions serve only to illustrate structural settings that imply Assumption~\ref{sparseQ}; the subsequent analysis imposes Assumption~\ref{sparseQ} directly and does not rely on either condition.}

\begin{proposition}\label{condition_sparseQ}
Assumption~\ref{sparseQ} holds if there exists an index set
$\mathcal K\subseteq[d]$ with $|\mathcal K|\ll d$ such that either of the
following conditions is satisfied:
\begin{itemize}
    \item \textbf{(Sufficient Condition 1)}
    The reward function and the full transition kernel depend on the current
    state only through $x_{\mathcal K}$:
    $r(x,a)=r(x_{\mathcal K},a)$ and
    $
    p(x'\mid x,a)=p(x'\mid x_{\mathcal K},a)$.

   \item \textbf{(Sufficient Condition 2)}
     {The reward function and the target policy depend only on
    $x_{\mathcal K}$:
    $
    r(x,a)=r(x_{\mathcal K},a)$,
    $
    \pi(a\mid x)=\pi(a\mid x_{\mathcal K})$,
    and the transition of the active state satisfies
    $p(x'_{\mathcal K}\mid x,a)
    =
    p(x'_{\mathcal K}\mid x_{\mathcal K},a)$.}
\end{itemize}
Under either condition,
$
Q^\pi(x,a)=Q^\pi(x_{\mathcal K},a)$,
so that $\mathcal K$ can serve as an active feature set in
Assumption~\ref{sparseQ}.
\end{proposition}

\begingroup
 
\begin{remark}
Both sufficient conditions imply the joint-transition property
\begin{equation}\label{sparse_kernal}
p^\pi(x'_{\mathcal K},a'\mid x,a)
=
p^\pi(x'_{\mathcal K},a'\mid x_{\mathcal K},a),
\end{equation}
where \(p^\pi\) denotes the conditional distribution of
\((X'_{\mathcal K},A')\) under the target policy. Together with
\(r(x,a)=r(x_{\mathcal K},a)\), this property closes the Bellman
recursion on the reduced state and implies
\(Q^\pi(x,a)=Q^\pi(x_{\mathcal K},a)\).

Sufficient Condition~1 ensures \eqref{sparse_kernal} by requiring
the full transition kernel of \(X'\) to depend on \(x\) only
through \(x_{\mathcal K}\), without restricting the target policy.
Sufficient Condition~2 instead combines
$p(x'_{\mathcal K}\mid x,a)
=
p(x'_{\mathcal K}\mid x_{\mathcal K},a)$
and
$\pi(a'\mid x')
=
\pi(a'\mid x'_{\mathcal K})$,
which yields
\[
p^\pi(x'_{\mathcal K},a'\mid x,a)
=
p(x'_{\mathcal K}\mid x_{\mathcal K},a)
\pi(a'\mid x'_{\mathcal K}).
\]
Thus, the conditions are not nested: the first imposes a stronger
transition restriction, whereas the second adds a policy restriction
to a reduced-state transition condition. This policy restriction is
only sufficient for verifying \eqref{sparse_kernal}; it is not used
in the subsequent analysis, which relies directly on
Assumption~\ref{sparseQ}.

Property~\eqref{sparse_kernal} identifies a low-dimensional,
decision-relevant subsystem: under the target policy, the joint
transition of \((X'_{\mathcal K},A')\) depends on the current state
only through \(x_{\mathcal K}\). Thus, the Bellman continuation can
be expressed using the reduced state \(X_{\mathcal K}\). The goal is
to reduce the state information needed for policy evaluation, not to
obtain a sparse representation of the full transition kernel.
This perspective is related to the minimum sufficient MDP of
\citet{ma2023sequential}, although we do not require
\(\mathcal K\) to be minimal. In contrast,
\citet{hao2021sparse} retain the full state space and impose sparsity
on the transition representation relative to a prespecified feature
map. Property~\eqref{sparse_kernal} is only sufficient for verifying
Assumption~\ref{sparseQ}; the subsequent analysis imposes that
assumption directly.
\end{remark}

\endgroup

 \subsection{Notation}
We next introduce some notations. Denote $[d] = \{1,\dots,d\}$. For a $d$-dimensional vector $v = (v_1,\dots,v_d)\T$, let $\|v\|_1 = \sum_{i=1}^d|v_i|$,  $\|v\|_2 = \sum_{i=1}^d|v_i|^2$ and $\|v\|_{\infty} = \max_{i}|v_i|$. Let $(v,w)_{\oplus}= (v_1,\dots,v_d,w_1,\dots,w_d)\T$ denote the concatenation between vectors.   For an arbitrary matrix $A$, we define the matrix norms: $\|A\|_{max} = \max_{i,j} |A_{i,j}|$ and $\|A\|_{2} = \max_{\|x\|_2=1} \|Ax\|_2$. For any real value bounded function $f(x)$ on $\mathcal{X}$, $||f(x)||_{\infty} = \sup_{x\in \mathcal{X}}|f(x)|$. For
any set $\mathcal{S}$,  $\mathcal{S}^c $  denotes its complement, and $|\mathcal{S}|$  denotes its cardinality. In the case $\mathcal{S}$ is an index set, $v_{\mathcal{S}}$ denotes the $|\mathcal{S}|$-dimensional sub-vector $(v_i)_{i\in \mathcal{S}}$. For two generic sequences
$\{a(n)\}_{n\geq 1}$ and $\{b(n)\}_{n\geq 1}$, 
$a(n)\lesssim b(n)$ means that there exists some
positive constant $c$ such that $a(n)\leq cb(n)$, $\forall \ n$. 

\section{NONLINEAR SPARSE MODELING AND ESTIMATION FOR OFF-POLICY EVALUATION}

\subsection{Sparse Additive Modeling of the Q function} \label{section_method}

Let the state $x = (x_1,\ldots,x_d)^\top\in \mathcal{X}^d$. We assume a compact state space and, without loss of generality, take 
$\mathcal{X}^d = [0,1]^d$.  We consider a high-dimensional setting 
%in which the state dimension $d$ may far exceed both the number of trajectories and the time horizon, i.e., 
where $d \gg \max\{n, T\}$ and focus on binary actions, but the method extends straightforwardly to any finite action space.
%In high-dimensional reinforcement learning, %the relationship between state variables and the Q-function is often nonlinear. Purely linear models may miss important structures. 
To balance flexibility, interpretability, and statistical efficiency for high-dimensional RL, we adopt an additive modeling framework:  
\begin{equation}
Q^{\pi}(x,a) = \alpha_a^* + \sum_{k=1}^d q^{\pi}_{k,a}(x_k), \quad a \in \{0,1\},
\end{equation}
where $\alpha_a^*$ are unknown constants and $q^{\pi}_{k,a}(x_k)$ are unknown univariate functions. This additive structure captures nonlinear effects of individual state components without requiring the estimation of a fully nonparametric multivariate function, thereby mitigating the curse of dimensionality.
%while preserving interpretability.

%To enhance interpretability and improve estimation in the high-dimensional regime, we impose 
Under the sparse Q-function assumption (Assumption~\ref{sparseQ}),
%and further assume that 
%for each action $a \in \{0,1\}$,
\begin{equation}
Q^{\pi}(x,a) = Q^{\pi}(x_{S_a},a) = \alpha_a^* + \sum_{s \in S_a} q^{\pi}_{s,a}(x_s),
\quad a \in \{0,1\},
\end{equation}
where $S_a$ is the active feature set associated with action $a$. 
%Thus, state components with indices outside $S_0 \cup S_1$ do not affect the Q-function. 
For identifiability, we assume 
$\mathbb{E}_{x \sim \mathbb{G}} \{ q^{\pi}_{s,a}(x_s) \} = 0$
for a reference distribution $\mathbb{G}$ on $\mathcal{X}^d$ (e.g., the state visitation distribution).
%given a reference distribution $\mathbb{G}$ on $\mathcal{X}^d$ (e.g., the state visitation distribution), we assume without loss of generality that
%$\mathbb{E}_{x \sim \mathbb{G}} \{ q^{\pi}_{s,a}(x_s) \} = 0$.
\begingroup
 
We approximate each $q_{s,a}^{\pi}(x_s)$ by a linear combination of
basis functions and treat the basis coefficients associated with each
state variable $x_s$ as a group. 
The resulting representation is linear in the spline coefficients while retaining a flexible nonlinear dependence on the state variables,
with the spline approximation error explicitly accounted for in our theory.
%We allow the
%number of basis functions to increase with the sample size and explicitly
%account for the resulting spline approximation error in our theory.
\endgroup

\subsection{Assumptions for Function Approximation}

\begin{definition}[Smoothness]
For $\nu > 0$, let $\lfloor \nu \rfloor$ be the largest integer not exceeding $\nu$. A real-valued function $h$ on $[0,1]$ is said to be $\nu$-smooth if its $\lfloor \nu \rfloor$-th derivative $h^{(\lfloor \nu \rfloor)}$ satisfies
$
\sup_{z \neq z'} \frac{\left| h^{(\lfloor \nu \rfloor)}(z') - h^{(\lfloor \nu \rfloor)}(z) \right|}{|z'-z|^{\nu - \lfloor \nu \rfloor}} \leq c,
$
for some constant $c > 0$. The collection of such functions is denoted $\mathcal{H}_{\nu,c}$.
\end{definition}

\begin{assumption}[Function class]\label{assumption_smooth}
There exist constants $(\nu,c)$ with $\nu \geq 3$ such that $q^{\pi}_{s,a}(\cdot) \in \mathcal{H}_{\nu,c}$ for all $s \in S_a$ and $a \in \{0,1\}$.
\end{assumption}

 {Assumption~\ref{assumption_smooth} imposes smoothness conditions on  $q^{\pi}_{s,a}(\cdot)$.
We note that Assumption~\ref{assumption_smooth} can be relaxed to $\nu \geq 2$ for any $\epsilon > 0$ when $d$ is small and independent of $n$ and $T$.}
%Notice that $D^{\zeta^{*}}q_{s,a}^{\pi}(x) = D^{\zeta}Q^{\pi}(x,a)$ by setting $\zeta^{*} = (0,\dots,0,\|\zeta\|_1 ,0,\dots,0)$ whose $s$-th entry is $\|\zeta\|_1$ and $\zeta^* = \|\zeta\|_1$. As a result, $q_{s,a}^{\pi}(x) \in \mathcal{H}_{\nu,c,1}$ for all $s =1,\dots,d$. For a fixed value of $a$, $q_{s,a}^{\pi}(x)$ and $Q^{\pi}(x,a)$ share the same smoothness properties. 
%In our setting, we use the $L$ dimensional B-spline basis function with order $ m = \lfloor %\nu \rfloor$. 
%We further denote $\boldsymbol{\beta}_{s,a}^* = 0$ for $s\notin S_a$. We have
%\begin{equation}\label{pro_error}
%  \sup _{x \in \mathcal{X}, a \in \mathcal{A}}\left|Q^{\pi}( x, a)-\alpha^*_{a}-\sum_{s %=1}^d \Phi_s\T(x) \boldsymbol{\beta}_{s,a}^*\right| \leq d_0 C L^{-m }  
%\end{equation}
We approximate each component function $q_{s,a}^{\pi}(x_s)$ using $L$ B-spline basis functions of order $m = \lfloor \nu \rfloor$, denoted by  
$\Phi_s(x_s) = (\phi_{s,1}(x_s), \ldots, \phi_{s,L}(x_s))^\top$.  
A brief review of B-spline basis functions and their properties is provided in Appendix S2 of the online supplement. Since $\nu - 1 \leq m \leq \nu$, there exists a coefficient vector $\boldsymbol{\beta}_{s,a}^* \in \mathbb{R}^L$ such that for all $a \in \{0, 1\}$ and $s \in S_a$,  
$
\sup_{x_s \in \mathcal{X}} \big| q_{s,a}^{\pi}(x_s) - \Phi_s^\top(x_s)\boldsymbol{\beta}_{s,a}^* \big| \leq C L^{-m}$, 
for some constant $C > 0$  (e.g., \citealt{huang1998}, \citealt{schumaker2007}). Consequently,   
\begin{equation}\label{funcional_est}
\sup_{x \in \mathcal{X}} \Big| Q^{\pi}(x,a) - \alpha_a^* - \sum_{s \in S_a} \Phi_s^\top(x_s)\boldsymbol{\beta}_{s,a}^* \Big| \leq C d_0 L^{-m}, \quad a \in \{0, 1\}.
\end{equation}
The number of basis functions $L$ can diverge with $n$ and/or $T$, alleviating the bias of sparse linear modeling while incorporating the group structure in estimation. We take $L\asymp (nT)^{1/(2m-1)}$ for theoretical derivations.

\subsection{Regularized Estimation in High Dimensions}
The Bellman equation for reinforcement learning (
\citet{sutton2018reinforcement}) implies that $\forall \ i, t$,
\begin{equation}\label{10}
 \begin{aligned}
     \mathbb{E}\bigg[R_{i, t}+\gamma \sum_{a \in \mathcal{A}} Q^{\pi}\left(X_{i, t+1}, a\right) \pi\left( X_{i, t+1},a\right)- Q^{\pi}\left(X_{i, t}, A_{i, t}\right) \mid X_{i, t}, A_{i, t}\bigg]=0.
 \end{aligned}
 \end{equation}

Let $\boldsymbol{\beta}_{s,a}^* = 0$ for $s\notin S_a$,
Define the   basis functions  $\Phi (x) = (1,\Phi_1(x_1),\dots,\Phi_d(x_d))_{\oplus}$, where $\oplus$ denotes concatenation of vectors.
We approximate $Q^{\pi}(x,a)$ by
 \begin{equation}
 %Q^{\pi}(x,a) \approx   
 Q^{*}_{\pi}(x,a)=\Phi(x)^T(\alpha_a^*,\boldsymbol{\beta}_{1,a}^*,\dots,\boldsymbol{\beta}_{d,a}^*)_{\oplus}.
 \end{equation}
 Then it follows from (\ref{10}) that
\begin{equation}\label{fbellman}
 \begin{aligned}
 \mathbb{E}\bigg[\big\{R_{i, t}+\gamma \sum_{a \in \mathcal{A}} Q^{*}_{\pi}\left( X_{i, t+1}, a\right) \pi\left( X_{i, t+1},a\right)- Q^{*}_{\pi}\left(  X_{i, t}, A_{i, t}\right)\big\} \mid X_{i, t}, A_{i, t}\bigg]\approx 0.    
 \end{aligned}
 \end{equation}
To introduce the proposed estimation procedure, we next introduce some new notation. Let $\boldsymbol{\alpha}^* = (\alpha_0^*,\alpha_1^*)\T$.
Consider the $(2Ld)$-dimensional vector
$\boldsymbol{\beta}^* = (\boldsymbol{\beta}_{1,0}^*,
\ldots, \boldsymbol{\beta}_{d,0}^*,
\boldsymbol{\beta}_{1,1}^*,
\ldots, \boldsymbol{\beta}_{d,1}^*)_{\oplus}$.
Here and in the sequel, we drop 
$\pi$ from the superscript in defining $\boldsymbol{\alpha}^*$, $\boldsymbol{\beta}^*$ and other related notation for simplicity.
 We introduce the group index $G_{s}$ to denote the set of indices associated with $x_{s}$. 
This enables us to rewrite $\boldsymbol{\beta}^* = (\boldsymbol{\beta}_{G_1}^*,\dots,\boldsymbol{\beta}_{G_{2d}}^*)_{\oplus}$, that is, $\boldsymbol{\beta}_{G_s}^{*} = \boldsymbol{\beta}_{s,0}^{*}$,  $\boldsymbol{\beta}_{G_{s+d}}^{*} = \boldsymbol{\beta}_{s,1}^{*}$, for 
$1\leq s\leq d$.  Similarly, for $ 1\leq s\leq d$, denote the vector $\boldsymbol{\xi}_{G_s}(x, a) =\Phi_s(x_s) \mathbb{I}(a=0), ~\boldsymbol{\xi}_{G_{s+d}}(x, a)=\Phi_s(x_s) \mathbb{I}(a=1)$, $\boldsymbol{U}_{G_s}(x, a) =\Phi_s(x_s) \pi(0|x)$ and $\boldsymbol{U}_{G_{s+d}}(x, a)=\Phi_s(x_s) \pi(1|x)$, where $\mathbb{I}(\cdot)$ is the indicator function. Furthermore, we slightly abuse the notation and let
$\Phi_0(\cdot) = 1$, $\boldsymbol{\xi}_{G_0}(x, a) =(\mathbb{I}(a=0),\mathbb{I}(a=1))\T$, 
and let
$\boldsymbol{U}_{G_0}(x, a) =(\pi(0|x),\pi(1|x))\T$.

%Let $\{G_0, G_1, \dots G_{2d}\}$ be a partition of the index set $\{1,\dots,2(Ld+1)\}$ such that $G_0 = \{1,2\}$, and $G_1, \ldots, G_{2d}$ each has $L$ 
%elements. For $1\leq s\leq d$ and $a\in \{0,1\}$,
%G_{s+ad}$ denotes the size $L$ index set corresponding to %modeling $q_{s,a}^{\pi}(x_s)$. 

 %For $ 1\leq s\leq d$,  denote $\boldsymbol{\alpha}^* = (\alpha_0^*,\alpha_1^*)\T$, $\boldsymbol{\beta}_{G_s}^{*} = \boldsymbol{\beta}_{s,0}^{*}$,  $\boldsymbol{\beta}_{G_{s+d}}^{*} = \boldsymbol{\beta}_{s,1}^{*}$ and $\boldsymbol{\beta}^* = (\boldsymbol{\beta}_{G_1}\starT,\dots,\boldsymbol{\beta}_{G_{2d}}\starT)\T$. 
%Let $\Phi_s(x)$ denotes the B-spline basis functions %$\Phi_s(x) = B(x_s)= (B_1(x_s),\dots,B_L(x_s))\T$  which %we described in Appendix \ref{Bspline}.

%Note that our target policy is fixed, we will omit subscript $\pi$ in definitions of $\boldsymbol{u}, \boldsymbol{U},\widehat{\boldsymbol{U}},\boldsymbol{\Sigma},\widehat{\boldsymbol{\Sigma}},\boldsymbol{\alpha}^*,\widehat{\boldsymbol{\alpha}},\boldsymbol{\beta}^*,\widehat{\boldsymbol{\beta}}$  
%and use $q_{s,a}^{\pi}(x)$ and $\Phi_s(x)$ instead of $q_{s,a}^{\pi}(x_s)$ and $\Phi_s(x_s)$ for brevity.

%\LW{Make the notation $\boldsymbol{\xi}$ and $\boldsymbol{U}$ clearer by writing out
%$\boldsymbol{\xi}$ using the $\boldsymbol{\xi}_{G_k}(x, a)$ defined above. Similarly for $\boldsymbol{U}$.}

When estimating the $Q^{\pi}$ function and evaluating the relevance of specific features,
the group structure will be taken into account. 
Consider the short-hand notation  
$\boldsymbol{\xi}_{G_k,i,t}= \boldsymbol{\xi}_{G_k}(X_{i,t},A_{i,t})$ and $\boldsymbol{U}_{G_k,i,t}= \boldsymbol{U}_{G_k}(X_{i,t})$,
$0\leq k\leq 2d$. Define
%$\boldsymbol{\xi}_{i,t}= \boldsymbol{\xi}(X_{i,t},A_{i,t})$,  $\boldsymbol{U}_{i,t}= \boldsymbol{U}(X_{i,t})$ and  Here,
\begin{equation}
    \begin{aligned}
        \boldsymbol{\xi}_{i,t} &=\boldsymbol{\xi}(X_{i,t},A_{i,t})=(\boldsymbol{\xi}_{G_0,i,t},\boldsymbol{\xi}_{G_1,i,t},\dots,\boldsymbol{\xi}_{G_{2d},i,t})_{\oplus},
        \\
       \boldsymbol{U}_{i,t} &= \boldsymbol{U}(X_{i,t})=(\boldsymbol{U}_{G_0,i,t},\boldsymbol{U}_{G_1,i,t},\dots,\boldsymbol{U}_{G_{2d},i,t})_{\oplus}.
    \end{aligned}
\end{equation}

Note that $\boldsymbol{\xi}_{i,t}$ only depends on $X_{i,t} $ and $A_{i,t}$. It follows from (\ref{fbellman}) that
\begin{equation}\label{13}
\mathbb{E}\bigg[\boldsymbol{\xi}_{i,t}\bigg\{R_{i,t}-(\boldsymbol{\xi}_{i,t}-\gamma \boldsymbol{U}_{i,t+1})\T(\boldsymbol{\alpha}^*,\boldsymbol{\beta}^*)\bigg\} \bigg]\approx 0.
\end{equation}

A similar set of estimating equations was recently considered in \citep{shi2022statistical} for the case $d$ is small and fixed.
In our setting $d\gg \max\{n, T\}$, the above set of estimating equations is ill-posed as the number of unknown parameters may be much larger than the number of equations.

To address the challenge of high-dimensional parameter estimation and to explicitly incorporate the group-wise sparsity structure, we introduce the following group-Dantizg type 
procedure for estimating $\boldsymbol{\alpha}^*$ and $\boldsymbol{\beta}^*$. Specifically, the proposed
estimator $\widehat{\boldsymbol{\alpha}}$ and $\widehat{\boldsymbol{\beta}}$
are obtained by solving the following constrained optimization problem:
%In this section, we introduce the following regularized %estimator for system (\ref{13}):
\begin{equation}\label{dant_est}
\begin{aligned}
&\quad\quad\quad \quad\quad\quad\widehat{\boldsymbol{\alpha}},\widehat{\boldsymbol{\beta}}=\underset{\boldsymbol{\beta}}{\arg \min }\sum_{k =1}^{2d}\|\boldsymbol{\beta}_{G_k}\|_2, \;\;\;
\text{such that } \\
&\left\|\sum_{i=1}^n\sum_{t=0}^{T-1}\boldsymbol{\xi}_{G_0,i,t}\bigg\{R_{i,t}-(\boldsymbol{\xi}_{i,t}-\gamma \boldsymbol{U}_{i,t+1})\T(\boldsymbol{\alpha},\boldsymbol{\beta})_{\oplus}\bigg\}\right\|_{2}  \leqslant \frac{nT\lambda}{\sqrt{L}},\\& ~\text{ and for } 1\leq k\leq 2d,\\
&%\max_{k = 1,\dots,2d}
\left\|\sum_{i=1}^n\sum_{t=0}^{T-1}\boldsymbol{\xi}_{G_k,i,t}
\bigg\{R_{i,t}-(\boldsymbol{\xi}_{i,t}-\gamma \boldsymbol{U}_{i,t+1})\T(\boldsymbol{\alpha},\boldsymbol{\beta})_{\oplus}\bigg\}\right\|_{2} \leqslant nT\lambda. \\
\end{aligned}
\end{equation}
%\begingroup
 {
Lemma~\ref{feasiable_dantzig} in the Appendix shows that the oracle
coefficient vector
$(\boldsymbol{\alpha}^*,\boldsymbol{\beta}^*)_{\oplus}$ satisfies the
constraints with high probability. Specifically, when the empirical Bellman
score is evaluated at
$(\boldsymbol{\alpha}^*,\boldsymbol{\beta}^*)_{\oplus}$, it decomposes into
a centered stochastic fluctuation and a spline-approximation term that admits
a deterministic bound. After normalization by $nT$, the stochastic term is
controlled uniformly over the coefficient groups at the rate
$\lambda \asymp
\sqrt{L}\frac{\log(nT\vee d)}{\sqrt{nT}}$,
whereas the spline-approximation term is bounded by
$Cd_0L^{1/2-m}$. To ensure that the latter is asymptotically negligible
relative to $\lambda$, a convenient sufficient condition is $\frac{d_0\sqrt{nT}}{L^m}=o(1)$.
Under the choice $L\asymp(nT)^{1/(2m-1)}$, this condition reduces to
$d_0=o\left((nT)^{1/(4m-2)}\right)$.}

Given the estimators $\widehat{\boldsymbol{\alpha}}$ and $\widehat{\boldsymbol{\beta}}$, we estimate the $Q$-function or the value function at a given state $x$ as follows. Let $\mathbb{G}$ denote the reference distribution of $x$, and let $\mathbb{G}_s$ denote the corresponding marginal distribution of the feature $x_s$. Let $\widehat{q}_{s,a}(x_s) = \Phi_s\T(x_s) \widehat{\boldsymbol{\beta}}_{s,a}$. 
and $
\Tilde{\alpha}_a = \widehat{\alpha}_a + \sum_{s=1}^d \int_{x\in\mathcal{X}^d}\widehat{q}_{s,a}^{\pi}(x)d\mathbb{G}$. Our estimator of $q_{s,a}^{\pi}(x_s) $ is given by 
       $$ \tilde{q}_{s,a}(x_s) = \widehat{q}_{s,a}^{\pi}(x_s) - \int_{z\in\mathcal{X}}\widehat{q}_{s,a}^{\pi}(z)d\mathbb{G}_S, ~~1\leq s \leq d.
$$
Correspondingly, we estimate $ Q^{*}_{\pi}( x,a)$ by $\widehat{Q}^{\pi}(x,a) = \Tilde{\alpha}_a + \sum_{s=1}^d \tilde{q}_{s,a}^{\pi}(x_s)$, which can also be expressed as 
\begin{equation*}
    \widehat{Q}^{\pi}(x,a) 
= \Phi\T(x)(\widehat{\alpha}_a,\widehat{\boldsymbol{\beta}}_{1,a},\dots,\widehat{\boldsymbol{\beta}}_{d,a})_{\oplus}, \quad a\in \{0, 1\}.
\end{equation*}

\begin{remark}
Although \eqref{dant_est} draws on high-dimensional additive regression
\citep{yuan2006model,candes2007dantzig,pmlr-v9-liu10a}, its analysis must
accommodate Markov-dependent Bellman scores rather than independent
regression errors. The estimator is also partially penalized:
\((\alpha_0,\alpha_1)^\top\) are intrinsic model parameters and remain
unpenalized, while the spline coefficient groups are regularized.
Unlike a conventional regression intercept, these parameters are retained
explicitly rather than removed by centering. This formulation also permits
prespecified important covariates to remain unpenalized.
\end{remark}

    %Second, the constraints of group dantzig estimator in \citep{pmlr-v9-liu10a} are based on the relationship between LASSO \citep{tibshirani1996regression} and dantzig estimator  and then modify the constraints by 'mimicking' the construction of Group LASSO \citep{yuan2006model}. However, in our problem, construct the corresponding LASSO type optimization problem is impracticable. The equations system (\ref{13}) comes from the conditional expectation with respect to $X_{i,t}$ and $A_{i,t}$ in (\ref{10}). Even though $\boldsymbol{\xi}_{i,t} $ is on $ \sigma\{(X_{i,t},A_{i,t})\}$, $\boldsymbol{U}_{i,t+1}$ still depends on $X_{i,t+1}$. In other words, we have no guarantees for $\mathbb{E}_n[(\boldsymbol{\xi}_{i,t}-\gamma \boldsymbol{U}_{i,t+1})\{R_{i,t}-(\boldsymbol{\xi}_{i,t}-\gamma \boldsymbol{U}_{i,t+1})\T(\boldsymbol{\alpha}\starT,\boldsymbol{\beta}\starT)\T\} ]\leq \lambda$ which appears in Proposition 1 in \citep{pmlr-v9-liu10a}. Our proposed constraints are based on the feasibility of $(\boldsymbol{\alpha}^*,\boldsymbol{\beta}\starT)\T$. The feasibility is proved in lemma \ref{feasiable_dantzig}. 
 
    %Third, unlike symmetric covariance matrix in previous linear model, the conditional expectation in (\ref{10}) leads to a non-symmetric "covariance" matrix $\widehat{\boldsymbol{\Sigma}}$, which play the same role as the classical covariance matrix. The non-symmetric $\widehat{\boldsymbol{\Sigma}}$ will bring more technique difficulties. We should be more careful with results which only holds true for symmetric matrix (projection matrix, positive definite matrix and so on).

\begin{remark}
The formulation extends naturally to other group structures in RL.
Groups may reflect scientifically meaningful feature sets, such as genes
sharing biological functions or pathways; collections of indicators
encoding categorical variables; or learned feature representations from
neural networks. The spline coefficients associated with each state
variable are one instance of this general structure. Treating each group
as a unit can improve value estimation and interpretability by identifying
the feature groups most relevant to policy evaluation.
\end{remark}

%\subsection{Algorithm}
\begin{remark}
 We observe that the objective function in (\ref{dant_est}) is equivalent to 
$$\label{dant_est2}
\begin{aligned}
\min_{\boldsymbol{\tau},\boldsymbol{\alpha},\boldsymbol{\beta}}\sum_{k =1}^{2d}\textbf{1}\T\boldsymbol{\tau}~~~\text{such that }~~~ \|\boldsymbol{\beta}_{G_k}\|_2\leq \tau_k,
\end{aligned}
$$
%$$
%\begin{aligned}
%&\left\|\sum_{i=1}^n\sum_{t=1}\T\boldsymbol{\xi}_{G_0,i,t}\bigg\{R_{i,t}-(\boldsymbol{\xi}_{i,t}-\gamma \boldsymbol{U}_{i,t+1})\T(\boldsymbol{\alpha}\T,\boldsymbol{\beta}\T)\T\bigg\}\right\|_{2}\\&\quad\quad\quad\quad\quad\quad\quad\quad\quad\quad\quad\quad\quad\quad\quad\quad\quad\quad\quad\quad \leqslant \frac{nT\lambda}{\sqrt{L}}\\
%\end{aligned}
%$$
% and for $1\leq k\leq 2d$,
%$$
%\begin{aligned}
%&\left\|
%\sum_{i=1}^n\sum_{t=1}\T\boldsymbol{\xi}_{G_k,i,t}\bigg\{R_{i,t}-(\boldsymbol{\xi}_{i,t}-\gamma \boldsymbol{U}_{i,t+1})\T(\boldsymbol{\alpha}\T,\boldsymbol{\beta}\T)\T\bigg\}\right\|_{2}\\&\quad\quad\quad\quad\quad\quad\quad\quad\quad\quad\quad\quad\quad\quad\quad\quad\quad\quad\quad\quad  \leqslant nT\lambda. \\
%\end{aligned}
%$$
where $\boldsymbol{1} = (1,1,\dots,1)\T \in \mathbb{R}^{2d}$ and $\boldsymbol{\tau} =(\tau_1,\dots,\tau_{2d})\T \in \mathbb{R}^{2d}$. With the same constraints as in (\ref{dant_est}), this becomes a second-order cone programming (SOCP) problem \citep{lobo1998applications, alizadeh2003second}.
%which can be solved using interior-point methods.
%(e.g., the SeDuMi  \citep{sedumi} or ECOS \citep{domahidi2013ecos}).
\end{remark}

\section{FINITE-SAMPLE ERROR BOUNDS}\label{mainsection}

\subsection{Additional Technical Conditions}

The Markov chain $\{X_{0,t}\}_{t\geq 0 }$ is assumed to have a unique invariant distribution with some density function $\mu(\cdot)$. The density function $\mu$ and the initial density  (i.e., density of the initial state) $\mu_0$ are both uniformly bounded away from 0 and $\infty$. This is a standard assumption in the RL literature for studying off-policy estimation.

We next introduce the
restricted minimum eigenvalue condition.
Recall that $S_0$ and $S_1$ denote the set of active features corresponding to $a\in\{0,1\}$, respectively.
Let $\mathcal{K} = S_0\cup S_1$ and denote its cardinality $|\mathcal{K}|= d_0$. We refer to $d_0$, the total number of active features,
as the sparsity size.  Let $\mathcal{J}$ denote the index set corresponding to all active features, and let $\mathcal{J}^C$ denote its complement. Define the set
$$
 \begin{aligned}
    \mathbb{C} = \bigg\{\delta \in \mathbb{R}^{2(dL+1)} : \delta\not= \boldsymbol{0},  \sum_{G_k\subset \mathcal{J}^C, k\not=0} \|\delta_{G_k}\|_2 \leq \sum_{G_k\subset \mathcal{J}}\|\delta_{G_k} \|_2\bigg\}.
 \end{aligned}
 $$    
It can be shown that $(\widehat{\boldsymbol{\alpha}}-\boldsymbol{\alpha}^{*},
\widehat{\boldsymbol{\beta}}-\boldsymbol{\beta}^{*})_{\oplus}$
belongs to the set $\mathbb{C}$ with high probability,
see the online supplement for the proof.

\begin{definition}[Restricted minimum eigenvalue] 
The restricted minimum eigenvalue of a matrix $K$ 
with respect to the set $\mathbb{C}$ is defined as
%$\boldsymbol{\lambda}_{min}^{r}(K,\mathbb{C})$ denotes the with respect %to set $\mathbb{C}$. 
    $
 \lambda_{min}^{r}(K,\mathbb{C}) = \min_{\delta \in \mathbb{C}}\frac{\delta\T K\delta}{\|\delta\|_2^2}.
 $
\end{definition}

Let $\Lambda=\sum_{t=0}^{T-1} \mathbb{E}\left\{\boldsymbol{\xi}_{0, t} \boldsymbol{\xi}_{0, t}^{\top}- \gamma^2 \boldsymbol{u}\left(X_{0, t}, A_{0, t}\right) \boldsymbol{u}^{\top}\left(X_{0, t}, A_{0, t}\right)\right\},$ where \\ $\boldsymbol{u}(x, a)=\mathbb{E}\left\{\boldsymbol{U}\left(X_{0,1}\right) \mid X_{0,0}=x, A_{0,0}=a\right\}$. The restricted eigenvalue condition is stated as follows:
\begin{assumption}\label{a2}
 There exists some $ \bar{c}>0$ such that $\lambda_{min}^{r}(\Lambda,\mathbb{C})\geq  \bar{c}T$. 
 \end{assumption}

 \begingroup
 
\begin{remark}[Coverage interpretation of Assumption~\ref{a2}]
Assumption~\ref{a2} is a restricted eigenvalue condition analogous to
those used in high-dimensional regression
\citep[e.g.,][]{bickel2009simultaneous}. It can also be interpreted as
a feature-level coverage condition for the behavior policy. In linear
off-policy evaluation, coverage is typically characterized through the
covariance matrix of the features observed under the behavior policy
\citep{duan2020minimax,hao2021sparse}. Our condition plays a similar
role but additionally accounts for the target-policy-weighted
next-state features appearing in the Bellman equation. Specifically,
$\boldsymbol U_{0,t+1}$ averages the spline features at the next state
$X_{0,t+1}$ over the possible next actions according to the target
policy, while $\boldsymbol u(X_{0,t},A_{0,t})$ is its conditional
expectation given the current state and action.

Let $\mu_t$ denote the marginal density of $X_{0,t}$ under the behavior
policy and define
$\bar\mu_T(x)=T^{-1}\sum_{t=0}^{T-1}\mu_t(x)$. Then
\[
\frac{\Lambda}{T}
=
\int_{\mathcal X^d}
\sum_{a\in\mathcal A}
\left\{
\boldsymbol{\xi}(x,a)\boldsymbol{\xi}^{\top}(x,a)
-
\gamma^2
\boldsymbol{u}(x,a)\boldsymbol{u}^{\top}(x,a)
\right\}
b(a\mid x)\bar\mu_T(x)\,dx.
\]
Thus, Assumption~\ref{a2} is equivalent to requiring, for every
$\delta\in\mathbb C$,
\[
\mathbb E_{\bar\mu_T,b}
\left[
\{\delta^{\top}\boldsymbol{\xi}(X,A)\}^2
-
\gamma^2
\{\delta^{\top}\boldsymbol{u}(X,A)\}^2
\right]
\geq
\bar c\|\delta\|_2^2.
\]
The behavior policy must therefore generate sufficient variation in
every relevant sparse spline direction after accounting for the
discounted target-policy continuation features. The condition may fail
if the behavior policy rarely selects a relevant action or visits
regions where a relevant basis function varies, or if the current and
continuation features are nearly collinear along a sparse direction.
No separate pointwise overlap or target-policy eigenvalue condition is
imposed; the required coverage is encoded through $\Lambda$ in the
feature directions entering the Bellman equation.

The factor $T$ in Assumption~\ref{a2} reflects that $\Lambda$ sums the
population matrices over $T$ time points:
$
\lambda_{\min}^{r}
\left(
\frac{\Lambda}{T},\mathbb C
\right)
\geq \bar c.
$
The normalized matrix $\Lambda/T$ may nevertheless vary with $T$
through $\bar\mu_T$. If $\bar\mu_T$ converges and the limiting Bellman
matrix has a restricted eigenvalue bounded away from zero, the
assumption holds for all sufficiently large $T$. However, increasing
$T$ cannot remedy a lack of coverage under the long-run behavior
distribution.

Unlike fixed-dimensional analyses, which typically impose an ordinary
minimum eigenvalue condition, we require the bound only over the sparse
cone $\mathbb C$, since the spline-expanded dimension may exceed $nT$.

\end{remark}
\endgroup

Our theory in Section 4.2 implies that the value function of interest can be consistently estimated when either $n$
 or $T$ is large.
    Assumption \ref{a3} below allows us to evaluate the accuracy of $\|\widehat{\boldsymbol{\Sigma}}-\boldsymbol{\Sigma} \|_{max}$ when $T$ is large.  However, it is not needed when $n$ is large.

\begin{assumption}\label{a3}
  The Markov chain $\{X_{0,t}\}_{t\geq 0 }$ is exponentially $\beta$ mixing.  
\end{assumption}

 \begin{remark} 
   When $X_{0,t} $ is stationary, Assumption \ref{a3} can be reduced to the assumption that the Markov chain $X_{0,t}$ is geometrically ergodic \citep{bradley2005basic}. 
   %and is a standard assumption in the RL literature.
   %The proof comes from lemma1 in \citep{meitz2021subgeometric} directly.(\textcolor{red}{ check the proof of . Do we really need $\mu_0 = \mu$?})
  {
 %The mixing condition has a separate role. 
 Assumption~\ref{a3}
does not imply the population coverage condition; rather, it controls the concentration of empirical
Bellman matrices around their population counterparts. Under
exponential $\beta$-mixing,
$\beta(k)\leq C_\beta\rho^k$ for some $\rho\in(0,1)$, the proof uses
blocks of length of order
$\log(nT\vee d)/|\log\rho|$ to make the coupling error negligible.
Thus, observations from a long trajectory contribute information,
although temporal dependence reduces the effective number of
approximately independent observations. 
%The constants deteriorate as
%$\rho$ approaches one. 
}
%Mixing is needed for our guarantees when
%information is obtained by letting $T$ increase, but it is not needed
%for cross-trajectory concentration when $T$ is fixed and
%$n\rightarrow\infty$.}
\end{remark}

\subsection{Finite-sample Error Bound for the Estimated Coefficients } 
Theorem \ref{main1}  below provides high-probability finite error bounds for the proposed estimator $(\widehat{\boldsymbol{\alpha}},\widehat{\boldsymbol{\beta}})_{\oplus}$. 
Let  $\widehat{\boldsymbol{\kappa}} = (\widehat{\boldsymbol{\alpha}}-\boldsymbol{\alpha}^{*},
\widehat{\boldsymbol{\beta}}-\boldsymbol{\beta}^{*})_{\oplus}$.
\begin{theorem}\label{main1} Under Assumption \ref{sparseQ}-\ref{a3}, we have 
    \begin{equation}
        ||{\widehat{\boldsymbol{\kappa}} }||_{1} \lesssim \frac{\log(nT \vee d)}{\sqrt{nT}}*d_0L,
        \end{equation}
        and \begin{equation}||\widehat{\boldsymbol{\kappa}}||_{2} \lesssim  \frac{\log(nT \vee d)}{\sqrt{nT}}*\sqrt{d_0L},
    \end{equation}
    with probability at least $1- \exp\{-c\log(nT \vee d)\}$ for some constant $c>0$.
\end{theorem}
\begingroup

The ambient dimension $d$ enters the
finite-sample error bounds only through the logarithmic factor
$\log(nT\vee d)$. To illustrate this dependence, suppose $d\asymp(nT)^\kappa$ for some fixed
$\kappa>0$. 
%Then $\log(nT\vee d)\asymp\log(nT)$.
Under
$L\asymp(nT)^{1/(2m-1)}$, 
$
\|\widehat{\boldsymbol{\kappa}}\|_1
\lesssim
d_0\log(nT)(nT)^{-\frac{2m-3}{4m-2}}$.
Consequently, if
$d_0=o\{(nT)^{1/(4m-2)}\}$, then
$
\|\widehat{\boldsymbol{\kappa}}\|_1
=
o\left\{
\log(nT)(nT)^{-\frac{m-2}{2m-1}}
\right\}
=
o(1)$
for $m\geq3$. Similarly,
$
\|\widehat{\boldsymbol{\kappa}}\|_2
\lesssim
\sqrt{d_0}\log(nT)(nT)^{-\frac{m-1}{2m-1}}$,
and $\ell_2$ consistency requires only
$
\sqrt{d_0}\log(nT)
=
o\left\{
(nT)^{\frac{m-1}{2m-1}}
\right\}$. The estimator thus remains consistent when the ambient dimension grows at
any fixed polynomial rate, provided that the sparsity size satisfies the
stated growth conditions.
The polynomial-growth regime is used only to illustrate the dimensional
dependence; the theorem applies to any growth rate of $d$ satisfying its
assumptions and yielding vanishing upper bounds. When $d$ and $d_0$ are
fixed, a spectral-norm argument replaces the high-dimensional max-norm
analysis and relaxes Assumption~\ref{assumption_smooth} to $\nu\geq2$,
or equivalently, $m=\lfloor\nu\rfloor\geq2$.

\endgroup

\begingroup
 
\begin{remark}[Approximate group sparsity]
The exact group-sparsity assumption can be relaxed by allowing the
coefficients outside a reference set $\mathcal S$ to be nonzero, provided
that their aggregate group norm
$
R_{\mathcal S}
:=
\sum_{k\notin\mathcal S}
\|\boldsymbol{\beta}_{G_k}^*\|_2
$
is sufficiently small. However, the condition
$R_{\mathcal S}=o(1)$ alone is generally insufficient to retain the
convergence rate obtained under exact sparsity. In particular, the
approximation tail contributes additional terms to the coefficient error,
and translating this error into value-function error introduces an
additional factor of $\sqrt{L}$. Consequently, for fixed $d_0$ and
$m\geq3$, these additional terms are of no larger order than the
exact-sparsity error if, up to logarithmic factors,
$
R_{\mathcal S}
=
O\left(
(nT)^{-\frac{m-1}{2m-1}}
\right)$ and
$L\asymp(nT)^{1/(2m-1)}$.
Thus, approximate sparsity can preserve the exact-sparsity convergence rate,
but doing so requires a quantitative decay rate for the coefficient tail
that is substantially stronger than $R_{\mathcal S}=o(1)$.
\end{remark}
\endgroup

\subsection{Finite-sample Error Bound of Value Function}
Given the target policy $\pi$, 
we estimate the expected cumulative discounted reward $V^{\pi}(x)$ by  $\widehat{V}^{\pi}(x) = \mathbb{E}^{\pi}\{\Phi(x)^T(\widehat{\alpha}_a,\widehat{\boldsymbol{\beta}}_{1,a},\dots,\widehat{\boldsymbol{\beta}}_{d,a})_{\oplus}\}.$
The next theorem provides a finite-sample error bound for $\widehat{V}^{\pi}(x)$.

\begin{theorem}\label{main2}
    Assume the conditions in Theorem  \ref{main1} are satisfied. With probability at least $1- \exp\{-c\log(nT \vee d)\}$, 
    we have
\begin{equation}
        \sup_{x\in \mathcal{X}^d }|\widehat{V}^{\pi}(x)-V^{\pi}(x)|\lesssim d_0L\frac{\log(nT \vee d)}{\sqrt{nT}}.
\end{equation}
If the density of distribution $\mathbb{G}$ is upper bounded from infinity, then for a deterministic policy $\pi$, we have 
\begin{equation}
        |\mathbb{E}_{x\sim \mathbb{G}}\widehat{V}^{\pi}(x)-\mathbb{E}_{x\sim \mathbb{G}}V^{\pi}(x)|\lesssim d_0\sqrt{L}\frac{\log(nT \vee d)}{\sqrt{nT}}.
\end{equation}
\end{theorem}
 {The proof of the theorem relies on the following decomposition
\[
\big|
\widehat V^\pi(x)-V^\pi(x)
\big|
\leq
\big|
\widehat V^\pi(x)-V_\pi^*(x)
\big|
+
\big|
V_\pi^*(x)-V^\pi(x)
\big|,
\]
where $V_{\pi}^{*}(x)=
\sum_{a\in\mathcal A}
\left\{\alpha_a^*+
\sum_{k=1}^d\Phi_k^\top(x_k)\boldsymbol{\beta}_{k,a}^*\right\}\pi(a|x)$, 
and the two terms on the right-hand side represent the estimation error
and the spline approximation error, respectively.
Standard B-spline
approximation results imply that
$\sup_{x\in\mathcal X^d}
\big|V^\pi(x)-V_\pi^*(x)\big|
\lesssim d_0L^{-m}$.
Theorem~\ref{main1} is then used to control the estimation error. In
particular, it yields the uniform bound
$\sup_{x\in\mathcal X^d}\big|
\widehat V^\pi(x)-V_\pi^*(x)
\big|\lesssim
d_0L\frac{\log(nT\vee d)}{\sqrt{nT}}$.
For the integrated value error, when $\mathbb G$ has a uniformly bounded
density and $\pi$ is deterministic, the integral properties of the
B-spline basis yield the sharper bound
$
\big|
\mathbb E_{x\sim\mathbb G}\widehat V^\pi(x)
-
\mathbb E_{x\sim\mathbb G}V_\pi^*(x)
\big|
\lesssim
d_0\sqrt L\frac{\log(nT\vee d)}{\sqrt{nT}}.
$
Therefore,
\begin{equation}
\big|
\mathbb{E}_{x\sim\mathbb{G}}\widehat{V}^{\pi}(x)
-
\mathbb{E}_{x\sim\mathbb{G}}V^{\pi}(x)
\big|
\lesssim
d_0\sqrt{L}\frac{\log(nT\vee d)}{\sqrt{nT}}
+
d_0L^{-m}.
\end{equation}
For $L\asymp(nT)^{1/(2m-1)}$, the approximation error is
asymptotically negligible relative to the integrated estimation error
since
$
\frac{d_0L^{-m}}
{d_0\sqrt L\log(nT\vee d)/\sqrt{nT}}
%\asymp
%\frac{(nT)^{-1/(2m-1)}}{\log(nT\vee d)}
=
o(1).
$
The same conclusion holds for the uniform error bound,
whose estimation component is larger by a factor of $\sqrt L$. 
%Indeed,
%\[
%\frac{d_0L^{-m}}
%{d_0L\log(nT\vee d)/\sqrt{nT}}
%\asymp
%\frac{(nT)^{-3/\{2(2m-1)\}}}
%{\log(nT\vee d)}
%=
%o(1).
%\]
Thus, the spline approximation error can be absorbed into the respective
estimation errors, yielding the simplified bounds stated in
Theorem~\ref{main2}.}

%Under the stated choice $L\asymp(nT)^{1/(2m-1)}$, the functional
%approximation error is of smaller order than the estimation error, since
%\[
%\frac{d_0L^{-m}}
%{d_0\sqrt{L}\log(nT\vee d)/\sqrt{nT}}
%\asymp
%\frac{(nT)^{-1/(2m-1)}}{\log(nT\vee d)}
%=o(1).
%\]
%The corresponding requirement for the uniform error bound is weaker, since
%its estimation-error term
%$d_0L\log(nT\vee d)/\sqrt{nT}$ is larger by a factor of $\sqrt{L}$. Therefore, the functional approximation error can be absorbed into the
%estimation error, yielding the simplified bound stated in
%Theorem~\ref{main2}. 

\begin{remark}
 {Our result differs from the high-probability guarantee of
\citep[Theorem~4.8]{hao2021sparse} in how the trajectory length affects the confidence
level. In their notation, the data consist of $N=KL$ observations from
$K$ independent trajectories, each of length $L$. Their sample-size
condition requires
$N \gtrsim\frac{L s^2 \log(d^2/\delta)}
{C_{\min}^2(\Sigma,s)}
+
\frac{\gamma^2 Ls\log(s/\delta)}
{(1-\gamma)^2}$.
Because $N=KL$, this condition is equivalent to a lower bound on the
number of independent trajectories,
$
K
\gtrsim
\frac{s^2 \log(d^2/\delta)}
{C_{\min}^2(\Sigma,s)}
+
\frac{\gamma^2 s\log(s/\delta)}
{(1-\gamma)^2}.
$
Consequently, under their stated sufficient conditions, increasing the
trajectory length $L$ while holding $K$ fixed does not permit the failure
probability $\delta$ to converge to zero, even though the error bound on
the corresponding good event decreases with the total sample size. In
contrast, 
%the event in Theorem~\ref{main2} has probability at least
%$1-\exp\{-c\log(nT\vee d)\}$, which converges to one when $n$ is fixed
%and $T\to\infty$. Under the stated growth conditions, the error bound on
%this event also converges to zero. 
our result establishes
long-trajectory consistency both in the magnitude of the estimation
error and in the confidence level of the guarantee.}
\end{remark}

\subsection{Feature Screening}
Next, we propose a screening procedure to assess feature relevance. The procedure explicitly incorporates the group structure: if a feature is irrelevant, the entire group of coefficients associated with its basis functions is expected to be screened out, leading to more interpretable results. We define the set of selected features as 
\begin{equation}\label{var_selection}
    \widehat{\mathcal{K}} = \{1\leq s \leq d: \|\widehat{\boldsymbol{\beta}}_{G_s}\|_2 > h\quad \text{or}\quad \|\widehat{\boldsymbol{\beta}}_{G_{s+d}}\|_2 > h\},
\end{equation}
where $h$ is a predefined threshold. 
The effectiveness of this screening step relies on the marginal signal of active components not being dominated by estimation error, which vanishes under suitable conditions implied by our earlier error bounds. The following theorem provides a theoretical guarantee for the proposed screening procedure.

\begin{theorem}\label{main3}
    Suppose the conditions in Theorem (\ref{main1}) are satisfied. If there exists a constant $v$ such that $\max_{a\in \mathcal{A}}\mathbb{E}_{x\sim \mathbb{G}}(q_{s,a}^{\pi2}(x)) \geq v$ for all $s \in \mathcal{K}$, then setting $h = \sqrt{v/(8c_m^2L)}$, we have
$P(\mathcal{K}\subseteq \widehat{\mathcal{K}}) \geq 1- \exp\{-c\log(nT \vee d)\}$.
    %\end{equation}
\end{theorem}
\enlargethispage{2\baselineskip}
\section{NUMERICAL EXAMPLES}
We compare our proposed estimator
(denoted by GD, due to the group Dantzig strcuture used in estimation)  with the SAVE estimator from \citet{shi2022statistical} and  {Lasso-FQE (denoted by LFQ)} from \citet{hao2021sparse}. All results are summarized in Appendix~\ref{app:numerical}.
\subsection{Experiment 1}
In this model, we generate data from a Markovian decision process described below. 
The system evolves according to the following dynamics: 
\begin{align*}
    &x_{t+1,1:d} = A x_{t,1:d} + B a_t + \varepsilon_t, 
    && \varepsilon_t \sim \mathcal{N}(0, \sigma_{\text{trans}}^2 I), \\
    &r_t = w\T x_{t,1:4} + {\rho\bigl(x_{t,1}x_{t,3}+x_{t,2}x_{t,4}\bigr)} - c a_t + \eta_t, 
    && \eta_t \sim \mathcal{N}(0, \sigma_{\text{reward}}^2).
\end{align*}
The transition matrix $A$ is block diagonal, consisting of $d/2$ blocks of size $2 \times 2$, so that every pair of adjacent features is coupled. The elements within each $2 \times 2$ block are sampled from the interval $[0.80, 0.90]$ to ensure heterogeneity among features. The $d$-dimensional vector $B$ satisfies $\|B\|_{\infty} \leq 0.2$ and ensures that action can affect each state coordinate. The reward function is linear in the first four features $x_{1:4}$ and the action $a$ with coefficients $(w,c) =(0.9, -0.7, 0.9, 0.6,0.2)$, $\rho = 0.7$ and $\sigma_{\text{reward}} = 0.1$. The behavior policy belongs to the logistic policy class and is defined over the first $4$ features with coefficients $(1.2, -0.9, 0.6, 0.7)$. The target policy follows same class with with coefficients $(1, -0.5,1.5,-1)$. 

We consider a small state with $d =10$ and a larger state $d = 100 \gg \max\{n,T\}$, and  {vary the number of
independent trajectories and the trajectory length over
$n\in\{20,30,40,50\}$ and
$T\in\{20,30,40\}$}.
\begingroup   We set $\lambda = 0.001$, $L=8$ and $\gamma = 0.8$, and report the results based on 100 independent runs. In each run, we evaluate the bias of $\widehat{V}(x)$ at $x=0$, that is,
\(
\left[\widehat{V}(x) - V_{\text{MC}}(x)\right]\big|_{x = 0},
\)
where $V_{\text{MC}}(x)$ denotes the Monte Carlo estimate of the cumulative reward over a horizon of 100, averaged over 1000 repetitions. Tables~\ref{tab:d10_rho07} and~\ref{tab:d100_rho07} summarize the
results for $d=10$ and $100$, respectively. For $d=10$, GD remains nearly
unbiased across all combinations of $(n,T)$, with its RMSE generally
decreasing as either $n$ or $T$ increases. Its performance is comparable to
that of the other methods in this setting.
For $d=100$, GD continues to produce
small bias and low variability over the entire $(n,T)$ grid. Its RMSE 
generally improves as the available
sample size increases. Lasso-FQE remains numerically stable, but has a
substantial positive bias for shorter trajectories. However, SAVE is substantially less stable in the high-dimensional
setting. 
%Although it performs well for several combinations of $(n,T)$, 
It
occasionally produces extreme estimates, as reflected by the large standard
deviations and RMSEs at $(n,T)=(30,40)$, $(40,30)$, and $(50,30)$. 

Overall, GD provides the most
reliable performance across both the low-dimensional and high dimensional cases.

\endgroup

\subsection{Experiment 2: Cartpole}
We consider the CartPole environment from OpenAI Gym and manually incorporated null variables into the state, with $d_0 = 4$ and $d = 80$. The null variables were generated from a standard normal distribution independently. To encourage the pole to remain balanced for as long as possible, we set the discount rate $\gamma = 0.98$. The target policy selects \( A = 1 \) with probability $0.5$. 
%a choice commonly used in randomized controlled trials,
 The behavior policy also from logistic policy class and designed as (Push right if the pole is leaning to right.):
$
\pi(1|x)= \frac{1}{1+e^{-(x_3+x_4)}}.
$

We evaluate $V(x)$ at $x = 0$.
The true value function $V^{\pi}(x)|_{x=0}\approx 33.50$ is computed using Monte Carlo approximations with $10^5$ independent trajectories. Each trajectory consists of $T = 500$ decision points. We set $\lambda = 0.001$ and $L =8$.
\begingroup
 
The results are given in the online supplement to save space. We observe that GD
performs satisfactorily across the considered combinations of $(n,T)$.  Lasso-FQE exhibits a systematic negative bias when $T$ is
small, but this bias decreases substantially as the trajectory length
increases. SAVE is still substantially less stable. Although it performs well for some
settings, its standard deviation and RMSE can be very large, with particularly
extreme behavior at $(n,T)=(80,50)$, where the RMSE reaches $357.95$. 

Overall, GD delivers the most stable performance across the reported configurations, even though the true Q-function does not satisfy a purely additive structure.
\endgroup

\appendix

\section{Discussions on Sparsity Assumptions }\label{app_sparse}

\subsection{Sparse Q Function and Sparse Completeness}
Recall the sparse completeness. Let $\mathcal{G}_{\mathcal{K}} = \{f:\mathcal{X}^d \times \mathcal{A} \rightarrow \mathbb{R}|~ f(x,a) = f(x_{\mathcal{K}},a)\}$. Given the target policy $\pi$, define the following operator,
\begin{equation*}
\Gamma^{\pi}f(x,a) = r(x,a) +\gamma \mathbb{E}^{\pi}[f(x^{\prime},a^{\prime})|x,a].    
\end{equation*}

\begin{assumption}{\textbf{sparse completeness.}}
    Given the target policy $\pi$, there exists an active set $\mathcal{K} \subseteq [d], |\mathcal{K}|\leq d_0$ such that $\mathcal{G}_{\mathcal{K}}$ is closed under the operator $\Gamma^{\pi}$. 
\end{assumption}
The sparse completeness of target policy $\pi$ guarantees the $\mathcal{G}_{\mathcal{K}}$ can capture the operator $\Gamma^{\pi}$. Note that $f\equiv 0 $ belongs 
 to $\mathcal{G}_{\mathcal{K}}$ for all $\mathcal{K}$, the poilcy sparse completeness implies that $r \in \mathcal{G}_{\mathcal{K}}$. As such, $\mathbb{E}^{\pi}[f(x^{\prime},a^{\prime})|x,a] \in \mathcal{G}_{\mathcal{K}}$ if $f \in \mathcal{G}_{\mathcal{K}}$.
  Denote $r^{(t)}(x,a) = \mathbb{E}^\pi\left(R_{0, t} \mid X_{0,0}=x,A_{0,0} = a\right)$, $r^{(0)}(x,a) = r(x,a)$. And hence,
    \begin{equation}
        Q^{\pi}(x,a) = \sum_{t\geq 0} \gamma^{t}r^{(t)}(x,a).
    \end{equation}
To prove $Q \in \mathcal{G}_{\mathcal{K}}$, we only  need to prove 
\begin{equation}
    r^{(t)}(x,a) = r^{(t)}(x_{\mathcal{K}},a)
\end{equation}
%\newpage

By CMIA,
\begin{equation}\label{26}
    \begin{aligned}
      r^{(t)}(x,a) &= \mathbb{E}^\pi\left(R_{0, t} \mid X_{0,0}=x,A_{0,0} = a\right)\\ &=  \mathbb{E}^\pi[\mathbb{E}^\pi\left(R_{0, t} \mid X_{0,1}=x^{\prime},A_{0,1} = a^{\prime}\right)|X_{0,0}=x,A_{0,0} = a]\\
      &=\mathbb{E}^\pi[r^{(t-1)}(x^{\prime},a^{\prime})|X_{0,0}=x,A_{0,0} = a]\\
    \end{aligned}
\end{equation}
Note that $r^{(0)}(x,a) = r(x,a) \in \mathcal{G}_{\mathcal{K}}$. If $r^{(t-1)}(x,a) \in \mathcal{G}_{\mathcal{K}}$, $r^{(t)}(x,a) = \mathbb{E}^{\pi}[r^{(t-1)}(x^{\prime},a^{\prime})|X_{0,0}=x,A_{0,0} = a] \in \mathcal{G}_{\mathcal{K}}$. $Q \in \mathcal{G}_{\mathcal{K}}$ is proved by induction.
We proved the sparse completeness implied Assumption \ref{sparseQ} directly.  
\subsection{Sufficient State}
\begin{definition}{\textbf{Sufficient State.} \cite{ma2023sequential}} 
 We say the sufficient state $\mathcal{X}_{\mathcal{K}}$ in an MDP if $R_t \perp\!\!\!\perp$ $X_{t, \mathcal{K}^c} \mid\left(X_{t, \mathcal{K}}, A_t\right)$ and $X_{t+1,\mathcal{K}} \perp\!\!\!\perp X_{t,\mathcal{K}^c} \mid\left(X_{t,\mathcal{K}}, A_t\right)$, for all $t \geq 0$.
\end{definition}
If we further consider a online learning problem, with the definition of sufficient state, \cite{ma2023sequential} prove the following the proposition,
\begin{proposition}If $\mathcal{X}_{\mathcal{K}}$ is a sufficient state, then there exists an optimal policy that depends
only on $\mathcal{X}_{\mathcal{K}}$.
\end{proposition}
This proposition implies that the optimal policy $\pi^*$ only depend on the sufficient state $\mathcal{X}_{\mathcal{K}}$, and allows us to only focus on policies that are functions of sufficient states.  \\

\subsection{Proof of Proposition~\ref{condition_sparseQ}}

\begin{proof}
For an index set $\mathcal K\subseteq[d]$, define
\[
\mathcal G_{\mathcal K}
=
\left\{
f:\mathcal X^d\times\mathcal A\to\mathbb R:
f(x,a)=\widetilde f(x_{\mathcal K},a)
\text{ for some bounded function }\widetilde f
\right\}.
\]
Thus, $\mathcal G_{\mathcal K}$ is the class of bounded functions that
depend on the state only through $x_{\mathcal K}$.

For any bounded function $f$, define the target-policy transition
operator
\[
(\mathcal P^\pi f)(x,a)
=
\int_{\mathcal X^d}
\sum_{a'\in\mathcal A}
f(x',a')\pi(a'\mid x')P(dx'\mid x,a),
\]
where $P(\cdot\mid x,a)$ denotes the transition kernel. The corresponding
Bellman operator is
\[
(\Gamma^\pi f)(x,a)
=
r(x,a)+\gamma(\mathcal P^\pi f)(x,a).
\]
We show that, under either sufficient condition,
$\Gamma^\pi$ maps $\mathcal G_{\mathcal K}$ into itself.

First, suppose that Sufficient Condition 1 holds. Let $x,\bar x\in
\mathcal X^d$ satisfy $x_{\mathcal K}=\bar x_{\mathcal K}$. Because the
full transition kernel depends on the current state only through
$x_{\mathcal K}$, we have
$P(\cdot\mid x,a)=P(\cdot\mid\bar x,a)$.
Therefore, for any bounded function $f$,
\[
\begin{aligned}
(\mathcal P^\pi f)(x,a)
&=
\int_{\mathcal X^d}
\sum_{a'\in\mathcal A}
f(x',a')\pi(a'\mid x')P(dx'\mid x,a)
\\
&=
\int_{\mathcal X^d}
\sum_{a'\in\mathcal A}
f(x',a')\pi(a'\mid x')P(dx'\mid\bar x,a)
\\
&=
(\mathcal P^\pi f)(\bar x,a).
\end{aligned}
\]
Thus, $(\mathcal P^\pi f)(x,a)$ depends on the current state only through
$x_{\mathcal K}$. Since
$r(x,a)=r(x_{\mathcal K},a)$, it follows that
$\Gamma^\pi f\in\mathcal G_{\mathcal K}$ whenever
$f\in\mathcal G_{\mathcal K}$.

Next, suppose that Sufficient Condition 2 holds. Let
$f\in\mathcal G_{\mathcal K}$, so that
$f(x',a')=\widetilde f(x'_{\mathcal K},a')$ for some bounded function
$\widetilde f$. Let $P_{\mathcal K}(\cdot\mid x,a)$ denote the marginal
transition kernel of $X'_{\mathcal K}$. Because the target policy depends
only on the active coordinates, we have
$\pi(a'\mid x')=\pi(a'\mid x'_{\mathcal K})$.
Consequently,
\[
\begin{aligned}
(\mathcal P^\pi f)(x,a)
&=
\int_{\mathcal X^{|\mathcal K|}}
\sum_{a'\in\mathcal A}
\widetilde f(x'_{\mathcal K},a')
\pi(a'\mid x'_{\mathcal K})
P_{\mathcal K}(dx'_{\mathcal K}\mid x,a).
\end{aligned}
\]
The assumed reduced transition property implies that
$P_{\mathcal K}(\cdot\mid x,a)
=
P_{\mathcal K}(\cdot\mid\bar x,a)$,
whenever 
$x_{\mathcal K}=\bar x_{\mathcal K}$.
Hence,
\[
(\mathcal P^\pi f)(x,a)
=
(\mathcal P^\pi f)(\bar x,a),
\]
so that $\mathcal P^\pi f\in\mathcal G_{\mathcal K}$. Together with
$r\in\mathcal G_{\mathcal K}$, this gives
$\Gamma^\pi f\in\mathcal G_{\mathcal K}$.

Thus, under either sufficient condition, $\mathcal G_{\mathcal K}$ is
closed under the Bellman operator $\Gamma^\pi$. To complete the argument,
let $f_0=0$ and define the Bellman iterates
$f_{m+1}=\Gamma^\pi f_m$, $m\geq0$.
Because $f_0\in\mathcal G_{\mathcal K}$ and $\mathcal G_{\mathcal K}$ is
closed under $\Gamma^\pi$, we have
$f_m\in\mathcal G_{\mathcal K}$ for every $m$. Moreover, since
$\gamma<1$,
\[
\|\Gamma^\pi f-\Gamma^\pi g\|_\infty
\leq
\gamma\|f-g\|_\infty.
\]
Therefore, the Bellman iterates converge uniformly to the unique bounded
fixed point of $\Gamma^\pi$.
Since $Q^\pi$ is a fixed point, 
the limit must equal $Q^\pi$.
As $\mathcal G_{\mathcal K}$ is closed under uniform limits, it follows that
$Q^\pi\in\mathcal G_{\mathcal K}$. Equivalently,
\[
Q^\pi(x,a)=Q^\pi(x_{\mathcal K},a)
\qquad
\text{for every }(x,a)\in\mathcal X^d\times\mathcal A.
\]
Taking $S_a=\mathcal K$ for every $a\in\mathcal A$ and using
$|\mathcal K|\ll d$ establishes Assumption~\ref{sparseQ}.
\end{proof}

\section{Properties of B-spline basis functions}
\label{Bspline}

In this section, we summarize several properties of the B-spline basis
functions used in our analysis. Let $L$ denote the number of basis functions,
and let $m$ denote the order of the B-splines, where $m$ is fixed and chosen
according to the smoothness parameter $v$. Let
\[
\boldsymbol b(x)
=
\bigl(b_1(x),\ldots,b_L(x)\bigr)^\top
\]
be the normalized B-spline basis associated with a knot sequence
$\{\tau_j\}$ \citep{deboor1976}. We augment the knot sequence at the two
boundaries such that
\[
\tau_0=\tau_1=\cdots=\tau_m
<
\tau_{m+1}
<
\cdots
<
\tau_{L-1}
<
\tau_L=\tau_{L+1}=\cdots=\tau_{L+m}.
\]
We further assume that the distinct knots are equally spaced, namely,
\[
\tau_{m+1}-\tau_m
=
\cdots
=
\tau_L-\tau_{L-1}.
\]
Define
\[
d_j=\frac{\tau_{j+m}-\tau_j}{m}
\]
for all admissible $j$. Since $m$ is fixed and the knots are equally spaced,
we have $d_j\asymp L^{-1}$ uniformly in $j$.

Throughout the proofs, we use the rescaled B-spline basis
\[
B_j(x_{i,t})=\sqrt{L}\,b_j(x_{i,t}),
\qquad j=1,\ldots,L,\quad t\geq 1.
\]

\begin{lemma}
\label{bs_bound_uni}
For the knot sequence defined above and any $x\in\mathcal{X}$,
\[
\sum_{j=1}^{L} b_j(x)=1
\qquad\text{and}\qquad
0\leq b_j(x)\leq 1,\quad j=1,\ldots,L.
\]
Moreover, there exists a constant $c_m>0$, depending only on the order $m$,
such that
\[
\sup_{x\in\mathcal{X}}
\|\boldsymbol b(x)\|_2
\leq c_m.
\]
\end{lemma}

\begin{proof}
The partition-of-unity property
\[
\sum_{j=1}^{L}b_j(x)=1
\]
follows from Marsden's identity \citep{deboor1976}. In addition, the
B-spline basis functions are nonnegative, and therefore
$0\leq b_j(x)\leq 1$ for every $j$.

Since B-splines of order $m$ have local support, at most $m$ components of
$\boldsymbol b(x)$ are nonzero at any fixed $x$. It follows that
\[
\|\boldsymbol b(x)\|_2^2
=
\sum_{j=1}^{L}b_j(x)^2
\leq c_m.
\]
This bound is purely a
property of the B-spline basis and does not depend on the distribution of
$X$.
\end{proof}
\begin{lemma}
\label{Bs_bound_uni}
There exists a constant $c^*>0$, independent of $L$, such that
\[
\sup_{x\in\mathcal{X}}
\|\boldsymbol B(x)\|_{\infty}
\leq \sqrt{L},
\qquad
\sup_{x\in\mathcal{X}}
\|\boldsymbol B(x)\|_2
\leq c_m\sqrt{L},
\]
and
\begin{equation}
\label{eq:Bspline_gram_bound}
(c^*)^{-1}
\leq
\lambda_{\min}
\left\{
\int_{\mathcal{X}}
\boldsymbol B(x)\boldsymbol B(x)^\top\,\mathrm{d}x
\right\}
\leq
\lambda_{\max}
\left\{
\int_{\mathcal{X}}
\boldsymbol B(x)\boldsymbol B(x)^\top\,\mathrm{d}x
\right\}
\leq
c^*.
\end{equation}
\end{lemma}

\begin{proof}
By Lemma~\ref{bs_bound_uni},
\[
\sup_{x\in\mathcal{X}}
\|\boldsymbol b(x)\|_{\infty}
\leq 1
\qquad\text{and}\qquad
\sup_{x\in\mathcal{X}}
\|\boldsymbol b(x)\|_2
\leq c_m.
\]
Since $\boldsymbol B(x)=\sqrt{L}\boldsymbol b(x)$, the first two
claims follow immediately.

It remains to establish the eigenvalue bounds. A standard stability property
of B-spline bases implies that there exist constants $c_1,c_2>0$, depending
only on the order $m$, such that, for every
$\boldsymbol c=(c_1,\ldots,c_L)^\top\in\mathbb{R}^L$,
\begin{equation}
\label{eq:Bspline_stability}
c_1\sum_{j=1}^L d_jc_j^2
\leq
\left\|
\sum_{j=1}^L c_jb_j
\right\|_{L_2(\mathcal{X})}^2
\leq
c_2\sum_{j=1}^L d_jc_j^2;
\end{equation}
see, for example, \citet{deboor1976}. Because the knots are equally spaced
and $m$ is fixed, $d_j\asymp L^{-1}$ uniformly in $j$. Hence, there exist
constants $\underline c,\overline c>0$, independent of $L$, such that
\[
\underline cL^{-1}\|\boldsymbol c\|_2^2
\leq
\left\|
\sum_{j=1}^L c_jb_j
\right\|_{L_2(\mathcal{X})}^2
\leq
\overline cL^{-1}\|\boldsymbol c\|_2^2.
\]
Equivalently,
\[
\underline cL^{-1}\|\boldsymbol c\|_2^2
\leq
\boldsymbol c^\top
\left\{
\int_{\mathcal{X}}
\boldsymbol b(x)\boldsymbol b(x)^\top\,\mathrm{d}x
\right\}
\boldsymbol c
\leq
\overline cL^{-1}\|\boldsymbol c\|_2^2.
\]
Therefore,
\[
\underline cL^{-1}
\leq
\lambda_{\min}
\left\{
\int_{\mathcal{X}}
\boldsymbol b(x)\boldsymbol b(x)^\top\,\mathrm{d}x
\right\}
\leq
\lambda_{\max}
\left\{
\int_{\mathcal{X}}
\boldsymbol b(x)\boldsymbol b(x)^\top\,\mathrm{d}x
\right\}
\leq
\overline cL^{-1}.
\]
Finally, since $\boldsymbol B(x)=\sqrt{L}\boldsymbol b(x)$,
\[
\int_{\mathcal{X}}
\boldsymbol B(x)\boldsymbol B(x)^\top\,\mathrm{d}x
=
L
\int_{\mathcal{X}}
\boldsymbol b(x)\boldsymbol b(x)^\top\,\mathrm{d}x.
\]
The desired result follows by taking
$c^*=\max\{\underline c^{-1},\overline c\}$.
\end{proof}

\begin{remark}
\label{basis_bound}
In the main paper, we use the augmented basis vector
\[
\boldsymbol\Phi(x)
=
\left(
1,
\boldsymbol\Phi_1(x)^\top,
\ldots,
\boldsymbol\Phi_d(x)^\top
\right)^\top,
\]
where, for each $k=1,\ldots,d$,
\[
\boldsymbol\Phi_k(x)
=
\boldsymbol B_k(x_k)
=
\left(
B_{k,1}(x_k),
\ldots,
B_{k,L}(x_k)
\right)^\top.
\]
For each state variable, the corresponding B-spline basis is constructed
using a deterministic sequence of equally spaced knots over its support.
Therefore, by Lemma~\ref{Bs_bound_uni},
\[
\sup_{x\in\mathcal{X}^d}
\|\boldsymbol\Phi(x)\|_{\infty}
\leq \sqrt{L},
\]
and
\[
\max_{1\leq k\leq d}
\sup_{x\in\mathcal{X}^d}
\|\boldsymbol\Phi_k(x)\|_2
\leq c_m\sqrt{L}.
\]
\end{remark}

\section{ Additional technical details
}

\begin{lemma}\label{eigen_sigma}
    Under assumption (A2) and (A3), $\boldsymbol{\lambda_{min}^r}(\boldsymbol{\Sigma})\geq \Bar{c}/2$.
\end{lemma}
\begin{proof}
    Denote the marginal distribution of $X_{0,t}$ by $\mu_t$ and the average distribution $\Bar{\mu} = \sum_{t =0 }^{T-1}\mu_t/T$.  By the definition of $\boldsymbol{\Sigma}$, we have
    \begin{equation}
    \begin{aligned}
             \boldsymbol{\Sigma} &= \mathbb{E} \bigg\{\frac{1}{T}\sum_{t = 0}^{T-1}(\boldsymbol{\xi}_{0,t}(\boldsymbol{\xi}_{0,t}-\gamma U_{0,t+1})^T)\bigg\}\\
             &= \int_{x\in \mathcal{X}^d}\sum_{a \in \mathcal{A}}\boldsymbol{\xi}(x,a)(\boldsymbol{\xi}(x,a)-\gamma \mathbb{E}[U(X_{0,1})|X_{0,0} = x, A_{0,0}=a] )^Tb(a|x)\Bar{\mu}(x)dx\\
             &= \int_{x\in \mathcal{X}^d}\sum_{a \in \mathcal{A}}\boldsymbol{\xi}(x,a)(\boldsymbol{\xi}(x,a)-\gamma \boldsymbol{u}(x,a) )^Tb(a|x)\Bar{\mu}(x)dx\\
             &= \int_{x\in \mathcal{X}^d}\sum_{a \in \mathcal{A}}[\boldsymbol{\xi}(x,a)\boldsymbol{\xi}(x,a)^T-\gamma\boldsymbol{\xi}(x,a) \boldsymbol{u}(x,a)^T] b(a|x)\Bar{\mu}(x)dx\\
    \end{aligned}
    \end{equation}  
For $\forall \boldsymbol{a} \in \mathbb{R}^{2(1+dL)}$, denote
\begin{equation}
\begin{aligned}
      &\eta_1(\boldsymbol{a}) = \int_{x\in \mathcal{X}^d}\sum_{a \in \mathcal{A}}\{\boldsymbol{a}^T\boldsymbol{\xi}(x,a)\}^2b(a|x)\Bar{\mu}(x)dx \\
&\eta_2(\boldsymbol{a}) = \int_{x\in \mathcal{X}^d}\sum_{a \in \mathcal{A}}\{\boldsymbol{a}^Tu(x,a)\}^2b(a|x)\Bar{\mu}(x)dx \\
\end{aligned}
\end{equation}
It follows from Cauchy-Schwarz inequality that
\begin{equation}
\int_{x\in \mathcal{X}^d}\sum_{a \in \mathcal{A}}[\boldsymbol{a}^T\boldsymbol{\xi}(x,a)\boldsymbol{u}(x,a)^T\boldsymbol{a} ]b(a|x)\Bar{\mu}(x)dx \leq \eta_1(\boldsymbol{a})^{1/2}\eta_2(\boldsymbol{a})^{1/2}
\end{equation}
Therefore
\begin{equation}
    \boldsymbol{a}^T\boldsymbol{\Sigma} \boldsymbol{a} \geq \eta_1(\boldsymbol{a})-\gamma\eta_1(\boldsymbol{a})^{1/2}\eta_2(\boldsymbol{a})^{1/2} = \eta_1(\boldsymbol{a})^{1/2}(\frac{\eta_1(\boldsymbol{a})-\gamma^2\eta_2(\boldsymbol{a})}{\eta_1(\boldsymbol{a})^{1/2}+\gamma \eta_2(\boldsymbol{a})^{1/2}})
\end{equation}
Under the assumption (A3), if $\boldsymbol{a} \in \mathbb{C}$, $\eta_1(\boldsymbol{a})-\gamma^2\eta_2(\boldsymbol{a}) \geq \bar{c}\|\boldsymbol{a}\|_2^2 \geq 0$, $\eta_1(\boldsymbol{a})^{1/2}\geq \gamma \eta_2(\boldsymbol{a})^{1/2}$. We have
\begin{equation}
    \boldsymbol{a}^T\boldsymbol{\Sigma} \boldsymbol{a} \geq \frac{\bar{c}}{2}\|\boldsymbol{a}\|_2^2 ~~~ \forall \boldsymbol{a}\in \mathbb{C}
\end{equation}
\end{proof}

\begin{lemma}\label{feasiable_dantzig}
There exists a good event $\mathcal{G}_1$ satisfying
\[
P(\mathcal{G}_1)
\geq
1-\exp\{-c_2\log(nT\vee d)\}
\]
such that, on $\mathcal{G}_1$,
$(\boldsymbol{\alpha}^{*},\boldsymbol{\beta}^{*})_{\oplus}$ given by
\eqref{funcional_est} is feasible for the optimization problem
\eqref{dant_est} with
\[
\lambda
=
c_1\sqrt{L}\frac{\log(nT\vee d)}{\sqrt{nT}},
\]
where $c_1$ and $c_2$ are positive constants.
\end{lemma}
\begin{proof}
Throughout this proof, $C_1,C_2,\ldots$ denote positive constants
that are local to the present proof and are re-indexed from $C_1$.
They are distinct from the lowercase constants
$c_1,c_2,\ldots$ appearing in the statements of the results.

Define
\begin{equation}
\varepsilon_{i,t}
=
R_{i,t}
+
\gamma
\sum_{a\in\mathcal A}
Q^\pi(X_{i,t+1},a)
\pi(a\mid X_{i,t+1})
-
Q^\pi(X_{i,t},A_{i,t}),
\end{equation}
and
\begin{equation}
\begin{aligned}
e_{i,t}
={}&
\gamma
\sum_{a\in\mathcal A}
\bigg\{
Q^\pi(X_{i,t+1},a)
-\alpha_a^*
-\sum_{k=1}^d
\Phi_k^T(X_{k,i,t+1})
\boldsymbol{\beta}_{k,a}^*
\bigg\}
\pi(a\mid X_{i,t+1})
\\
&-
\bigg\{
Q^\pi(X_{i,t},A_{i,t})
-\alpha_{A_{i,t}}^*
-\sum_{k=1}^d
\Phi_k^T(X_{k,i,t})
\boldsymbol{\beta}_{k,A_{i,t}}^*
\bigg\}.
\end{aligned}
\end{equation}

Under the smoothness condition on $q^\pi_{k,a}(\cdot)$ and the sparsity
assumption, $Q^\pi(x,a)$ is uniformly bounded by a constant $C_1$. It
follows that
\begin{equation}
|\varepsilon_{i,t}|
\leq
c_r+(1+\gamma)C_1
\leq
c_r+2C_1
=
c_\varepsilon,
\qquad
\forall i,t.
\end{equation}

Combined with \eqref{funcional_est}, the definition of $e_{i,t}$
implies that
\begin{equation}\label{E_r_bound}
|e_{i,t}|
\leq
(1+\gamma)d_0C_2L^{-m}
\leq
2d_0C_2L^{-m},
\qquad
\forall i,t.
\end{equation}
Our goal is to prove that, with high probability,
\begin{equation}
    \max_{k}
    \left\|
    \frac{1}{nT}
    \sum_{i,t}
    \boldsymbol{\xi}_{G_k,i,t}
    (\varepsilon_{i,t}-e_{i,t})
    \right\|_{2}
    \leq
    \lambda.
\end{equation}

For any $k=1,\dots,d$, by \eqref{E_r_bound},
\begin{align}
    \left\|
    \frac{1}{nT}
    \sum_{i,t}
    \boldsymbol{\xi}_{G_k,i,t}e_{i,t}
    \right\|_{2}
    &\leq
    \frac{1}{nT}
    \sum_{i,t}
    \left\|
    \boldsymbol{\xi}_{G_k,i,t}e_{i,t}
    \right\|_{2}
    \nonumber\\
    &\leq
    2d_0C_2L^{-m}
    \frac{1}{nT}
    \sum_{i,t}
    \left\|
    \boldsymbol{\xi}_{G_k,i,t}
    \right\|_{2}.
\end{align}
Moreover,
\begin{equation}
    \left\|
    \boldsymbol{\xi}_{G_k,i,t}
    \right\|_2
    \leq
    \max_k
    \sup_{x\in\mathcal X}
    \left\|
    \Phi_{G_k}(x)
    \right\|_2.
\end{equation}
Together with the bound in Remark~\ref{basis_bound}, we have
\begin{equation}
    \left\|
    \frac{1}{nT}
    \sum_{i,t}
    \boldsymbol{\xi}_{G_k,i,t}e_{i,t}
    \right\|_{2}
    \leq
    2d_0c_mC_2L^{1/2-m}.
\end{equation}

By the triangle inequality, we have
\begin{align}
&P\left(
\frac{1}{nT}
\left\|
\sum_{i,t}
\boldsymbol{\xi}_{G_k,i,t}
(\varepsilon_{i,t}-e_{i,t})
\right\|_{2}
\geq
\lambda
\right)
\nonumber\\
&\qquad\leq
P\left(
\frac{1}{nT}
\left\|
\sum_{i,t}
\boldsymbol{\xi}_{G_k,i,t}
\varepsilon_{i,t}
\right\|_{2}
\geq
\lambda
-
\frac{1}{nT}
\left\|
\sum_{i,t}
\boldsymbol{\xi}_{G_k,i,t}
e_{i,t}
\right\|_{2}
\right)
\nonumber\\
&\qquad\leq
P\left(
\frac{1}{nT}
\left\|
\sum_{i,t}
\boldsymbol{\xi}_{G_k,i,t}
\varepsilon_{i,t}
\right\|_{2}
\geq
\lambda
-
2d_0c_mC_2L^{1/2-m}
\right).
\end{align}
Denote
\begin{equation}
    \widetilde{\lambda}
    =
    \lambda
    -
    2d_0c_mC_2L^{1/2-m}.
\end{equation}

For any integer $1\leq g\leq nT$, let $i(g)$ and $t(g)$ be defined by
\begin{equation}
    i(g)
    =
    \left\lfloor\frac{g-1}{T}\right\rfloor+1,
    \qquad
    t(g)
    =
    g-\{i(g)-1\}T-1.
\end{equation}
Then
\begin{equation}
    g
    =
    \{i(g)-1\}T+t(g)+1,
\end{equation}
where $1\leq i(g)\leq n$ and $0\leq t(g)\leq T-1$.

We define the following filtration
$\{\boldsymbol{\sigma}(\mathcal{F}^{(g)})\}_{g=0}^{nT}$
iteratively. Begin with
\[
\mathcal{F}^{(0)}
=
\{X_{1,0},A_{1,0}\}.
\]
For $1\leq g\leq nT$, define
\[
\begin{aligned}
\mathcal{F}^{(g)}
={}&
\mathcal{F}^{(g-1)}
\cup
\left\{
R_{i(g),t(g)},
X_{i(g),t(g)+1},
A_{i(g),t(g)+1}
\right\},
&&\text{if }t(g)<T-1,
\\
\mathcal{F}^{(g)}
={}&
\mathcal{F}^{(g-1)}
\cup
\left\{
R_{i(g),T-1},
X_{i(g),T},
X_{i(g)+1,0},
A_{i(g)+1,0}
\right\},
&&\text{if }t(g)=T-1\text{ and }i(g)<n,
\\
\mathcal{F}^{(g)}
={}&
\mathcal{F}^{(g-1)}
\cup
\left\{
R_{n,T-1},
X_{n,T}
\right\},
&&\text{if }t(g)=T-1\text{ and }i(g)=n.
\end{aligned}
\]

By the Bellman equation, together with MA and CMIA,
\begin{align}
\mathbb{E}
\left\{
\varepsilon_{i(g),t(g)}
\mid
\boldsymbol{\sigma}(\mathcal{F}^{(g-1)})
\right\}
&=
\mathbb{E}
\left\{
\varepsilon_{i(g),t(g)}
\mid
X_{i(g),t(g)},A_{i(g),t(g)}
\right\}
\nonumber\\
&=0.
\end{align}
Notice that
$\boldsymbol{\xi}_{G_k,i,t}$ is a deterministic function of
$X_{i,t}$ and $A_{i,t}$. Therefore,
\begin{equation}
\mathbb{E}
\left\{
\boldsymbol{\xi}_{G_k,i(g),t(g)}
\varepsilon_{i(g),t(g)}
\mid
\boldsymbol{\sigma}(\mathcal{F}^{(g-1)})
\right\}
=
0.
\end{equation}
This indicates that
$\left\{
\boldsymbol{\xi}_{G_k,i(g),t(g)}
\varepsilon_{i(g),t(g)}
\right\}_{g=1}^{nT}$
is a vector-valued martingale difference sequence with respect to the
filtration
$\{\boldsymbol{\sigma}(\mathcal{F}^{(g)})\}_{g=0}^{nT}$.

The martingale difference sequence
$\boldsymbol{\xi}_{G_k}^{(g)}\varepsilon^{(g)}$
satisfies the following almost-sure bound:
\begin{equation}\label{Azuma_bd1}
    \left\|
    \boldsymbol{\xi}_{G_k}^{(g)}
    \varepsilon^{(g)}
    \right\|_2
    \leq
    c_{\varepsilon}
    \max_k
    \sup_{x\in\mathcal X}
    \|\Phi_k(x)\|_2
    \leq
    c_{\varepsilon}c_m\sqrt{L}.
\end{equation}
Moreover,
\begin{align}
&
\left\|
\sum_{g=1}^{nT}
\mathbb{E}
\left\{
\left(
\boldsymbol{\xi}_{G_k}^{(g)}
\varepsilon^{(g)}
\right)^T
\boldsymbol{\xi}_{G_k}^{(g)}
\varepsilon^{(g)}
\mid
\boldsymbol{\sigma}(\mathcal{F}^{(g-1)})
\right\}
\right\|_2
\nonumber\\
&\qquad\leq
c_{\varepsilon}^2
\sum_{g=1}^{nT}
\left\|
\left(\boldsymbol{\xi}_{G_k}^{(g)}\right)^T
\boldsymbol{\xi}_{G_k}^{(g)}
\right\|_2
\nonumber\\
&\qquad\leq
nTc_{\varepsilon}^2c_m^2L,
\end{align}
and
\begin{align}
&
\left\|
\sum_{g=1}^{nT}
\mathbb{E}
\left\{
\boldsymbol{\xi}_{G_k}^{(g)}
\varepsilon^{(g)}
\left(
\boldsymbol{\xi}_{G_k}^{(g)}
\varepsilon^{(g)}
\right)^T
\mid
\boldsymbol{\sigma}(\mathcal{F}^{(g-1)})
\right\}
\right\|_2
\nonumber\\
&\qquad\leq
c_{\varepsilon}^2
\sum_{g=1}^{nT}
\left\|
\boldsymbol{\xi}_{G_k}^{(g)}
\left(\boldsymbol{\xi}_{G_k}^{(g)}\right)^T
\right\|_2
\nonumber\\
&\qquad\leq
nTc_{\varepsilon}^2c_m^2L.
\end{align}

By the matrix Freedman inequality
\citep[Corollary~1.3]{Tropp2011}, we have
\begin{align}
&P\left(
\left\|
\sum_{i,t}
\boldsymbol{\xi}_{G_k,i,t}
\varepsilon_{i,t}
\right\|_2
\geq
nT\widetilde{\lambda}
\right)
\leq
(L+1)
\exp
\left\{
-\frac{nT\widetilde{\lambda}^2/2}
{
c_{\varepsilon}^2c_m^2L
+
c_{\varepsilon}c_m\sqrt{L}
\widetilde{\lambda}/3
}
\right\}.
\label{azuma_entry}
\end{align}
Let
\[
\widetilde{\lambda}
=
c_{\lambda}\sqrt{L}
\frac{\log(nT\vee d)}{\sqrt{nT}},
\]
where $c_{\lambda}>0$ is a sufficiently large constant. By
\eqref{azuma_entry},
\begin{align}
&P\left(
\left\|
\sum_{i,t}
\boldsymbol{\xi}_{G_k,i,t}
\varepsilon_{i,t}
\right\|_2
\geq
nT\widetilde{\lambda}
\right)
\leq
(L+1)
\exp\left[
-\frac{
c_{\lambda}^2L\{\log(nT\vee d)\}^2/2
}{
c_{\varepsilon}^2c_m^2L
+
c_{\varepsilon}c_mc_{\lambda}
L\log(nT\vee d)/(3\sqrt{nT})
}
\right].
\end{align}
Using
\[
\frac{x}{a+b}
\geq
\frac{1}{2}
\min\left\{
\frac{x}{a},
\frac{x}{b}
\right\},
\]
the exponent in the preceding display is bounded above by
\[
-\min\left\{
\frac{c_{\lambda}^2\{\log(nT\vee d)\}^2}
{4c_{\varepsilon}^2c_m^2},
\frac{3c_{\lambda}\sqrt{nT}\log(nT\vee d)}
{4c_{\varepsilon}c_m}
\right\}.
\]
Since $L\asymp(nT)^{1/(2m-1)}$, we have
\[
\log(L+1)
\leq
C_3\log(nT\vee d)
\]
for some constant $C_3>0$. Therefore, by choosing $c_{\lambda}$
sufficiently large, there exists a constant $C_4>2$ such that
\begin{equation}\label{azuma_entry2}
P\left(
\left\|
\sum_{i,t}
\boldsymbol{\xi}_{G_k,i,t}
\varepsilon_{i,t}
\right\|_2
\geq
nT\widetilde{\lambda}
\right)
\leq
\exp\left\{
-C_4\log(nT\vee d)
\right\}.
\end{equation}

With $L\asymp(nT)^{1/(2m-1)}$, we have
\[
L^{1/2-m}
\asymp
\frac{1}{\sqrt{nT}}.
\]
Therefore,
\[
2d_0c_mC_2L^{1/2-m}
\asymp
\frac{d_0}{\sqrt{nT}}.
\]
Under the condition $d_0=o(\sqrt{L})$,
\[
2d_0c_mC_2L^{1/2-m}
=
o\left\{
\sqrt{L}
\frac{\log(nT\vee d)}{\sqrt{nT}}
\right\}.
\]
Thus, for all sufficiently large $nT$,
\begin{equation}
2d_0c_mC_2L^{1/2-m}
\leq
\sqrt{L}
\frac{\log(nT\vee d)}{\sqrt{nT}}.
\end{equation}
Recalling that
\[
\widetilde{\lambda}
=
\lambda
-
2d_0c_mC_2L^{1/2-m},
\]
we may take
\[
\lambda
=
c_1\sqrt{L}
\frac{\log(nT\vee d)}{\sqrt{nT}},
\qquad
c_1=c_{\lambda}+1.
\]

Finally, taking the finite union of \eqref{azuma_entry2} over all
$k=1,\dots,2d$, we obtain
\begin{equation}
P\left\{
\max_{k=1,\dots,2d}
\left\|
\frac{1}{nT}
\sum_{i,t}
\boldsymbol{\xi}_{G_k,i,t}
(\varepsilon_{i,t}-e_{i,t})
\right\|_2
\leq
c_1\sqrt{L}
\frac{\log(nT\vee d)}{\sqrt{nT}}
\right\}
\geq
1-\exp\left\{
-c_2\log(nT\vee d)
\right\}.
\end{equation}

Note that $\|\boldsymbol{\xi}_{G_0,i,t}\|_2=1.$ Using the same argument
as above, we have
\begin{equation}
P\left\{
\left\|
\frac{1}{nT}
\sum_{i,t}
\boldsymbol{\xi}_{G_0,i,t}
(\varepsilon_{i,t}-e_{i,t})
\right\|_2
\leq
c_1
\frac{\log(nT\vee d)}{\sqrt{nT}}
\right\}
\geq
1-\exp\left\{
-c_2\log(nT\vee d)
\right\}.
\end{equation}
Therefore, there exists a good event that,
\[
P(\mathcal{G}_1)
\geq
1-\exp\{-c_2\log(nT\vee d)\}.
\]
On the event $\mathcal{G}_1$,
$(\boldsymbol{\alpha}^{*},\boldsymbol{\beta}^{*})_{\oplus}$
satisfies all the constraints in \eqref{dant_est} and is therefore
feasible. This completes the proof.
\end{proof}

\begin{lemma}
\label{Asy_Sigma_infty}
There exist positive constants $c_3$ and $c_4$ and a good event
$\mathcal{G}_2$ satisfying
\[
P(\mathcal{G}_2)
\geq
1-\exp\left\{-c_4\log(nT\vee d)\right\}
\]
such that, on $\mathcal{G}_2$,
\[
\left\|
\widehat{\boldsymbol{\Sigma}}
-
\boldsymbol{\Sigma}
\right\|_{\max}
\leq
c_3L\log(nT\vee d)
\sqrt{
\frac{\log(nT\vee d)}{nT}
}.
\]
\end{lemma}

\begin{proof}
Throughout this proof, $C_1,C_2,\ldots$ denote positive constants
that are local to the present proof and are re-indexed from $C_1$.
They are distinct from the lowercase constants
$c_1,c_2,\ldots$ appearing in the statements of the results.

The main argument of this proof follows from the proof of Lemma~3 in
\cite{shi2022statistical}. The major difference is the matrix norm considered
here.

Let $\Theta(l,k)$ be the $(l,k)$th entry of
\[
\widehat{\boldsymbol{\Sigma}}
=
\frac{1}{nT}
\sum_{i=1}^n\sum_{t=0}^{T-1}
\boldsymbol{\xi}_{i,t}
\left(
\boldsymbol{\xi}_{i,t}
-
\gamma\boldsymbol{U}_{i,t+1}
\right)^\top,
\]
and let $\Theta_{i,t}(l,k)$ be the corresponding entry of
$\boldsymbol{\xi}_{i,t}
(\boldsymbol{\xi}_{i,t}-\gamma\boldsymbol{U}_{i,t+1})^\top$.
The following argument holds for every pair $(l,k)$ and does not depend on
the specific values of $(l,k)$. Thus, we suppress $(l,k)$ and write
$\Theta$ and $\Theta_{i,t}$ for simplicity.

For $1\leq l,k\leq 2(d+1)L$, define
$\Theta_i=T^{-1}\sum_{t=0}^{T-1}\Theta_{i,t}$. We have
\[
|\Theta_{i,t}|
\leq
\|\boldsymbol{\xi}_{i,t}\|_\infty
\left(
\|\boldsymbol{\xi}_{i,t}\|_\infty
+
\|\boldsymbol{U}_{i,t+1}\|_\infty
\right)
\leq
2\|\boldsymbol{\Phi}(X_{i,t})\|_\infty^2
\leq 2L.
\]
Consequently,
\begin{equation}
\label{theta_uni_bound}
\left|
\Theta_{i,t}-\mathbb{E}\Theta_{i,t}
\right|
\leq 4L.
\end{equation}
We divide the proof into two scenarios: (i) $T$ is bounded and
(ii) $T$ diverges.

We first consider the case where $T$ is bounded. By
\eqref{theta_uni_bound},
\[
\left|
\Theta_i-\mathbb{E}\Theta_i
\right|
\leq
\frac{1}{T}
\sum_{t=0}^{T-1}
\left|
\Theta_{i,t}-\mathbb{E}\Theta_{i,t}
\right|
\leq 4L.
\]
Since $\Theta_1,\ldots,\Theta_n$ are independent, Hoeffding's inequality
gives
\[
\Pr\left(
\left|
\frac{1}{n}\sum_{i=1}^n
\{\Theta_i-\mathbb{E}\Theta_i\}
\right|>\tau
\right)
\leq
2\exp\left\{
-\frac{n\tau^2}{32L^2}
\right\}.
\]
Taking $\tau=C_1L\sqrt{\log(nT\vee d)/n}$, we obtain
\[
\left|
\Theta-\mathbb{E}\Theta
\right|
\leq
C_1L\sqrt{\frac{\log(nT\vee d)}{n}}
\]
with probability at least
$1-\exp\{-C_2\log(nT\vee d)\}$. Since $T$ is bounded, the factor
$T^{1/2}$ can be absorbed into the constant $C_1$. Therefore,
\begin{equation}
\label{theta_bounded_T}
\left|
\Theta-\mathbb{E}\Theta
\right|
\leq
C_1L
\sqrt{
\frac{\log(nT\vee d)}{nT}
}
\end{equation}
with probability at least
$1-\exp\{-C_2\log(nT\vee d)\}$.

We next consider the case where $T$ diverges. We first derive a tighter
high-probability bound for
$\Theta_i-\mathbb{E}\Theta_i$ than the deterministic bound of order $L$.
Define
\[
\Theta_{0,t}^*
=
\mathbb{E}
\left[
\Theta_{0,t}\mid X_{0,t},A_{0,t}
\right],
\qquad
\Theta_{0,t}^{**}
=
\mathbb{E}
\left[
\Theta_{0,t}^*\mid X_{0,t}
\right].
\]
We use the decomposition
\begin{equation}
\label{theta_decomposition}
\Theta_{0,t}-\mathbb{E}\Theta_{0,t}
=
\left(
\Theta_{0,t}-\Theta_{0,t}^*
\right)
+
\left(
\Theta_{0,t}^*-\Theta_{0,t}^{**}
\right)
+
\left(
\Theta_{0,t}^{**}-\mathbb{E}\Theta_{0,t}
\right).
\end{equation}

Consider the first term in \eqref{theta_decomposition}. Let
\[
\mathcal{F}_t
=
\boldsymbol{\sigma}
\left\{
X_{0,0},A_{0,0},X_{0,1},\ldots,
X_{0,t},A_{0,t},X_{0,t+1}
\right\}.
\]
Then $\Theta_{0,t}-\Theta_{0,t}^*$ is $\mathcal{F}_t$-measurable and
\[
\mathbb{E}
\left[
\Theta_{0,t}-\Theta_{0,t}^*
\mid
\mathcal{F}_{t-1}
\right]
=
\mathbb{E}
\left[
\mathbb{E}
\left[
\Theta_{0,t}-\Theta_{0,t}^*
\mid
X_{0,t},A_{0,t}
\right]
\middle|
\mathcal{F}_{t-1}
\right]
=0.
\]
Thus, $\{\Theta_{0,t}-\Theta_{0,t}^*\}_{t\geq0}$ is a martingale
difference sequence with respect to $\{\mathcal{F}_t\}_{t\geq0}$.
Moreover, by \eqref{theta_uni_bound},
$|\Theta_{0,t}-\Theta_{0,t}^*|\leq 4L$. Azuma's inequality therefore
gives
\begin{equation}
\label{theta_martingale_bound}
\Pr\left(
\left|
\sum_{t=0}^{T-1}
\left(
\Theta_{0,t}-\Theta_{0,t}^*
\right)
\right|>\tau
\right)
\leq
2\exp\left\{
-\frac{\tau^2}{32L^2T}
\right\}.
\end{equation}

Consider the second term in \eqref{theta_decomposition}. Since the
behavior policy is a fixed stationary Markov policy and \(X_{0,t}\) is
\(\mathcal F_{t-1}\)-measurable,
\[
\begin{aligned}
\mathbb{E}
\left[
\Theta_{0,t}^*-\Theta_{0,t}^{**}
\mid \mathcal F_{t-1}
\right]
&=
\mathbb{E}
\left[
\Theta_{0,t}^*
\mid \mathcal F_{t-1}
\right]
-
\Theta_{0,t}^{**} \\
&=
\mathbb{E}
\left[
\Theta_{0,t}^*
\mid X_{0,t}
\right]
-
\Theta_{0,t}^{**}
=0.
\end{aligned}
\]
Therefore,
\(\{\Theta_{0,t}^*-\Theta_{0,t}^{**}\}_{t\geq0}\) is also a
martingale difference sequence with respect to
\(\{\mathcal F_t\}_{t\geq0}\). Moreover,
\(
|\Theta_{0,t}^*-\Theta_{0,t}^{**}|\leq 4L
\).
By Azuma's inequality,
\begin{equation}
\label{theta_action_bound}
\Pr\left(
\left|
\sum_{t=0}^{T-1}
\left(
\Theta_{0,t}^*-\Theta_{0,t}^{**}
\right)
\right|>\tau
\right)
\leq
2\exp\left\{
-\frac{\tau^2}{32L^2T}
\right\}.
\end{equation}
We now consider the last term in \eqref{theta_decomposition}. Let
$\widetilde{\Theta}_{0,t}
=
\Theta_{0,t}^{**}-\mathbb{E}\Theta_{0,t}$.
Then $\widetilde{\Theta}_{0,t}$ is a deterministic function of
$X_{0,t}$ and $|\widetilde{\Theta}_{0,t}|\leq 4L$. By Assumption~(A4),
the Markov chain $\{X_{0,t}\}_{t\geq0}$ is exponentially
$\beta$-mixing, so that $\beta(q)\leq C_\beta\rho^q$ for some
$\rho\in(0,1)$.

Following the proof of Theorem~4.2 in \cite{CHEN2015447}, we extend the
probability space and construct coupled blocks using Berbee's lemma.
Let $m=\lfloor T/q\rfloor$ and define
\[
Y_k
=
\left\{
X_{0,(k-1)q},\ldots,X_{0,kq-1}
\right\},
\qquad
k=1,\ldots,m.
\]
There exist coupled blocks $Y_k^*$ such that $Y_k^*$ and $Y_k$ are
identically distributed,
$\Pr(Y_k\neq Y_k^*)\leq\beta(q)$, the odd-numbered blocks
$Y_1^*,Y_3^*,\ldots$ are independent, and the even-numbered blocks
$Y_2^*,Y_4^*,\ldots$ are independent.

Let $X_{0,t}^*$ denote the coupled process within these blocks and define
$\widetilde{\Theta}_{0,t}^*
=
\widetilde{\Theta}_{0,t}(X_{0,t}^*)$.
Let $I_o$ and $I_e$ be the sets of time indices belonging to the odd-
and even-numbered complete blocks, respectively, and let
$I_r=\{mq,\ldots,T-1\}$ be the set of remaining indices. Then
$\{0,\ldots,T-1\}=I_o\cup I_e\cup I_r$ and $|I_r|<q$.

By the triangle inequality,
\begin{align}
&\Pr\left(
\left|
\sum_{t=0}^{T-1}
\widetilde{\Theta}_{0,t}
\right|>4\tau
\right)
\nonumber\\
&\quad\leq
\Pr\left(
\left|
\sum_{t\in I_r}
\widetilde{\Theta}_{0,t}
\right|>\tau
\right)
+
\Pr\left(
\left|
\sum_{t\in I_o}
\widetilde{\Theta}_{0,t}^*
\right|>\tau
\right)
+
\Pr\left(
\left|
\sum_{t\in I_e}
\widetilde{\Theta}_{0,t}^*
\right|>\tau
\right)
\nonumber\\
&\qquad+
\Pr\left(
\sum_{k=1}^m
\boldsymbol{1}\{Y_k\neq Y_k^*\}>0
\right).
\label{theta_block_decomposition}
\end{align}

When $\tau>4qL$, the first probability on the right-hand side of
\eqref{theta_block_decomposition} is zero. For the odd-numbered blocks,
the block sums
\[
\sum_{j=0}^{q-1}
\widetilde{\Theta}_{0,(k-1)q+j}^*,
\qquad k=1,3,\ldots,
\]
are independent, mean-zero random variables bounded in absolute value
by $4qL$. Hence, Hoeffding's inequality gives
\[
\Pr\left(
\left|
\sum_{t\in I_o}
\widetilde{\Theta}_{0,t}^*
\right|>\tau
\right)
\leq
2\exp\left\{
-\frac{C_3\tau^2}{qL^2T}
\right\}.
\]
The same argument applies to the even-numbered blocks. Moreover, by the
union bound and Berbee's coupling,
\[
\Pr\left(
\sum_{k=1}^m
\boldsymbol{1}\{Y_k\neq Y_k^*\}>0
\right)
\leq
m\beta(q)
\leq
\frac{T}{q}C_\beta\rho^q.
\]
Combining these results, we obtain
\begin{equation}
\label{theta_mixing_bound}
\Pr\left(
\left|
\sum_{t=0}^{T-1}
\widetilde{\Theta}_{0,t}
\right|>4\tau
\right)
\leq
4\exp\left\{
-\frac{C_4\tau^2}{qL^2T}
\right\}
+
\frac{T}{q}C_\beta\rho^q.
\end{equation}

Combining \eqref{theta_martingale_bound},
\eqref{theta_action_bound}, and \eqref{theta_mixing_bound}, we have
\begin{align}
&\Pr\left(
\left|
\sum_{t=0}^{T-1}
\left(
\Theta_{0,t}-\mathbb{E}\Theta_{0,t}
\right)
\right|>6\tau
\right)
\nonumber\\
&\quad\leq
4\exp\left\{
-\frac{C_5\tau^2}{L^2T}
\right\}
+
4\exp\left\{
-\frac{C_6\tau^2}{qL^2T}
\right\}
+
\frac{T}{q}C_\beta\rho^q
\nonumber\\
&\quad\leq
8\exp\left\{
-\frac{C_7\tau^2}{qL^2T}
\right\}
+
\frac{T}{q}C_\beta\rho^q.
\label{theta_trajectory_tail}
\end{align}

Set
\[
q
=
\left\lceil
-\frac{5\log(nT\vee d)}{\log\rho}
\right\rceil
\]
and
$\tau=\max\{C_1L\log(nT\vee d)\sqrt{T},\,5qL\}$.
When $T$ is sufficiently large,
$\tau=C_1L\log(nT\vee d)\sqrt{T}$.
Substituting these choices into \eqref{theta_trajectory_tail} and
absorbing fixed constants into $C_1$ gives
\begin{equation}
\label{theta_single_trajectory}
\Pr\left(
\left|
\sum_{t=0}^{T-1}
\left(
\Theta_{0,t}-\mathbb{E}\Theta_{0,t}
\right)
\right|
>
C_1L\log(nT\vee d)\sqrt{T}
\right)
\leq
\exp\left\{
-C_8\log(nT\vee d)
\right\}.
\end{equation}

Since
$\{R_{i,t},X_{i,t},A_{i,t}\}_{t\geq0}$,
$i=1,\ldots,n$, are independent copies of
$\{R_{0,t},X_{0,t},A_{0,t}\}_{t\geq0}$, the bound
\eqref{theta_single_trajectory} holds for every $i$. Define
\[
Z_i
=
\sum_{t=0}^{T-1}
\left(
\Theta_{i,t}-\mathbb{E}\Theta_{i,t}
\right),
\qquad
B
=
C_1L\log(nT\vee d)\sqrt{T}.
\]
Then $Z_1,\ldots,Z_n$ are independent and mean zero, and
$\Pr(|Z_i|>B)\leq\exp\{-C_8\log(nT\vee d)\}$. Define
\[
\overline{Z}_i
=
Z_i\boldsymbol{1}\{|Z_i|\leq B\}
-
\mathbb{E}
\left[
Z_i\boldsymbol{1}\{|Z_i|\leq B\}
\right].
\]
The variables $\overline{Z}_1,\ldots,\overline{Z}_n$ are independent,
mean zero, and satisfy $|\overline{Z}_i|\leq 2B$. Therefore,
Hoeffding's inequality gives
\begin{equation}
\label{theta_truncated_hoeffding}
\Pr\left(
\left|
\sum_{i=1}^n\overline{Z}_i
\right|>x
\right)
\leq
2\exp\left\{
-\frac{x^2}{8nB^2}
\right\}.
\end{equation}

Since $\mathbb{E}Z_i=0$ and $|Z_i|\leq 4LT$,
\[
\left|
\mathbb{E}
\left[
Z_i\boldsymbol{1}\{|Z_i|\leq B\}
\right]
\right|
=
\left|
\mathbb{E}
\left[
Z_i\boldsymbol{1}\{|Z_i|>B\}
\right]
\right|
\leq
4LT\exp\left\{
-C_8\log(nT\vee d)
\right\}.
\]
Consequently, after adjusting the constants, the total truncation bias is
smaller than

\noindent $C_1L\{\log(nT\vee d)\}^{3/2}\sqrt{nT}$.

Taking
$x=c_3L\{\log(nT\vee d)\}^{3/2}\sqrt{nT}$ in
\eqref{theta_truncated_hoeffding}, and using
\[
\Pr\left(
\max_{1\leq i\leq n}|Z_i|>B
\right)
\leq
n\exp\left\{
-C_8\log(nT\vee d)
\right\},
\]
we obtain
\begin{align}
&\Pr\left(
\left|
\sum_{i=1}^n\sum_{t=0}^{T-1}
\left(
\Theta_{i,t}-\mathbb{E}\Theta_{i,t}
\right)
\right|
>
c_3
L\{\log(nT\vee d)\}^{3/2}\sqrt{nT}
\right)
\nonumber\\
&\quad\leq
\exp\left\{
-(c_3/C_1)^2\log(nT\vee d)/8
\right\}.
\label{theta_all_trajectories}
\end{align}
Dividing both sides of the event in
\eqref{theta_all_trajectories} by $nT$ gives
\begin{equation}
\label{theta_entry_bound}
\Pr\left(
\left|
\Theta-\mathbb{E}\Theta
\right|
<
c_3L\log(nT\vee d)
\sqrt{
\frac{\log(nT\vee d)}{nT}
}
\right)
\geq
1-
\exp\left\{
-C_9\log(nT\vee d)
\right\}.
\end{equation}

Finally, define
\[
\mathcal{G}_2
=
\left\{
\left\|
\widehat{\boldsymbol{\Sigma}}
-
\boldsymbol{\Sigma}
\right\|_{\max}
<
c_3L\log(nT\vee d)
\sqrt{
\frac{\log(nT\vee d)}{nT}
}
\right\}.
\]
Taking a finite union over all entries of the matrix and combining
\eqref{theta_bounded_T} and \eqref{theta_entry_bound}, we obtain
\[
P(\mathcal{G}_2)
\geq
1-
\{2(d+1)L\}^2
\exp\left\{
-\min(C_2,C_9)
\log(nT\vee d)
\right\}
\geq
1-
\exp\left\{
-c_4\log(nT\vee d)
\right\}.
\]
This completes the proof.
\end{proof}

\begin{lemma}[Matrix concentration inequalities]
\label{lem:matrix_concentration}
Let \(p_1,p_2\geq1\) and write \(D=p_1+p_2\).

\begin{itemize}
\item
Let \(\{\boldsymbol{Z}_j\}_{j=1}^m\) be independent, mean-zero
random matrices of dimension \(p_1\times p_2\). Suppose that
\(\|\boldsymbol{Z}_j\|_2\leq R\) almost surely, and define
\[
\nu^2
=
\max\left\{
\left\|
\sum_{j=1}^m
\mathbb{E}
\left(
\boldsymbol{Z}_j\boldsymbol{Z}_j^\top
\right)
\right\|_2,
\left\|
\sum_{j=1}^m
\mathbb{E}
\left(
\boldsymbol{Z}_j^\top\boldsymbol{Z}_j
\right)
\right\|_2
\right\}.
\]
Then, for any \(u>0\),
\[
\Pr\left[
\left\|
\sum_{j=1}^m\boldsymbol{Z}_j
\right\|_2
>
\sqrt{2\nu^2\{u+\log D\}}
+
\frac{2R}{3}\{u+\log D\}
\right]
\leq
\exp(-u).
\]
In particular, if \(\nu^2\leq mR^2\), then
\[
\Pr\left[
\left\|
\sum_{j=1}^m\boldsymbol{Z}_j
\right\|_2
>
R\sqrt{2m\{u+\log D\}}
+
\frac{2R}{3}\{u+\log D\}
\right]
\leq
\exp(-u).
\]

\item Let \(\{\boldsymbol{X}_j,\mathcal{F}_j\}_{j=1}^m\) be an adapted
sequence of random matrices of dimension \(p_1\times p_2\) satisfying
\(\mathbb{E}(\boldsymbol{X}_j\mid\mathcal{F}_{j-1})
=\boldsymbol{0}\). Suppose that
\(\|\boldsymbol{X}_j\|_2\leq R_j\) almost surely, where
\(R_1,\ldots,R_m\) are deterministic. Then, for any \(u>0\),
\[
\Pr\left[
\left\|
\sum_{j=1}^m\boldsymbol{X}_j
\right\|_2
>
\sqrt{
8\{u+\log D\}
\sum_{j=1}^m R_j^2
}
\right]
\leq
\exp(-u).
\]
In particular, if \(R_j\leq R\) for all \(j\), then
\[
\Pr\left[
\left\|
\sum_{j=1}^m\boldsymbol{X}_j
\right\|_2
>
R\sqrt{8m\{u+\log D\}}
\right]
\leq
\exp(-u).
\]
\end{itemize}
\end{lemma}

\begin{proof}
For a \(p_1\times p_2\) matrix \(\boldsymbol{M}\), define its
self-adjoint dilation as
\[
\mathcal{D}(\boldsymbol{M})
=
\begin{pmatrix}
\boldsymbol{0} & \boldsymbol{M}\\
\boldsymbol{M}^\top & \boldsymbol{0}
\end{pmatrix}.
\]
The dilation satisfies
\[
\lambda_{\max}\{\mathcal{D}(\boldsymbol{M})\}
=
\|\boldsymbol{M}\|_2
\]
and
\[
\mathcal{D}(\boldsymbol{M})^2
=
\begin{pmatrix}
\boldsymbol{M}\boldsymbol{M}^\top & \boldsymbol{0}\\
\boldsymbol{0} & \boldsymbol{M}^\top\boldsymbol{M}
\end{pmatrix}.
\]

For part (i), the rectangular matrix Bernstein inequality
\citep[Theorem~1.6]{tropp2012user} gives, for every \(t>0\),
\[
\Pr\left(
\left\|
\sum_{j=1}^m\boldsymbol{Z}_j
\right\|_2
>t
\right)
\leq
D\exp\left\{
-\frac{t^2/2}{\nu^2+Rt/3}
\right\}.
\]
Let \(x=u+\log D\) and set
\[
t
=
\sqrt{2\nu^2x}
+
\frac{2R}{3}x.
\]
A direct calculation gives
\[
t^2
-
2x\left(
\nu^2+\frac{Rt}{3}
\right)
=
\frac{2R}{3}x\sqrt{2\nu^2x}
\geq0.
\]
It follows that
\[
\frac{t^2/2}{\nu^2+Rt/3}
\geq x.
\]
Therefore,
\[
\Pr\left(
\left\|
\sum_{j=1}^m\boldsymbol{Z}_j
\right\|_2
>t
\right)
\leq
D\exp(-x)
=
\exp(-u),
\]
which proves the first assertion. The second follows immediately from
\(\nu^2\leq mR^2\).

For part (ii), consider the self-adjoint matrix differences
\(\mathcal{D}(\boldsymbol{X}_j)\). They satisfy
\[
\mathbb{E}
\left[
\mathcal{D}(\boldsymbol{X}_j)
\mid\mathcal{F}_{j-1}
\right]
=
\boldsymbol{0}
\]
and
\[
\mathcal{D}(\boldsymbol{X}_j)^2
\preceq
R_j^2\boldsymbol{I}_D.
\]
The matrix Azuma inequality and its rectangular extension
\citep[Theorem~7.1 and Remark~7.3]{tropp2012user} therefore imply that
\[
\Pr\left(
\left\|
\sum_{j=1}^m\boldsymbol{X}_j
\right\|_2
>t
\right)
\leq
D\exp\left\{
-\frac{t^2}{
8\sum_{j=1}^mR_j^2
}
\right\}.
\]
Taking
\[
t
=
\sqrt{
8\{u+\log D\}
\sum_{j=1}^mR_j^2
}
\]
yields the desired probability bound. The final assertion follows by
using \(\sum_{j=1}^mR_j^2\leq mR^2\).
\end{proof}

\begin{lemma}
\label{Asy_Sigma_spectral}
There exist constants \(c_5>0\) and
\(c_6>0\) and a good event \(\mathcal G_3\) satisfying
\[
\Pr(\mathcal G_3)
\geq
1-\exp\{-c_6\log(nT\vee d)\}
\]
such that, on \(\mathcal G_3\),
\[
\left\|
\widehat{\boldsymbol{\Sigma}}
-
\boldsymbol{\Sigma}
\right\|_2
\leq
c_5dL
\left\{
\frac{\log(nT\vee d)}{\sqrt{nT}}
+
\frac{\{\log(nT\vee d)\}^2}{nT}
\right\}.
\]
\end{lemma}

\begin{proof}
Throughout this proof, \(C_1,C_2,\ldots\) denote positive constants
that are local to the present proof and are re-indexed from \(C_1\).
They are distinct from the lowercase constants
\(c_1,c_2,\ldots\) appearing in the statements of the results.

All matrix concentration inequalities used in this proof are equivalent
reformulations of those in \citet{tropp2012user}; see
Lemma~\ref{lem:matrix_concentration} for their derivations.

Let \(p'=2(d+1)L\) and
\[
\boldsymbol{\Theta}_{i,t}
=
\boldsymbol{\xi}_{i,t}
\left(
\boldsymbol{\xi}_{i,t}
-
\gamma\boldsymbol{U}_{i,t+1}
\right)^\top.
\]
Then
\[
\widehat{\boldsymbol{\Sigma}}
-
\boldsymbol{\Sigma}
=
\frac{1}{nT}
\sum_{i=1}^n\sum_{t=0}^{T-1}
\left(
\boldsymbol{\Theta}_{i,t}
-
\mathbb E\boldsymbol{\Theta}_{i,t}
\right).
\]

By Remark~\ref{basis_bound},
\(\|\boldsymbol{\xi}_{i,t}\|_2\leq
\sqrt{1+dc_m^2L}\) and
\(\|\boldsymbol{U}_{i,t+1}\|_2\leq
\sqrt{1+dc_m^2L}\). Hence,
\[
\|\boldsymbol{\Theta}_{i,t}\|_2
\leq
(1+\gamma)(1+dc_m^2L).
\]
Write \(R=(1+\gamma)(1+dc_m^2L)\). Notice that
\(R\leq(1+\gamma)(1+c_m^2)dL\).

Note that \(L\asymp (nT)^{{1}/{(2m-1)}}\). There exists a constant
\(C_1>0\) such that
\[
\log(2p')
\leq
C_1\log(nT\vee d).
\]

We first consider the case where \(T\) is bounded. Let \(C_2>0\)
satisfy \(T\leq C_2\), and define
\[
\boldsymbol Z_i
=
\frac{1}{T}
\sum_{t=0}^{T-1}
\left(
\boldsymbol{\Theta}_{i,t}
-
\mathbb E\boldsymbol{\Theta}_{i,t}
\right).
\]
The matrices \(\boldsymbol Z_1,\ldots,\boldsymbol Z_n\) are independent
and mean zero, with \(\|\boldsymbol Z_i\|_2\leq2R\). Their matrix
variance parameter is bounded by \(4nR^2\). Therefore,
Lemma~\ref{lem:matrix_concentration}(i) gives, for every \(u>0\),
\[
\Pr\left[
\left\|
\frac{1}{n}\sum_{i=1}^n\boldsymbol Z_i
\right\|_2
>
2\sqrt{2}R
\sqrt{\frac{u+\log(2p')}{n}}
+
\frac{4R}{3n}\{u+\log(2p')\}
\right]
\leq
\exp(-u).
\]
Taking \(u=4\log(nT\vee d)\), the quantity \(u+\log(2p')\) is bounded
by \((4+C_1)\log(nT\vee d)\). Since \(T\leq C_2\) and
\(\log(nT\vee d)\geq1\) for sufficiently large \(nT\vee d\), it follows
that
\[
\left\|
\widehat{\boldsymbol{\Sigma}}
-
\boldsymbol{\Sigma}
\right\|_2
\leq
C_3R
\frac{\log(nT\vee d)}{\sqrt{nT}}
\]
with probability at least \(1-(nT\vee d)^{-4}\), where
\[
C_3
=
2\sqrt{2C_2(4+C_1)}
+
\frac{4C_2(4+C_1)}{3}.
\]

We next consider the case where \(T\) diverges. Define
\(\boldsymbol{\Theta}_{i,t}^*
=\mathbb E(\boldsymbol{\Theta}_{i,t}\mid X_{i,t},A_{i,t})\) and
\(\boldsymbol{\Theta}_{i,t}^{**}
=\mathbb E(\boldsymbol{\Theta}_{i,t}^*\mid X_{i,t})\). Then
\[
\begin{aligned}
\boldsymbol{\Theta}_{i,t}
-
\mathbb E\boldsymbol{\Theta}_{i,t}
&=
\left(
\boldsymbol{\Theta}_{i,t}
-
\boldsymbol{\Theta}_{i,t}^*
\right)
+
\left(
\boldsymbol{\Theta}_{i,t}^*
-
\boldsymbol{\Theta}_{i,t}^{**}
\right) \\
&\quad+
\left(
\boldsymbol{\Theta}_{i,t}^{**}
-
\mathbb E\boldsymbol{\Theta}_{i,t}
\right).
\end{aligned}
\]
Accordingly, write
\[
\sum_{i=1}^n\sum_{t=0}^{T-1}
\left(
\boldsymbol{\Theta}_{i,t}
-
\mathbb E\boldsymbol{\Theta}_{i,t}
\right)
=
\boldsymbol S_1+\boldsymbol S_2+\boldsymbol S_3,
\]
where
\[
\boldsymbol S_1
=
\sum_{i,t}
\left(
\boldsymbol{\Theta}_{i,t}
-
\boldsymbol{\Theta}_{i,t}^*
\right),
\qquad
\boldsymbol S_2
=
\sum_{i,t}
\left(
\boldsymbol{\Theta}_{i,t}^*
-
\boldsymbol{\Theta}_{i,t}^{**}
\right),
\]
and
\[
\boldsymbol S_3
=
\sum_{i,t}
\left(
\boldsymbol{\Theta}_{i,t}^{**}
-
\mathbb E\boldsymbol{\Theta}_{i,t}
\right).
\]

Order the trajectories by \(i=1,\ldots,n\). For each \(i\), let
\(\mathcal X_{<i}\) be the sigma-field generated by the complete
trajectories \(1,\ldots,i-1\). Define the within-trajectory filtration
by
\[
\begin{aligned}
\mathcal H_{i,2t}
&=
\sigma\left\{
\mathcal X_{<i},
X_{i,0},A_{i,0},X_{i,1},\ldots,
A_{i,t-1},X_{i,t}
\right\},\\
\mathcal H_{i,2t+1}
&=
\mathcal H_{i,2t}\vee\sigma(A_{i,t}),\\
\mathcal H_{i,2t+2}
&=
\mathcal H_{i,2t+1}\vee\sigma(X_{i,t+1}),
\qquad
t=0,\ldots,T-1.
\end{aligned}
\]
These sigma-fields, ordered first by \(i\) and then by the second
index, form a filtration.

For \(\boldsymbol S_1\), place the increment
\(\boldsymbol{\Theta}_{i,t}-\boldsymbol{\Theta}_{i,t}^*\) between
\(\mathcal H_{i,2t+1}\) and \(\mathcal H_{i,2t+2}\), and place a zero
increment between \(\mathcal H_{i,2t}\) and
\(\mathcal H_{i,2t+1}\). The increment is
\(\mathcal H_{i,2t+2}\)-measurable. Moreover, by the Markov property
and the independence of the trajectories,
\[
\begin{aligned}
\mathbb E\left[
\boldsymbol{\Theta}_{i,t}
-
\boldsymbol{\Theta}_{i,t}^*
\mid
\mathcal H_{i,2t+1}
\right]
&=
\mathbb E\left[
\boldsymbol{\Theta}_{i,t}
\mid
\mathcal H_{i,2t+1}
\right]
-
\boldsymbol{\Theta}_{i,t}^*\\
&=
\mathbb E\left[
\boldsymbol{\Theta}_{i,t}
\mid
X_{i,t},A_{i,t}
\right]
-
\boldsymbol{\Theta}_{i,t}^*\\
&=
\boldsymbol 0.
\end{aligned}
\]
Thus, the summands in \(\boldsymbol S_1\), together with the inserted
zero increments, form a matrix martingale difference sequence.

For \(\boldsymbol S_2\), place the increment
\(\boldsymbol{\Theta}_{i,t}^*-\boldsymbol{\Theta}_{i,t}^{**}\)
between \(\mathcal H_{i,2t}\) and \(\mathcal H_{i,2t+1}\), and place a
zero increment between \(\mathcal H_{i,2t+1}\) and
\(\mathcal H_{i,2t+2}\). Since the behavior policy is a fixed
stationary Markov policy,
\[
\begin{aligned}
\mathbb E\left[
\boldsymbol{\Theta}_{i,t}^*
-
\boldsymbol{\Theta}_{i,t}^{**}
\mid
\mathcal H_{i,2t}
\right]
&=
\mathbb E\left[
\boldsymbol{\Theta}_{i,t}^*
\mid
\mathcal H_{i,2t}
\right]
-
\boldsymbol{\Theta}_{i,t}^{**}\\
&=
\mathbb E\left[
\boldsymbol{\Theta}_{i,t}^*
\mid
X_{i,t}
\right]
-
\boldsymbol{\Theta}_{i,t}^{**}\\
&=
\boldsymbol 0.
\end{aligned}
\]
Therefore, the summands in \(\boldsymbol S_2\), together with the
inserted zero increments, also form a matrix martingale difference
sequence.
Because conditional expectations preserve the almost-sure bound
\(\|\boldsymbol{\Theta}_{i,t}\|_2\leq R\), every increment in
\(\boldsymbol S_1\) and \(\boldsymbol S_2\) has operator norm at most
\(2R\). Lemma~\ref{lem:matrix_concentration}(ii) therefore yields, for
\(j=1,2\) and every \(u>0\),
\[
\Pr\left[
\|\boldsymbol S_j\|_2
>
4\sqrt{2}R
\sqrt{nT\{u+\log(2p')\}}
\right]
\leq
\exp(-u).
\]
Taking \(u=4\log(nT\vee d)\) and combining the two bounds gives
\begin{equation}
\label{Sigma_matrix_martingale}
\frac{\|\boldsymbol S_1\|_2+\|\boldsymbol S_2\|_2}{nT}
\leq
8\sqrt{2(4+C_1)}\,R
\sqrt{
\frac{\log(nT\vee d)}{nT}
}
\end{equation}
with probability at least \(1-2(nT\vee d)^{-4}\).

It remains to control \(\boldsymbol S_3\). Let
\(\widetilde{\boldsymbol{\Theta}}_{i,t}
=\boldsymbol{\Theta}_{i,t}^{**}
-\mathbb E\boldsymbol{\Theta}_{i,t}\). This is a deterministic
matrix-valued function of \(X_{i,t}\), satisfies
\(\mathbb E\widetilde{\boldsymbol{\Theta}}_{i,t}=\boldsymbol 0\), and
has operator norm at most \(2R\).

By Assumption~(A4), the state process is exponentially
\(\beta\)-mixing, with \(\beta(q)\leq C_\beta\rho^q\) for some
\(C_\beta>0\) and \(\rho\in(0,1)\). Set
\[
q
=
\left\lceil
-\frac{5\log(nT\vee d)}{\log\rho}
\right\rceil
\]
and \(m=\lceil T/q\rceil\). By the growth conditions, \(q\leq T\) for
all sufficiently large \(nT\vee d\). Partition
\(\{0,\ldots,T-1\}\) into consecutive blocks
\(\mathcal I_1,\ldots,\mathcal I_m\), each containing \(q\) time points
except possibly the last block.

For each \(i\) and \(k\), let
\(Y_{i,k}=(X_{i,t}:t\in\mathcal I_k)\) and define
\[
\boldsymbol W_{i,k}
=
\sum_{t\in\mathcal I_k}
\widetilde{\boldsymbol{\Theta}}_{i,t}.
\]
Since \(\widetilde{\boldsymbol{\Theta}}_{i,t}\) is a deterministic
function of \(X_{i,t}\), the matrix \(\boldsymbol W_{i,k}\) is a
measurable function of \(Y_{i,k}\), and hence
\(\sigma(\boldsymbol W_{i,k})\subseteq\sigma(Y_{i,k})\).

Fix \(i\) and first consider the odd block sums
\(\boldsymbol W_{i,1},\boldsymbol W_{i,3},\ldots\). Regard each
\(p\times p\) random matrix as a random element of
\(\mathbb R^{p^2}\) through vectorization. Berbee's coupling lemma can
then be applied recursively to this sequence. On an enlarged
probability space, it produces random matrices
\(\boldsymbol W_{i,1}^*,\boldsymbol W_{i,3}^*,\ldots\) such that each
\(\boldsymbol W_{i,k}^*\) has the same distribution as
\(\boldsymbol W_{i,k}\), the coupled odd block sums are mutually
independent, and
\[
\Pr\left(
\boldsymbol W_{i,k}^*
\neq
\boldsymbol W_{i,k}
\right)
\leq
\beta(q)
\]
for every odd \(k\). Indeed, the preceding odd blocks and the current
odd block are separated by an intervening block of length \(q\), and
\[
\beta\left(
\sigma\{\boldsymbol W_{i,1},\ldots,\boldsymbol W_{i,k-2}\},
\sigma(\boldsymbol W_{i,k})
\right)
\leq
\beta(q),
\]
because the sigma-fields generated by the matrix block sums are
contained in those generated by the corresponding state blocks.

Applying the same construction to the even block sums produces
\(\boldsymbol W_{i,2}^*,\boldsymbol W_{i,4}^*,\ldots\), which are
mutually independent and satisfy the same marginal-distribution and
mismatch-probability properties. Since the \(n\) trajectories are
independent, the auxiliary variables in these coupling constructions
may be chosen independently across \(i\). Consequently,
\[
\left\{
\boldsymbol W_{i,k}^*:
1\leq i\leq n,\ k\ {\rm odd}
\right\}
\]
is a mutually independent family, and the same is true for the
even-numbered coupled block sums.

Let \(\mathcal E_{\rm cpl}\) be the event on which
\(\boldsymbol W_{i,k}^*=\boldsymbol W_{i,k}\) for every \(i\) and \(k\).
Since \(m\leq2T/q\),
\[
\begin{aligned}
\Pr(\mathcal E_{\rm cpl}^c)
&\leq
nm\beta(q)\\
&\leq
\frac{2nT}{q}C_\beta\rho^q\\
&\leq
2C_\beta(nT\vee d)^{-4},
\end{aligned}
\]
where the last inequality follows from
\(\rho^q\leq(nT\vee d)^{-5}\).

Each \(\boldsymbol W_{i,k}^*\) is mean zero and satisfies
\(\|\boldsymbol W_{i,k}^*\|_2\leq2qR\). For either parity, there are at
most \(2nT/q\) such block sums. Therefore, the corresponding matrix
variance parameter is bounded by
\[
\frac{2nT}{q}(2qR)^2
=
8nTqR^2.
\]
Lemma~\ref{lem:matrix_concentration}(i) gives, for every \(u>0\),
\[
\Pr\left[
\left\|
\sum_{i=1}^n
\sum_{\substack{1\leq k\leq m\\k\ {\rm odd}}}
\boldsymbol W_{i,k}^*
\right\|_2
>
4R\sqrt{nTq\{u+\log(2p')\}}
+
\frac{4qR}{3}\{u+\log(2p')\}
\right]
\leq
\exp(-u),
\]
and the same bound holds for the even block sums.

On \(\mathcal E_{\rm cpl}\), \(\boldsymbol S_3\) is the sum of the
coupled odd and even block sums. Taking
\(u=4\log(nT\vee d)\) and combining the two parity bounds yields
\begin{equation}
\label{Sigma_matrix_mixing}
\frac{\|\boldsymbol S_3\|_2}{nT}
\leq
8R
\sqrt{
\frac{
q(4+C_1)\log(nT\vee d)
}{nT}
}
+
\frac{
8qR(4+C_1)\log(nT\vee d)
}{3nT}
\end{equation}
except on an event of probability at most
\[
2(nT\vee d)^{-4}
+
2C_\beta(nT\vee d)^{-4}.
\]

Let
\[
C_4=-\frac{5}{\log\rho}+1.
\]
Since
\(q\leq C_4\log(nT\vee d)\) for all sufficiently large
\(nT\vee d\), \eqref{Sigma_matrix_martingale} and
\eqref{Sigma_matrix_mixing} imply
\[
\left\|
\widehat{\boldsymbol{\Sigma}}
-
\boldsymbol{\Sigma}
\right\|_2
\leq
R
\left[
C_5
\frac{\log(nT\vee d)}{\sqrt{nT}}
+
C_6
\frac{\{\log(nT\vee d)\}^2}{nT}
\right].
\]
%where
%\[
%C_5
%=
%8\sqrt{2(4+C_1)}
%+
%8\sqrt{C_4(4+C_1)},
%\qquad
%C_6
%=
%\frac{8C_4(4+C_1)}{3}.
%\]

The total failure probability in the diverging-\(T\) case is at most
\((4+2C_\beta)(nT\vee d)^{-4}\), which is bounded by
\(\exp\{-3\log(nT\vee d)\}\) for all sufficiently large
\(nT\vee d\). The same probability bound holds in the bounded-\(T\)
case after increasing the threshold if necessary.

Define
\[
\mathcal G_3
=
\left\{
\left\|
\widehat{\boldsymbol{\Sigma}}
-
\boldsymbol{\Sigma}
\right\|_2
\leq
c_5dL
\left[
\frac{\log(nT\vee d)}{\sqrt{nT}}
+
\frac{\{\log(nT\vee d)\}^2}{nT}
\right]
\right\}.
\]
Finally, since
\(R\leq(1+\gamma)(1+c_m^2)dL\), the result follows by taking
\[
c_6=3
\]
and any \(c_5\) satisfying
\[
c_5
\geq
(1+\gamma)(1+c_m^2)
\max\{C_3,C_5,C_6\}.
\]
Thus,
\[
\Pr(\mathcal G_3)
\geq
1-\exp\{-c_6\log(nT\vee d)\}.
\]
This completes the proof.
\end{proof}

\section{Proofs of Results in Section \ref{mainsection}}
\subsection{Proof of Theorem \ref{main1}}
\begin{proof}
Throughout this proof, \(C_1,C_2,\ldots\) denote positive constants
that are local to the present proof and are re-indexed from \(C_1\).
They are distinct from the lowercase constants
\(c_1,c_2,\ldots\) appearing in the statements of the results.

Recall that
\(\widehat{\boldsymbol{\kappa}}
=
(\widehat{\boldsymbol{\alpha}}-\boldsymbol{\alpha}^*,
\widehat{\boldsymbol{\beta}}-\boldsymbol{\beta}^*)_{\oplus}\).
Define
\(S=S_0\cup\{d+k:k\in S_1\}\) and
\(\mathcal S=G_0\cup(\cup_{k\in S}G_k)\). By construction,
\(|S|\leq 2d_0\). For convenience, let
\(\widehat{\boldsymbol{\kappa}}_{G_0}
=\widehat{\boldsymbol{\alpha}}-\boldsymbol{\alpha}^*\), and, for
\(k=1,\ldots,2d\), let
\(\widehat{\boldsymbol{\kappa}}_{G_k}
=\widehat{\boldsymbol{\beta}}_{G_k}
-\boldsymbol{\beta}_{G_k}^*\).

By the definition of the estimator,
\(\sum_{k=1}^{2d}\|\widehat{\boldsymbol{\beta}}_{G_k}\|_2
\leq
\sum_{k=1}^{2d}\|\boldsymbol{\beta}_{G_k}^*\|_2\).
Moreover, \(\boldsymbol{\beta}_{G_k}^*=\boldsymbol 0\) for
\(k\notin S\). Hence,
\[
\begin{aligned}
\sum_{k=1}^{2d}\|\boldsymbol{\beta}_{G_k}^*\|_2
&\geq
\sum_{k\in S}
\left\|
\boldsymbol{\beta}_{G_k}^*
+
\widehat{\boldsymbol{\kappa}}_{G_k}
\right\|_2
+
\sum_{k\notin S}
\|\widehat{\boldsymbol{\kappa}}_{G_k}\|_2\\
&\geq
\sum_{k\in S}\|\boldsymbol{\beta}_{G_k}^*\|_2
-
\sum_{k\in S}\|\widehat{\boldsymbol{\kappa}}_{G_k}\|_2
+
\sum_{k\notin S}\|\widehat{\boldsymbol{\kappa}}_{G_k}\|_2.
\end{aligned}
\]
It follows that
\[
\sum_{k\notin S}\|\widehat{\boldsymbol{\kappa}}_{G_k}\|_2
\leq
\sum_{k\in S}\|\widehat{\boldsymbol{\kappa}}_{G_k}\|_2.
\]
Consequently,
\begin{equation}
\label{56}
\begin{aligned}
\sum_{k=0}^{2d}\|\widehat{\boldsymbol{\kappa}}_{G_k}\|_2
&\leq
\|\widehat{\boldsymbol{\kappa}}_{G_0}\|_2
+
2\sum_{k\in S}\|\widehat{\boldsymbol{\kappa}}_{G_k}\|_2\\
&\leq
2\sum_{k\in S\cup\{0\}}
\|\widehat{\boldsymbol{\kappa}}_{G_k}\|_2\\
&\leq
2\sqrt{2d_0+1}\,
\|\widehat{\boldsymbol{\kappa}}_{\mathcal S}\|_2\\
&\leq
2\sqrt{2d_0+1}\,
\|\widehat{\boldsymbol{\kappa}}\|_2.
\end{aligned}
\end{equation}

Since each group contains at most \(L\) coordinates,
\(\|\widehat{\boldsymbol{\kappa}}_{G_k}\|_1
\leq
\sqrt L\|\widehat{\boldsymbol{\kappa}}_{G_k}\|_2\). Therefore,
\[
\|\widehat{\boldsymbol{\kappa}}\|_1
=
\sum_{k=0}^{2d}
\|\widehat{\boldsymbol{\kappa}}_{G_k}\|_1
\leq
\sqrt L
\sum_{k=0}^{2d}
\|\widehat{\boldsymbol{\kappa}}_{G_k}\|_2.
\]
Combining this inequality with \eqref{56} gives
\begin{equation}
\label{kappa_l1_l2}
\|\widehat{\boldsymbol{\kappa}}\|_1
\leq
2\sqrt{L(2d_0+1)}
\|\widehat{\boldsymbol{\kappa}}\|_2.
\end{equation}

We next bound the quadratic form
\(\widehat{\boldsymbol{\kappa}}^\top
\widehat{\boldsymbol{\Sigma}}
\widehat{\boldsymbol{\kappa}}\).
By Lemma~\ref{feasiable_dantzig}, the oracle coefficient vector is
feasible on the event \(\mathcal G_1\). Since the estimated coefficient
vector is feasible by construction, on \(\mathcal G_1\), the triangle
inequality gives
\[
\left\|
\left[
\widehat{\boldsymbol{\Sigma}}
\widehat{\boldsymbol{\kappa}}
\right]_{G_k}
\right\|_2
\leq
2\lambda,
\qquad
k=0,\ldots,2d.
\]
It follows that, on \(\mathcal G_1\),
\[
\begin{aligned}
\left|
\widehat{\boldsymbol{\kappa}}^\top
\widehat{\boldsymbol{\Sigma}}
\widehat{\boldsymbol{\kappa}}
\right|
&=
\left|
\sum_{k=0}^{2d}
\widehat{\boldsymbol{\kappa}}_{G_k}^\top
\left[
\widehat{\boldsymbol{\Sigma}}
\widehat{\boldsymbol{\kappa}}
\right]_{G_k}
\right|\\
&\leq
\sum_{k=0}^{2d}
\|\widehat{\boldsymbol{\kappa}}_{G_k}\|_2
\left\|
\left[
\widehat{\boldsymbol{\Sigma}}
\widehat{\boldsymbol{\kappa}}
\right]_{G_k}
\right\|_2\\
&\leq
2\lambda
\sum_{k=0}^{2d}
\|\widehat{\boldsymbol{\kappa}}_{G_k}\|_2\\
&\leq
4\lambda\sqrt{2d_0+1}\,
\|\widehat{\boldsymbol{\kappa}}\|_2.
\end{aligned}
\]

On the other hand, Lemma~\ref{eigen_sigma} implies
\[
\left|
\widehat{\boldsymbol{\kappa}}^\top
\boldsymbol{\Sigma}
\widehat{\boldsymbol{\kappa}}
\right|
\geq
\frac{\bar c}{2}
\|\widehat{\boldsymbol{\kappa}}\|_2^2.
\]
Therefore,
\[
\begin{aligned}
\left|
\widehat{\boldsymbol{\kappa}}^\top
\widehat{\boldsymbol{\Sigma}}
\widehat{\boldsymbol{\kappa}}
\right|
&\geq
\left|
\widehat{\boldsymbol{\kappa}}^\top
\boldsymbol{\Sigma}
\widehat{\boldsymbol{\kappa}}
\right|
-
\left|
\widehat{\boldsymbol{\kappa}}^\top
(\widehat{\boldsymbol{\Sigma}}-\boldsymbol{\Sigma})
\widehat{\boldsymbol{\kappa}}
\right|\\
&\geq
\frac{\bar c}{2}
\|\widehat{\boldsymbol{\kappa}}\|_2^2
-
\|\widehat{\boldsymbol{\Sigma}}-\boldsymbol{\Sigma}\|_{\max}
\|\widehat{\boldsymbol{\kappa}}\|_1^2.
\end{aligned}
\]
By \eqref{kappa_l1_l2} and
Lemma~\ref{Asy_Sigma_infty}, on the event \(\mathcal G_2\),
\[
\begin{aligned}
\left|
\widehat{\boldsymbol{\kappa}}^\top
\widehat{\boldsymbol{\Sigma}}
\widehat{\boldsymbol{\kappa}}
\right|
\geq
\bigg[
\frac{\bar c}{2}
-
4(2d_0+1)c_3L^2
\log(nT\vee d)
\sqrt{
\frac{\log(nT\vee d)}{nT}
}
\bigg]
\|\widehat{\boldsymbol{\kappa}}\|_2^2.
\end{aligned}
\]
Under \(L\asymp(nT)^{1/(2m-1)}\) with \(m\geq3\), for all sufficiently
large \(nT\),
\begin{equation}
\label{constraint1}
\frac{\bar c}{2}
-
4(2d_0+1)c_3L^2
\log(nT\vee d)
\sqrt{
\frac{\log(nT\vee d)}{nT}
}
\geq
\frac{\bar c}{4}.
\end{equation}
Combining the upper and lower bounds for the quadratic form, on the
event \(\mathcal G_1\cap\mathcal G_2\), yields
\[
\frac{\bar c}{4}
\|\widehat{\boldsymbol{\kappa}}\|_2^2
\leq
4\lambda\sqrt{2d_0+1}\,
\|\widehat{\boldsymbol{\kappa}}\|_2.
\]
Thus,
\(\|\widehat{\boldsymbol{\kappa}}\|_2
\lesssim\lambda\sqrt{d_0}\). Since
\(\lambda\asymp
\sqrt L\log(nT\vee d)/\sqrt{nT}\), we obtain
\[
\|\widehat{\boldsymbol{\kappa}}\|_2
\lesssim
\frac{\log(nT\vee d)}{\sqrt{nT}}
\sqrt{d_0L}.
\]
Equation \eqref{kappa_l1_l2} further gives
\[
\|\widehat{\boldsymbol{\kappa}}\|_1
\lesssim
\frac{\log(nT\vee d)}{\sqrt{nT}}
d_0L.
\]
These bounds hold on the event
\(\mathcal G_1\cap\mathcal G_2\), where
\[
P(\mathcal G_1\cap\mathcal G_2)
\geq
1-
\exp\{-c_2\log(nT\vee d)\}
-
\exp\{-c_4\log(nT\vee d)\}.
\]

Furthermore, if we relax the sparsity assumption and allow the group-norm tail
\[
R_S
=
\sum_{k\notin S}
\|\boldsymbol{\beta}_{G_k}^*\|_2 =o(1).
\]
By the definition of the estimator,
\[
\sum_{k=1}^{2d}
\|\widehat{\boldsymbol{\beta}}_{G_k}\|_2
\leq
\sum_{k=1}^{2d}
\|\boldsymbol{\beta}_{G_k}^*\|_2,
\]
we obtain
\[
\sum_{k\notin S}
\|\widehat{\boldsymbol{\kappa}}_{G_k}\|_2
\leq
\sum_{k\in S}
\|\widehat{\boldsymbol{\kappa}}_{G_k}\|_2
+
2R_S.
\]
Consequently,
\[
\sum_{k=0}^{2d}
\|\widehat{\boldsymbol{\kappa}}_{G_k}\|_2
\leq
2\sqrt{2d_0+1}\,
\|\widehat{\boldsymbol{\kappa}}\|_2
+
2R_S,
\]
and, since each group contains at most \(L\) coordinates,
\[
\|\widehat{\boldsymbol{\kappa}}\|_1
\leq
2\sqrt{L(2d_0+1)}
\|\widehat{\boldsymbol{\kappa}}\|_2
+
2\sqrt{L}R_S.
\]

As in the proof of Theorem~5, feasibility of the estimated and
oracle coefficient vectors implies
\[
\begin{aligned}
\left|
\widehat{\boldsymbol{\kappa}}^\top
\widehat{\boldsymbol{\Sigma}}
\widehat{\boldsymbol{\kappa}}
\right|
&\leq
2\lambda
\sum_{k=0}^{2d}
\|\widehat{\boldsymbol{\kappa}}_{G_k}\|_2\\
&\leq
4\lambda\sqrt{2d_0+1}\,
\|\widehat{\boldsymbol{\kappa}}\|_2
+
4\lambda R_S.
\end{aligned}
\]
On the other hand, letting
\[
\eta_{nT}
=
\|\widehat{\boldsymbol{\Sigma}}-\boldsymbol{\Sigma}\|_{\max} =o(1),
\]
we have
\[
\begin{aligned}
\left|
\widehat{\boldsymbol{\kappa}}^\top
\widehat{\boldsymbol{\Sigma}}
\widehat{\boldsymbol{\kappa}}
\right|
&\geq
\frac{\bar c}{2}
\|\widehat{\boldsymbol{\kappa}}\|_2^2
-
\eta_{nT}
\|\widehat{\boldsymbol{\kappa}}\|_1^2\\
&\geq
\left\{
\frac{\bar c}{2}
-
8L(2d_0+1)\eta_{nT}
\right\}
\|\widehat{\boldsymbol{\kappa}}\|_2^2
-
8L\eta_{nT}R_S^2.
\end{aligned}
\]
 Combining the upper and lower bounds gives
\[
\frac{\bar c}{4}
\|\widehat{\boldsymbol{\kappa}}\|_2^2
\leq
4\lambda\sqrt{2d_0+1}\,
\|\widehat{\boldsymbol{\kappa}}\|_2
+
4\lambda R_S
+
8L\eta_{nT}R_S^2.
\]
Solving this quadratic inequality yields
\[
\|\widehat{\boldsymbol{\kappa}}\|_2
\lesssim
\lambda\sqrt{d_0}
+
\sqrt{\lambda R_S}
+
\sqrt{L\eta_{nT}}\,R_S.
\]

In our setting,
\[
\lambda
\asymp
\frac{\sqrt{L}\log(nT\vee d)}{\sqrt{nT}},
\qquad
\eta_{nT}
\asymp
L\log(nT\vee d)
\sqrt{
\frac{\log(nT\vee d)}{nT}
}.
\]
Therefore,
\[
\begin{aligned}
\|\widehat{\boldsymbol{\kappa}}\|_2
\lesssim{}&
\frac{\sqrt{d_0L}\log(nT\vee d)}{\sqrt{nT}}\\
&+
\frac{L^{1/4}\{\log(nT\vee d)\}^{1/2}}
{(nT)^{1/4}}
\sqrt{R_S}\\
&+
\frac{L\{\log(nT\vee d)\}^{3/4}}
{(nT)^{1/4}}
R_S.
\end{aligned}
\]
If $d_0$ is fixed and \[
R_S
=
O\left(
(nT)^{-\frac{m-1}{2m-1}}
\right),
\]
$\|\widehat{\boldsymbol{\kappa}}\|_2$ keeps the same order as it with original sparsity assumption.

We finally consider the special case where \(d=O(1)\). By
Lemma~\ref{Asy_Sigma_spectral}, on the event \(\mathcal G_3\),
\[
\|\widehat{\boldsymbol{\Sigma}}-\boldsymbol{\Sigma}\|_2
\leq
c_5dL
\left\{
\frac{\log(nT\vee d)}{\sqrt{nT}}
+
\frac{\{\log(nT\vee d)\}^2}{nT}
\right\}.
\]
Since
\[
\left|
\widehat{\boldsymbol{\kappa}}^\top
(\widehat{\boldsymbol{\Sigma}}-\boldsymbol{\Sigma})
\widehat{\boldsymbol{\kappa}}
\right|
\leq
\|\widehat{\boldsymbol{\Sigma}}-\boldsymbol{\Sigma}\|_2
\|\widehat{\boldsymbol{\kappa}}\|_2^2,
\]
condition \eqref{constraint1} can be replaced by
\[
\frac{\bar c}{2}
-
c_5dL
\left\{
\frac{\log(nT\vee d)}{\sqrt{nT}}
+
\frac{\{\log(nT\vee d)\}^2}{nT}
\right\}
\geq
\frac{\bar c}{4}.
\]
Under \(d=O(1)\), it is sufficient that
\[
L
\left\{
\frac{\log(nT\vee d)}{\sqrt{nT}}
+
\frac{\{\log(nT\vee d)\}^2}{nT}
\right\}
=o(1).
\]
Recall that the feasibility requires
\[
\frac{(nT)^{1/(2m)}}{L}
=
o(1).
\]
In particular, if
\(L\asymp(nT)^{1/(2m-1)}\), this condition holds whenever
\(2m-1>2\). On the event
\(\mathcal G_1\cap\mathcal G_3\), the same argument then gives the
preceding \(\ell_2\)- and \(\ell_1\)-error bounds, where
\[
P(\mathcal G_1\cap\mathcal G_3)
\geq
1-
\exp\{-c_2\log(nT\vee d)\}
-
\exp\{-c_6\log(nT\vee d)\}.
\]

\end{proof}
\subsection{Proof of Theorem \ref{main2}}
\begin{proof}
Throughout this proof, \(C_1,C_2,\ldots\) denote positive constants
that are local to the present proof and are re-indexed from \(C_1\).
They are distinct from the lowercase constants
\(c_1,c_2,\ldots\) appearing in the statements of the results.

It follows from the triangle inequality that
\[
|\hat{V}^{\pi}(\tilde{x})-V^{\pi}(\tilde{x})|
\leq
|\hat{V}^{\pi}(\tilde{x})-V^{\pi*}(\tilde{x})|
+
|V^{\pi*}(\tilde{x})-V^{\pi}(\tilde{x})|.
\]
From inequality~\eqref{funcional_est}, we have
\[
|V^{\pi}(\tilde{x})-V^{\pi*}(\tilde{x})|
\leq
d_0C_1L^{-m},
\]
and
\[
\mathbb{E}_{\tilde{x}}
|V^{\pi}(\tilde{x})-V^{\pi*}(\tilde{x})|
\leq
d_0C_1L^{-m}.
\]
Thus, the remaining proof focuses on the estimation error
\[
|\hat{V}^{\pi}(\tilde{x})-V^{\pi*}(\tilde{x})|.
\]
We have
\begin{equation}
\begin{aligned}
|\hat{V}^{\pi}(\tilde{x})-V^{\pi*}(\tilde{x})|
&=
|U^T(\tilde{x})\hat{\boldsymbol{\kappa}}| \\
&=
\left|
\sum_{k=0}^{2d}
U_{G_k}^T(\tilde{x})
\hat{\boldsymbol{\kappa}}_{G_k}
\right| \\
&\leq
\sum_{k=0}^{2d}
\left|
U_{G_k}^T(\tilde{x})
\hat{\boldsymbol{\kappa}}_{G_k}
\right| \\
&\leq
\sum_{k=0}^{2d}
\|U_{G_k}^T(\tilde{x})\|_2
\|\hat{\boldsymbol{\kappa}}_{G_k}\|_2.
\end{aligned}
\end{equation}
Note that
\[
\|U_{G_0}^T(\tilde{x})\|_2
=
\sqrt{
\pi(0\mid\tilde{x})^2+
\pi(1\mid\tilde{x})^2
}
\leq
\sqrt{2}
\leq
\sqrt{L},
\]
and
\[
\|U_{G_k}^T(\tilde{x})\|_2
\leq
\|\Phi_k(\tilde{x})\|_2.
\]
Combined with Remark~\ref{basis_bound}, we have
\begin{equation}
\begin{aligned}
\sum_{k=0}^{2d}
\|U_{G_k}^T(\tilde{x})\|_2
\|\hat{\boldsymbol{\kappa}}_{G_k}\|_2
&\leq
c_m\sqrt{L}
\sum_{k=0}^{2d}
\|\hat{\boldsymbol{\kappa}}_{G_k}\|_2 \\
&\leq
\sqrt{4c_m^2L(2d_0+1)}
\|\hat{\boldsymbol{\kappa}}\|_2.
\end{aligned}
\end{equation}
The second inequality follows from equation~\eqref{56}. Combined with
Theorem~\ref{main1}, on the event
\(\mathcal G_1\cap\mathcal G_2\), we have
\begin{equation}
\sup_{\tilde{x}\in\mathcal{X}^d}
|\hat{V}^{\pi}(\tilde{x})-V^{\pi*}(\tilde{x})|
\lesssim
d_0L
\frac{\log(nT\vee d)}{\sqrt{nT}}.
\end{equation}

Under the choice \(L=(nT)^{1/(2m-1)}\), the functional approximation
error satisfies
\[
|V^{\pi}(\tilde{x})-V^{\pi*}(\tilde{x})|
\leq
d_0C_1L^{-m}
=
d_0C_1(\sqrt{nT})^{-2m/(2m-1)}
\lesssim
d_0(\sqrt{nT})^{-1}.
\]
Therefore, on the same event,
\begin{equation}
\begin{aligned}
%\label{m2e1}
\sup_{\tilde{x}\in\mathcal{X}^d}
|\hat{V}^{\pi}(\tilde{x})-V^{\pi}(\tilde{x})|
&\lesssim
d_0L
\frac{\log(nT\vee d)}{\sqrt{nT}}
+
d_0(\sqrt{nT})^{-1} \\
&\lesssim
d_0L
\frac{\log(nT\vee d)}{\sqrt{nT}}.
\end{aligned}
\end{equation}

Next, we establish the bound for the generalized test error. Assume that
the density \(\mu_{\tilde{x}}\) is uniformly bounded by
\(\tilde{c}_{\mu}\). By definition,
\begin{align}
&
\left|
\mathbb{E}_{\tilde{x}}\hat{V}^{\pi}(\tilde{x})
-
\mathbb{E}_{\tilde{x}}V^{\pi*}(\tilde{x})
\right|
\nonumber\\
&\quad=
\left|
\int_{\tilde{x}\in\mathcal{X}^d}
U^T(\tilde{x})
\hat{\boldsymbol{\kappa}}
\mu_{\tilde{x}}
\,d\tilde{x}
\right|
\nonumber\\
&\quad\leq
\tilde{c}_{\mu}
\sum_{k=0}^{2d}
\int_{\tilde{x}\in\mathcal{X}^d}
\left|
U_{G_k}^T(\tilde{x})
\hat{\boldsymbol{\kappa}}_{G_k}
\right|
\,d\tilde{x}.
\label{eq:integrated_value_first}
\end{align}
It follows from the Cauchy--Schwarz inequality that
\begin{align}
&
\int_{\tilde{x}\in\mathcal{X}^d}
\left|
U_{G_k}^T(\tilde{x})
\hat{\boldsymbol{\kappa}}_{G_k}
\right|
\,d\tilde{x}
\nonumber\\
&\quad=
\int_{\tilde{x}\in\mathcal{X}^d}
\left\{
\hat{\boldsymbol{\kappa}}_{G_k}^T
U_{G_k}(\tilde{x})
U_{G_k}^T(\tilde{x})
\hat{\boldsymbol{\kappa}}_{G_k}
\right\}^{1/2}
\,d\tilde{x}
\nonumber\\
&\quad\lesssim
\left\{
\int_{\tilde{x}\in\mathcal{X}^d}
\hat{\boldsymbol{\kappa}}_{G_k}^T
U_{G_k}(\tilde{x})
U_{G_k}^T(\tilde{x})
\hat{\boldsymbol{\kappa}}_{G_k}
\,d\tilde{x}
\right\}^{1/2}
\nonumber\\
&\quad=
\left[
\hat{\boldsymbol{\kappa}}_{G_k}^T
\left\{
\int_{\tilde{x}\in\mathcal{X}^d}
U_{G_k}(\tilde{x})
U_{G_k}^T(\tilde{x})
\,d\tilde{x}
\right\}
\hat{\boldsymbol{\kappa}}_{G_k}
\right]^{1/2},
\label{eq:integrated_group_bound}
\end{align}
where the first inequality also uses the boundedness of
\(\mathcal{X}^d\).

When \(\pi\) is a stationary deterministic policy, we have
\[
0\leq\pi(a\mid\tilde{x})\leq 1
\]
and
\[
\pi(0\mid\tilde{x})\pi(1\mid\tilde{x})=0.
\]
Therefore,
\[
U_{G_k}(\tilde{x})U_{G_k}^T(\tilde{x})
=
\operatorname{Diag}
\left\{
\Phi_k(\tilde{x})\Phi_k^T(\tilde{x})
\pi(0\mid\tilde{x}),
\,
\Phi_k(\tilde{x})\Phi_k^T(\tilde{x})
\pi(1\mid\tilde{x})
\right\}
\]
is a block-diagonal matrix.

For any
\[
\boldsymbol{a}
=
(\boldsymbol{a}_0^T,\boldsymbol{a}_1^T)^T
\in\mathbb{R}^{2L},
\qquad
\boldsymbol{a}_0,\boldsymbol{a}_1\in\mathbb{R}^L,
\]
we have
\begin{align}
&
\boldsymbol{a}^T
\left\{
\int_{\tilde{x}\in\mathcal{X}^d}
U_{G_k}(\tilde{x})
U_{G_k}^T(\tilde{x})
\,d\tilde{x}
\right\}
\boldsymbol{a}
\nonumber\\
&\quad=
\boldsymbol{a}_0^T
\left\{
\int_{\tilde{x}\in\mathcal{X}^d}
\Phi_k(\tilde{x})
\Phi_k^T(\tilde{x})
\pi(0\mid\tilde{x})
\,d\tilde{x}
\right\}
\boldsymbol{a}_0
\nonumber\\
&\qquad+
\boldsymbol{a}_1^T
\left\{
\int_{\tilde{x}\in\mathcal{X}^d}
\Phi_k(\tilde{x})
\Phi_k^T(\tilde{x})
\pi(1\mid\tilde{x})
\,d\tilde{x}
\right\}
\boldsymbol{a}_1
\nonumber\\
&\quad\leq
\lambda_{\max}
\left\{
\int_{\tilde{x}\in\mathcal{X}^d}
\Phi_k(\tilde{x})
\Phi_k^T(\tilde{x})
\,d\tilde{x}
\right\}
\left(
\|\boldsymbol{a}_0\|_2^2+
\|\boldsymbol{a}_1\|_2^2
\right)
\nonumber\\
&\quad=
\lambda_{\max}
\left\{
\int_{\tilde{x}\in\mathcal{X}^d}
\Phi_k(\tilde{x})
\Phi_k^T(\tilde{x})
\,d\tilde{x}
\right\}
\|\boldsymbol{a}\|_2^2
\nonumber\\
&\quad\lesssim
\|\boldsymbol{a}\|_2^2.
\label{eq:integrated_gram_bound}
\end{align}
The last inequality follows from Lemma~\ref{Bs_bound_uni}. The same
bound holds for the intercept group because
\[
U_{G_0}(\tilde{x})U_{G_0}^T(\tilde{x})
=
\operatorname{Diag}
\left\{
\pi(0\mid\tilde{x}),
\pi(1\mid\tilde{x})
\right\}.
\]
Combining \eqref{eq:integrated_group_bound} and
\eqref{eq:integrated_gram_bound}, we obtain
\begin{equation}
\left|
\mathbb{E}_{\tilde{x}}\hat{V}^{\pi}(\tilde{x})
-
\mathbb{E}_{\tilde{x}}V^{\pi*}(\tilde{x})
\right|
\lesssim
\sum_{k=0}^{2d}
\|\hat{\boldsymbol{\kappa}}_{G_k}\|_2.
\end{equation}
Combining this result with Theorem~\ref{main1}, on the event
\(\mathcal G_1\cap\mathcal G_2\), we obtain
\begin{equation}
\left|
\mathbb{E}_{\tilde{x}}\hat{V}^{\pi}(\tilde{x})
-
\mathbb{E}_{\tilde{x}}V^{\pi*}(\tilde{x})
\right|
\lesssim
d_0\sqrt{L}
\frac{\log(nT\vee d)}{\sqrt{nT}}.
\end{equation}

Moreover, by inequality~\eqref{funcional_est},
\begin{equation}
\begin{aligned}
\left|
\mathbb{E}_{\tilde{x}}
\left[
V^{\pi*}(\tilde{x})-V^{\pi}(\tilde{x})
\right]
\right|
&\leq
\mathbb{E}_{\tilde{x}}
\left|
V^{\pi*}(\tilde{x})-V^{\pi}(\tilde{x})
\right| \\
&\leq
d_0C_1L^{-m}.
\end{aligned}
\end{equation}
Therefore, by the triangle inequality, on
\(\mathcal G_1\cap\mathcal G_2\),
\begin{equation}
\begin{aligned}
\left|
\mathbb{E}_{\tilde{x}}\hat{V}^{\pi}(\tilde{x})
-
\mathbb{E}_{\tilde{x}}V^{\pi}(\tilde{x})
\right|
&\leq
\left|
\mathbb{E}_{\tilde{x}}\hat{V}^{\pi}(\tilde{x})
-
\mathbb{E}_{\tilde{x}}V^{\pi*}(\tilde{x})
\right| \\
&\quad+
\left|
\mathbb{E}_{\tilde{x}}V^{\pi*}(\tilde{x})
-
\mathbb{E}_{\tilde{x}}V^{\pi}(\tilde{x})
\right| \\
&\lesssim
d_0\sqrt{L}
\frac{\log(nT\vee d)}{\sqrt{nT}}
+
d_0L^{-m}.
\end{aligned}
\end{equation}
Under the choice \(L=(nT)^{1/(2m-1)}\), the approximation error
\(d_0L^{-m}\) is dominated by the preceding estimation error. Hence,
on the same event,
\begin{equation}
\left|
\mathbb{E}_{\tilde{x}}\hat{V}^{\pi}(\tilde{x})
-
\mathbb{E}_{\tilde{x}}V^{\pi}(\tilde{x})
\right|
\lesssim
d_0\sqrt{L}
\frac{\log(nT\vee d)}{\sqrt{nT}}.
\end{equation}

Finally,
\[
P(\mathcal G_1\cap\mathcal G_2)
\geq
1-
\exp\{-c_2\log(nT\vee d)\}
-
\exp\{-c_4\log(nT\vee d)\}.
\]
This completes the proof.
\end{proof}

\begin{proof}

It follows from the triangle inequality that
\[
|\hat{V}^{\pi}(\tilde{x})-V^{\pi}(\tilde{x})|
\leq
|\hat{V}^{\pi}(\tilde{x})-V^{\pi*}(\tilde{x})|
+
|V^{\pi*}(\tilde{x})-V^{\pi}(\tilde{x})|.
\]
From inequality~(\ref{funcional_est}), we have
\[
|V^{\pi}(\tilde{x})-V^{\pi*}(\tilde{x})|
\leq
d_0 C L^{-m},
\]
and
\[
\mathbb{E}_{\tilde{x}}
|V^{\pi}(\tilde{x})-V^{\pi*}(\tilde{x})|
\leq
d_0 C L^{-m}.
\]
Thus, the remaining proof focuses on the estimation error
\[
|\hat{V}^{\pi}(\tilde{x})-V^{\pi*}(\tilde{x})|.
\]

    \begin{equation}
    \begin{aligned}
       |\hat{V}^{\pi}(\tilde{x})-V^{\pi*}(\tilde{x})| = |U^T(\tilde{x})\hat{\boldsymbol{\kappa}}|&= |\sum_{k=0}^{2d}U_{G_k}^T(\tilde{x})\hat{\boldsymbol{\kappa}}_{G_k}|\leq \sum_{k=0}^{2d}|U_{G_k}^T(\tilde{x})\hat{\boldsymbol{\kappa}}_{G_k}|\\&\leq \sum_{k=0}^{2d}\|U_{G_k}^T(\tilde{x})\|_2\|\hat{\boldsymbol{\kappa}}_{G_k}\|_2.           
    \end{aligned}
\end{equation}
Note that$\|U^T_{G_0}(\tilde{x})\|_2 = \sqrt{\pi(0\mid \tilde{x})^2+\pi(1\mid \tilde{x})^2}\leq \sqrt{2} \leq \sqrt{L}$ and $\|U^T_{G_k}(\tilde{x})\|_2\leq \|\Phi_k(\tilde{x})\|_2$. Combined with the remark\ref{basis_bound}, we have
\begin{equation}
\sum_{k=0}^{2d}\|U_{G_k}^T(\tilde{x})\|_2\|\hat{\boldsymbol{\kappa}}_{G_k}\|_2\leq c_m\sqrt{L}\sum_{k=0}^{2d}\|\hat{\boldsymbol{\kappa}}_{G_k}\|_2\leq \sqrt{4c_m^2L(2d_0+1)}\|\hat{\boldsymbol{\kappa}}\|_2
\end{equation}
The second inequality follows from the equation(\ref{56}). Combined with Theorem\ref{main1}, we have 
\begin{equation}
        \sup_{\tilde{x}\in \mathcal{X}^d }|\hat{V}^{\pi}(\tilde{x})-V^{\pi*}(\tilde{x})|\lesssim d_0L\frac{\log(nT \vee d)}{\sqrt{nT}}
\end{equation}
with probability at least $1- \exp\{-c_3\log(nT \vee d)\}$. With the condition $L = (nT)^{1/(2m-1)}$, the functional approximation error $|\hat{V}^{\pi}(\tilde{x})-V^{\pi*}(\tilde{x})| \lesssim (\sqrt{nT})^{(-2m)/(2m-1)} \leq \sqrt{nT}^{-1}$ almost surely. With the same probability, we have
\begin{equation}\begin{aligned}\label{m2e1}
         \sup_{\tilde{x}\in \mathcal{X}^d }|\hat{V}^{\pi}(\tilde{x})-V^{\pi}(\tilde{x})|&\lesssim d_0L\frac{\log(nT \vee d)}{\sqrt{nT}}+d_0(\sqrt{nT})^{-1}\\   & \lesssim d_0L\frac{\log(nT \vee d)}{\sqrt{nT}}
\end{aligned}
\end{equation}
Next, we will show the bound of generalized test error. We assume $\mu_{\tilde{x}}\leq \tilde{c}_{\mu}$. By the definition
\begin{equation}
        |\mathbb{E}_{\tilde{x}}\hat{V}^{\pi}(\tilde{x})-\mathbb{E}_{\tilde{x}}V^{\pi*}(\tilde{x})|=\bigg|\int_{\tilde{x}\in \mathcal{X}} U^T(\tilde{x})\hat{\boldsymbol{\kappa}} \mu_{\tilde{x}}dx\bigg|\leq \tilde{c}_{\mu}\sum_{k=0}^{2d}\bigg|\int_{\tilde{x}\in \mathcal{X}} U^T_{G_k}(\tilde{x})\hat{\boldsymbol{\kappa}}_{G_k} dx\bigg|
\end{equation}
It follows from Jensen's inequality that
\begin{equation}
\begin{aligned}
        \bigg|\int_{\tilde{x}\in \mathcal{X}} U^T_{G_k}(\tilde{x})\hat{\boldsymbol{\kappa}}_{G_k} dx\bigg|&\leq    \int_{\tilde{x}\in \mathcal{X}} |U^T_{G_k}(\tilde{x})\hat{\boldsymbol{\kappa}}_{G_k} |dx \\ 
        &=\int_{\tilde{x}\in \mathcal{X}} \bigg(\hat{\boldsymbol{\kappa}}^T_{G_k}U_{G_k}(\tilde{x})U^T_{G_k}(\tilde{x})\hat{\boldsymbol{\kappa}}_{G_k}\bigg)^{1/2} dx \\
        &\leq \bigg(\int_{\tilde{x}\in \mathcal{X}} \hat{\boldsymbol{\kappa}}^T_{G_k}U_{G_k}(\tilde{x})U^T_{G_k}(\tilde{x})\hat{\boldsymbol{\kappa}}_{G_k} dx \bigg)^{1/2}\\
        &= \bigg(\hat{\boldsymbol{\kappa}}^T_{G_k}\bigg\{\int_{\tilde{x}\in \mathcal{X}} U_{G_k}(\tilde{x})U^T_{G_k}(\tilde{x}) dx \bigg\}\hat{\boldsymbol{\kappa}}_{G_k}\bigg)^{1/2}
\end{aligned}
\end{equation}
When $\pi $ is a stationary deterministic policy,   we have $\pi(\cdot|\tilde{x})\leq 1$ and $\pi(0\mid \tilde{x})\pi(1\mid \tilde{x}) = 0$. Therefore,
 $ U_{G_k}(\tilde{x})U^T_{G_k}(\tilde{x}) 
 = \text{Diag}(\Phi_k(\tilde{x})\Phi_k^T(\tilde{x})\pi(0\mid \tilde{x}),\Phi_k(\tilde{x})\Phi_k^T(\tilde{x})\pi(1\mid \tilde{x}))$ is a block diagonal matrix. For $\forall \boldsymbol{a} = (\boldsymbol{a}_0^T,\boldsymbol{a}_1^T)^T \in \mathbb{R}^{2L}, \boldsymbol{a}_0,\boldsymbol{a}_1 \in \mathbb{R}^{L}$
 \begin{equation}
 \begin{aligned}
      \boldsymbol{a}^T\bigg(\int_{\tilde{x}\in \mathcal{X}} U_{G_k}(\tilde{x})U^T_{G_k}(\tilde{x}) dx\bigg)\boldsymbol{a} &= \boldsymbol{a}_0^T\bigg(\int_{\tilde{x}\in \mathcal{X}}\Phi_k(\tilde{x})\Phi_k^T(\tilde{x}) \pi(0\mid \tilde{x})dx\bigg)\boldsymbol{a}_0\\
    &\quad +\boldsymbol{a}_1^T\bigg(\int_{\tilde{x}\in \mathcal{X}}\Phi_k(\tilde{x})\Phi_k^T(\tilde{x}) \pi(1\mid \tilde{x})dx\bigg)\boldsymbol{a}_1\\
     &\leq \lambda_{max}\bigg(\int_{\tilde{x}\in \mathcal{X}}\Phi_k(\tilde{x})\Phi_k^T(\tilde{x}) dx\bigg)(\|\boldsymbol{a}_0\|_2^2+\|\boldsymbol{a}_1\|_2^2)\\
     &\leq \lambda_{max}\bigg(\int_{\tilde{x}\in \mathcal{X}}\Phi_k(\tilde{x})\Phi_k^T(\tilde{x}) dx\bigg)\|\boldsymbol{a}\|_2\\
     & \lesssim \|\boldsymbol{a}\|_2
 \end{aligned}
 \end{equation}
 The last the inequality follows from lemma \ref{Bs_bound_uni} and the fact:  $U_{G_0}(\tilde{x})U^T_{G_0} = \text{Diag}(\pi(0\mid \tilde{x}),\pi(1\mid \tilde{x}))$. To summarize, we proved 
 \begin{equation}
     |\mathbb{E}_{\tilde{x}}\hat{V}^{\pi}(\tilde{x})-\mathbb{E}_{\tilde{x}}V^{\pi*}(\tilde{x})|\lesssim \sum_{k=0}^{2d}\|\hat{\boldsymbol{\kappa}}_{G_k}\|_2
 \end{equation}
 Combining this result with the proof of Theorem~\ref{main1}, we obtain
\begin{equation}
\left|
\mathbb{E}_{\tilde{x}}\hat{V}^{\pi}(\tilde{x})
-
\mathbb{E}_{\tilde{x}}V^{\pi*}(\tilde{x})
\right|
\lesssim
d_0\sqrt{L}
\frac{\log(nT\vee d)}{\sqrt{nT}}
\end{equation}
with probability at least
$1-\exp\{-c_3\log(nT\vee d)\}$.

Moreover, by inequality~(\ref{funcional_est}),
\begin{equation}
\begin{aligned}
\left|
\mathbb{E}_{\tilde{x}}
\left[
V^{\pi*}(\tilde{x})-V^{\pi}(\tilde{x})
\right]
\right|
&\leq
\mathbb{E}_{\tilde{x}}
\left|
V^{\pi*}(\tilde{x})-V^{\pi}(\tilde{x})
\right| \\
&\leq
d_0 C L^{-m}.
\end{aligned}
\end{equation}
Therefore, by the triangle inequality,
\begin{equation}
\begin{aligned}
\left|
\mathbb{E}_{\tilde{x}}\hat{V}^{\pi}(\tilde{x})
-
\mathbb{E}_{\tilde{x}}V^{\pi}(\tilde{x})
\right|
&\leq
\left|
\mathbb{E}_{\tilde{x}}\hat{V}^{\pi}(\tilde{x})
-
\mathbb{E}_{\tilde{x}}V^{\pi*}(\tilde{x})
\right| \\
&\quad+
\left|
\mathbb{E}_{\tilde{x}}V^{\pi*}(\tilde{x})
-
\mathbb{E}_{\tilde{x}}V^{\pi}(\tilde{x})
\right| \\
&\lesssim
d_0\sqrt{L}
\frac{\log(nT\vee d)}{\sqrt{nT}}
+
d_0L^{-m}.
\end{aligned}
\end{equation}
Under the choice $L=(nT)^{1/(2m-1)}$, the approximation error
$d_0L^{-m}$ is dominated by the preceding estimation error. Hence,
with the same probability,
\begin{equation}
\left|
\mathbb{E}_{\tilde{x}}\hat{V}^{\pi}(\tilde{x})
-
\mathbb{E}_{\tilde{x}}V^{\pi}(\tilde{x})
\right|
\lesssim
d_0\sqrt{L}
\frac{\log(nT\vee d)}{\sqrt{nT}}.
\end{equation}
\end{proof}

\subsection{Proof of Theorem \ref{main3}}
\begin{proof}
Throughout this proof, $C_1,C_2,\ldots$ denote positive constants
that are local to the present proof and are re-indexed from $C_1$.
They are distinct from the lowercase constants
$c_1,c_2,\ldots$ appearing in the statements of the results.

Denote
\[
q^{\pi *}_{s,a}(x_s)
=
\Phi_s^{T}(x_s)\boldsymbol{\beta}_{s,a}^*.
\]
In this proof, we omit the superscript $\pi$ from
$q$, $\widehat q$, and $q^*$ for simplicity. The expectation is taken
with respect to the reference distribution $x\sim\mathbb G$.

Consider an active feature $x_s$, where $s\in\mathcal K$. Under the
condition in Theorem~\ref{main3}, there exists an action
$a\in\{0,1\}$ such that
\[
\mathbb E\left\{q_{s,a}^2(x_s)\right\}
\geq
v.
\]
By \eqref{funcional_est},
\[
\left|
q_{s,a}(x_s)-q_{s,a}^*(x_s)
\right|
\leq
C_1L^{-m}
\]
uniformly over all action-state pairs. We have
\begin{equation}
\begin{aligned}
\label{96}
\mathbb E\left\{q_{s,a}^2(x_s)\right\}
&=
\mathbb E
\left[
\left\{
q_{s,a}^*(x_s)
+
q_{s,a}(x_s)-q_{s,a}^*(x_s)
\right\}^2
\right] \\
&\leq
2\mathbb E\left\{\left(q_{s,a}^*(x_s)\right)^2\right\}
+
2\mathbb E
\left[
\left\{
q_{s,a}(x_s)-q_{s,a}^*(x_s)
\right\}^2
\right] \\
&\leq
2\mathbb E\left\{\left(q_{s,a}^*(x_s)\right)^2\right\}
+
2C_1^2L^{-2m}.
\end{aligned}
\end{equation}
Thus,
\[
\mathbb E\left\{\left(q_{s,a}^*(x_s)\right)^2\right\}
\geq
\frac{v}{2}-C_1^2L^{-2m}.
\]
Since
\[
L\asymp(nT)^{1/(2m-1)},
\]
for all sufficiently large $nT$,
\[
\mathbb E\left\{\left(q_{s,a}^*(x_s)\right)^2\right\}
\geq
\frac{v}{4}.
\]

On the event $\mathcal G_1\cap\mathcal G_2$, Theorem~\ref{main1}
implies that
\[
\|\widehat{\boldsymbol{\kappa}}\|_2
\leq
c_{\kappa}\sqrt{L}
\frac{\log(nT\vee d)}{\sqrt{nT}}.
\]
Using the same inequality as in \eqref{96}, we have
\begin{equation}
\begin{aligned}
\label{97}
\mathbb E\left\{\left(q_{s,a}^*(x_s)\right)^2\right\}
&=
\mathbb E
\left[
\left\{
\widehat q_{s,a}(x_s)
+
q_{s,a}^*(x_s)-\widehat q_{s,a}(x_s)
\right\}^2
\right] \\
&\leq
(1+0.5)
\mathbb E\left\{\widehat q_{s,a}^2(x_s)\right\}
+
(1+2)
\mathbb E
\left[
\left\{
\widehat q_{s,a}(x_s)-q_{s,a}^*(x_s)
\right\}^2
\right] \\
&\leq
1.5\mathbb E\left\{\widehat q_{s,a}^2(x_s)\right\}
+
3\mathbb E
\left\{
\|\Phi_s(x_s)\|_2^2
\|\widehat{\boldsymbol{\kappa}}_{G_s}\|_2^2
\right\} \\
&\leq
1.5\mathbb E\left\{\widehat q_{s,a}^2(x_s)\right\}
+
3
\left\{
c_mc_{\kappa}L
\frac{\log(nT\vee d)}{\sqrt{nT}}
\right\}^2.
\end{aligned}
\end{equation}
The last inequality follows from
\[
\|\widehat{\boldsymbol{\kappa}}_{G_s}\|_2
\leq
\|\widehat{\boldsymbol{\kappa}}\|_2
\]
and the definition
\[
c_{\kappa}
=
\frac{
8c_1\sqrt{2d_0+1}
}{
\bar c
}.
\]

Suppose that there exists an index $s$ such that
$s\in\mathcal K$ but $s\notin\widehat{\mathcal K}$. Then, for every
$a\in\{0,1\}$,
\begin{equation}
\begin{aligned}
\label{98}
\mathbb E\left\{\widehat q_{s,a}^2(x_s)\right\}
&\leq
\mathbb E
\left\{
\|\Phi_s(x_s)\|_2^2
\|\widehat{\boldsymbol{\beta}}_{G_s}\|_2^2
\right\} \\
&\leq
(c_m\sqrt L)^2h^2 \\
&\leq
\frac{v}{8}.
\end{aligned}
\end{equation}
Combining \eqref{96}, \eqref{97}, and \eqref{98}, we obtain
\[
\frac{v}{4}
\leq
\frac{3v}{16}
+
3
\left\{
c_mc_{\kappa}L
\frac{\log(nT\vee d)}{\sqrt{nT}}
\right\}^2.
\]
Consequently,
\[
\frac{v}{16}
\leq
3
\left\{
c_mc_{\kappa}L
\frac{\log(nT\vee d)}{\sqrt{nT}}
\right\}^2,
\]
and hence
\[
v
\lesssim
\left\{
L
\frac{\log(nT\vee d)}{\sqrt{nT}}
\right\}^2.
\]
However, under
\[
L\asymp(nT)^{1/(2m-1)},
\]
we have
\[
L\frac{\log(nT\vee d)}{\sqrt{nT}}
\longrightarrow
0
\]
as $nT\to\infty$. This contradicts the assumption that $v$ is bounded
away from zero. Therefore, on the event
$\mathcal G_1\cap\mathcal G_2$,
\[
\mathcal K
\subseteq
\widehat{\mathcal K}.
\]

Since $\mathcal G_1\cap\mathcal G_2$ occurs with the probability stated
in Theorem~\ref{main1}, the desired result follows.
\end{proof}
\section{Numerical Result}\label{app:numerical}

\begin{table}[p]
\centering

\begin{adjustbox}{angle=0,center,max width=\textheight}
\begin{minipage}{0.95\textheight}
\centering
\caption{Simulation results for $d=10$. Each entry reports bias (SD) / RMSE.}
\label{tab:d10_rho07}

\scriptsize
\renewcommand{\arraystretch}{1.25}
\setlength{\tabcolsep}{5pt}

\begin{tabular}{c c c c}
\toprule
$n \backslash T$ & $T=20$ & $T=30$ & $T=40$ \\
\midrule

$n=20$
&
\rescellfirst{ 0.12}{0.46}{0.46}{-0.17}{0.01}{0.17}{ 0.02}{0.17}{0.17}
&
\rescell{ 0.01}{0.25}{0.25}{-0.16}{0.02}{0.16}{-0.03}{0.16}{0.16}
&
\rescell{ 0.05}{0.15}{0.15}{-0.14}{0.06}{0.15}{ 0.01}{0.12}{0.12}
\\[1.2em]
\hline

$n=30$
&
\rescellfirst{-0.02}{0.19}{0.19}{-0.17}{0.00}{0.17}{ 0.02}{0.10}{0.10}
&
\rescell{-0.02}{0.14}{0.14}{-0.16}{0.01}{0.17}{-0.03}{0.11}{0.11}
&
\rescell{ 0.03}{0.12}{0.12}{-0.15}{0.04}{0.15}{ 0.02}{0.10}{0.10}
\\[1.2em]
\hline

$n=40$
&
\rescellfirst{ 0.03}{0.11}{0.11}{-0.17}{0.00}{0.17}{ 0.02}{0.08}{0.08}
&
\rescell{-0.01}{0.09}{0.08}{-0.17}{0.01}{0.17}{-0.01}{0.08}{0.08}
&
\rescell{ 0.03}{0.09}{0.10}{-0.15}{0.03}{0.15}{ 0.02}{0.09}{0.09}
\\[1.2em]
\hline

$n=50$
&
\rescellfirst{ 0.04}{0.13}{0.13}{-0.17}{0.00}{0.17}{ 0.01}{0.07}{0.07}
&
\rescell{ 0.01}{0.09}{0.09}{-0.17}{0.01}{0.17}{ 0.01}{0.08}{0.08}
&
\rescell{ 0.03}{0.07}{0.07}{-0.15}{0.02}{0.16}{ 0.02}{0.07}{0.07}
\\

\bottomrule
\end{tabular}
\end{minipage}
\end{adjustbox}

\end{table}

\begin{table}[p]
\centering

\begin{adjustbox}{angle=0,center,max width=\textheight}
\begin{minipage}{0.95\textheight}
\centering
\caption{Simulation results for $d=100$. Each entry reports bias (SD) / RMSE.}
\label{tab:d100_rho07}

\scriptsize
\renewcommand{\arraystretch}{1.25}
\setlength{\tabcolsep}{5pt}

\begin{tabular}{c c c c}
\toprule
$n \backslash T$ & $T=20$ & $T=30$ & $T=40$ \\
\midrule

$n=20$
&
\rescellfirst{ 0.09}{0.18}{0.19}{ 0.73}{0.19}{0.75}{ 0.03}{0.21}{0.20}
&
\rescell{ 0.07}{0.24}{0.24}{ 0.42}{0.23}{0.47}{-0.01}{0.18}{0.17}
&
\rescell{-0.03}{0.28}{0.27}{ 0.23}{0.19}{0.29}{-0.06}{0.20}{0.20}
\\[1.2em]
\hline

$n=30$
&
\rescellfirst{ 0.07}{0.15}{0.15}{ 0.66}{0.15}{0.67}{ 0.02}{0.14}{0.13}
&
\rescell{-0.02}{0.22}{0.21}{ 0.43}{0.16}{0.46}{-0.02}{0.16}{0.16}
&
\rescell{-1.83}{5.64}{5.65}{ 0.26}{0.14}{0.29}{-0.05}{0.16}{0.16}
\\[1.2em]
\hline

$n=40$
&
\rescellfirst{ 0.09}{0.11}{0.14}{ 0.66}{0.13}{0.67}{ 0.03}{0.12}{0.11}
&
\rescell{-1.34}{2.95}{3.10}{ 0.44}{0.09}{0.44}{ 0.00}{0.15}{0.14}
&
\rescell{ 0.58}{1.78}{1.78}{ 0.25}{0.14}{0.29}{-0.05}{0.12}{0.13}
\\[1.2em]
\hline

$n=50$
&
\rescellfirst{ 0.05}{0.16}{0.16}{ 0.65}{0.11}{0.66}{ 0.01}{0.10}{0.10}
&
\rescell{ 7.68}{17.78}{18.53}{ 0.42}{0.09}{0.43}{-0.03}{0.12}{0.12}
&
\rescell{ 0.27}{1.49}{1.44}{ 0.24}{0.15}{0.28}{-0.04}{0.12}{0.12}
\\

\bottomrule
\end{tabular}
\end{minipage}
\end{adjustbox}

\end{table}

\begin{table}[p]
\begin{adjustbox}{angle=0,center,max width=\textheight}
\begin{minipage}{0.95\textheight}
\centering
\caption{Simulation Results of Cartpole Environment. Each entry reports bias (SD) / RMSE}
\label{table1}

\scriptsize
\renewcommand{\arraystretch}{1.25}
\setlength{\tabcolsep}{4pt}

\begin{tabular}{c c c c}
\toprule
$n \backslash T$ & $T=50$ & $T=60$ & $T=70$ \\
\midrule

$n=50$
&
\rescellfirst{-0.33}{12.85}{12.53}{-12.85}{0.07}{12.86}{8.18}{3.28}{8.78}
&
\rescell{-9.46}{58.79}{58.08}{-7.38}{0.20}{7.38}{2.66}{1.04}{2.85}
&
\rescell{0.59}{5.43}{5.32}{-3.25}{0.33}{3.26}{1.75}{0.61}{1.85}
\\[1.2em]
\hline

$n=60$
&
\rescellfirst{-0.88}{14.76}{14.41}{-12.84}{0.07}{12.84}{6.25}{2.32}{6.65}
&
\rescell{0.95}{25.33}{24.71}{-7.52}{0.21}{7.53}{2.97}{1.00}{3.12}
&
\rescell{-0.31}{4.06}{3.97}{-3.43}{0.17}{3.44}{1.66}{0.91}{1.88}
\\[1.2em]
\hline

$n=70$
&
\rescellfirst{-6.14}{16.76}{17.45}{-12.87}{0.08}{12.87}{6.57}{3.13}{7.24}
&
\rescell{-8.80}{24.89}{25.81}{-7.49}{0.16}{7.49}{2.52}{0.90}{2.67}
&
\rescell{0.42}{6.61}{6.45}{-3.47}{0.19}{3.48}{1.54}{0.64}{1.66}
\\[1.2em]
\hline

$n=80$
&
\rescellfirst{-82.87}{357.27}{357.95}{-12.88}{0.05}{12.88}{4.61}{1.34}{4.79}
&
\rescell{-3.05}{18.03}{17.83}{-7.56}{0.13}{7.56}{2.12}{0.94}{2.31}
&
\rescell{2.41}{9.88}{9.93}{-3.46}{0.15}{3.46}{1.30}{0.41}{1.36}
\\
\bottomrule
\end{tabular}
\end{minipage}
\end{adjustbox}
\end{table}

\clearpage

\bibhang=1.7pc
\bibsep=2pt
\fontsize{9}{14pt plus.8pt minus .6pt}\selectfont
\renewcommand\bibname{\large \bf References}
%\begin{thebibliography}{11}
\expandafter\ifx\csname
natexlab\endcsname\relax\def\natexlab#1{#1}\fi
\expandafter\ifx\csname url\endcsname\relax
  \def\url#1{\texttt{#1}}\fi
\expandafter\ifx\csname urlprefix\endcsname\relax\def\urlprefix{URL}\fi

% \vskip .3cm
%\centerline{(Received ???? 20??; accepted ???? 20??)}\par

\appendix

\vskip 0.2in
\bibliography{sample}
\end{document}